\documentclass[11pt]{article}

\usepackage[english]{babel}
\usepackage[utf8]{inputenc} 
\usepackage{hyperref}       
\usepackage{url}            
\usepackage{dsfont}
\usepackage{booktabs}       
\usepackage{amsfonts}       
\usepackage{nicefrac}       
\usepackage{microtype}      
\usepackage{enumitem}

\usepackage{graphicx}
\usepackage[titletoc,title]{appendix}
\usepackage{amsmath,amsthm,amssymb,mathtools,enumitem}
\usepackage{braket}
\newtheorem{theorem}{Theorem}[section]
\newtheorem*{theorem*}{Theorem}

\newtheorem{proposition}[theorem]{Proposition}
\newtheorem*{proposition*}{Proposition}
\newtheorem{lemma}[theorem]{Lemma}
\newtheorem*{lemma*}{Lemma}
\newtheorem{corollary}[theorem]{Corollary}
\newtheorem*{conjecture*}{Conjecture}
\newtheorem{fact}[theorem]{Fact}
\newtheorem*{fact*}{Fact}

\newtheorem*{hypothesis*}{Hypothesis}

\newtheorem{itheorem}[theorem]{Informal Theorem}
\newtheorem{iproposition}[theorem]{Informal Proposition}
\newtheorem{claim}[theorem]{Claim}
\newtheorem*{claim*}{Claim}

\theoremstyle{definition}
\newtheorem{definition}[theorem]{Definition}

\theoremstyle{remark}
\newtheorem{remark}[theorem]{Remark}
\newtheorem*{remark*}{Remark}
\newtheorem{observation}[theorem]{Observation}
\newtheorem*{observation*}{Observation}

\newcommand{\eat}[1]{}

\newcommand{\R}{\mathbb{R}}
\newcommand{\N}{\mathbb{N}}

\newcommand{\I}{\mathbb{I}}

\newcommand{\calD}{\mathcal{D}}

\newcommand{\calS}{{S}}

\newcommand{\poly}{\mathrm{poly}}

\newcommand{\paren}[1]{(#1)}
\newcommand{\Paren}[1]{\Big(#1\Big)}

\newcommand{\bigparen}[1]{\big(#1\big)}
\newcommand{\Bigparen}[1]{\Big(#1\Big)}

\newcommand{\Abs}[1]{\left\lvert#1\right\rvert}
\newcommand{\bigabs}[1]{\big\lvert#1\big\rvert}
\newcommand{\Bigabs}[1]{\Big\lvert#1\Big\rvert}

\newcommand{\bigset}[1]{\big\{#1\big\}}
\newcommand{\Bigset}[1]{\Big\{#1\Big\}}
\newcommand{\norm}[1]{\lVert #1 \rVert}
\newcommand{\Norm}[1]{\left\lVert#1\right\rVert}
\newcommand{\bignorm}[1]{\big\lVert#1\big\rVert}
\newcommand{\Bignorm}[1]{\Big\lVert#1\Big\rVert}

\newcommand{\iprod}[1]{\langle#1\rangle}
\newcommand{\Iprod}[1]{\left\langle#1\right\rangle}

\newcommand{\Esymb}{\mathbb{E}}
\newcommand{\Psymb}{\mathbb{P}}

\DeclareMathOperator*{\E}{\Esymb}
\DeclareMathOperator*{\Var}{\text{Var}}
\DeclareMathOperator*{\ProbOp}{\Psymb}

\renewcommand{\Pr}{\ProbOp}

\newcommand{\diag}{\text{diag}}

\newcommand{\tr}{\text{tr}}

\newcommand{\tDs}{\widetilde{\calD}^{(s)}}
\newcommand{\sModel}{\mathcal{M}_{\beta}(\Dsr, \tDs, \Dv)}

\newcommand{\eps}{\varepsilon}
\renewcommand{\epsilon}{\varepsilon}

\newcommand{\polylog}{\text{polylog}}

\newif\ifnotes\notesfalse

\ifnotes
\usepackage{color}
\definecolor{mygrey}{gray}{0.50}
\newcommand{\notename}[2]{{\textcolor{mygrey}{\footnotesize{\bf (#1:} {#2}{\bf ) }}}}

\else

\newcommand{\notename}[2]{{}}

\fi

\newcommand{\pnote}[1]{{\notename{Pranjal}{#1}}}
\newcommand{\anote}[1]{{\notename{Aravindan}{#1}}}

\newcommand{\TODO}[1]{{\bf TODO:} {\bf \it #1}}

\newcommand{\abs}[1]{\lvert#1\rvert}

\usepackage{tikz}
\usepackage{multirow}
\usepackage[noend]{algpseudocode}

\newcommand{\Ds}{\calD^{(s)}}
\newcommand{\Dsr}{\calD^{(s)}_R}
\newcommand{\Dv}{\calD^{(v)}}

\newcommand{\sign}{\text{sgn}}
\newcommand{\ubar}{\bar{u}}
\newcommand{\supp}{\text{supp}}
\newcommand{\tcalD}{\widetilde{\calD}}
\newcommand{\ysamp}[1]{y^{(#1)}}
\newcommand{\xsamp}[1]{x^{(#1)}}

\newcommand{\Zeta}{\mathrm{Z}}
\newcommand{\upzeta}[1]{\zeta^{(#1)}}
\newcommand{\upxi}[1]{\xi^{(#1)}}
\newcommand{\upZ}[1]{Z^{(#1)}}

\newcommand{\Tgamma}{T^{(\Gamma, \Gamma^c)}}
\newcommand{\Tgam}[1]{T^{(#1, #1^c)}}

\newcommand{\Ahat}{\widehat{A}}
\newcommand{\imbal}{\rho}
\newcommand{\tol}{\lambda}

\usepackage[a4paper,top=3cm,bottom=2cm,left=3cm,right=3cm,marginparwidth=1.75cm]{geometry}
\usepackage{amsmath}
\usepackage{graphicx}
\usepackage[colorinlistoftodos]{todonotes}

\title{Towards Learning Sparsely Used Dictionaries with Arbitrary Supports}
\author{Pranjal Awasthi\thanks{Department of Computer Science, Rutgers University.}
 \\{\tt pranjal.awasthi@rutgers.edu} \and
Aravindan Vijayaraghavan\thanks{
  Department of Electrical Engineering and Computer Science,
  Northwestern University. Supported by the National Science Foundation (NSF) under Grant No.~CCF-1652491 and CCF-1637585.} \\
  {\tt aravindv@northwestern.edu}
}
\date{}
\begin{document}
\maketitle

\begin{abstract}
Dictionary learning is a popular approach for inferring a hidden basis in which data has a sparse representation. There is a hidden dictionary or basis $A$ which is an $n \times m$ matrix, with $m>n$ typically (this is called the over-complete setting).  
Data generated from the dictionary is given by $Y = AX$ where $X$ is a matrix whose columns have supports chosen from a distribution over $k$-sparse vectors, and the non-zero values chosen from a symmetric distribution. Given $Y$, the goal is to recover $A$ and $X$ in polynomial time (in $m,n$).
Existing algorithms give polynomial time guarantees for recovering incoherent dictionaries, under strong distributional assumptions both on the supports of the columns of $X$, and on the values of the non-zero entries. In this work, we study the following question: {\em can we design efficient algorithms for recovering dictionaries when the supports of the columns of $X$ are arbitrary?}

To address this question while circumventing the issue of non-identifiability, 
we study a natural semirandom model for dictionary learning.   
In this model, there are a large number of samples $y=Ax$ with arbitrary $k$-sparse supports for $x$, along with a few samples where the sparse supports are chosen uniformly at random. While the presence of a few samples with random supports ensures identifiability, the support distribution can look almost arbitrary in aggregate. 
Hence, existing algorithmic techniques seem to break down as they make strong assumptions on the supports. 

Our main contribution is a new polynomial time algorithm for learning incoherent over-complete dictionaries that provably works under the semirandom model. 
Additionally the same algorithm provides polynomial time guarantees in new parameter regimes when the supports are fully random. 
Finally, as a by product of our techniques, we also identify a minimal set of conditions on the supports under which the dictionary can be (information theoretically) recovered from polynomially many samples for almost linear sparsity, i.e.,  $k=\widetilde{O}(n)$. 


\end{abstract}
\section{Introduction}
\label{sec:intro}

In many machine learning applications, the first step towards understanding the structure of naturally occurring data such as images and speech signals is to find an appropriate basis in which the data is sparse. Such sparse representations lead to statistical efficiency and can often uncover semantic features associated with the data. For example images are often represented using the {\it SIFT} basis~\cite{lowe1999object}. Instead of designing an appropriate basis by hand, 
 the goal of dictionary learning is to 
 algorithmically learn from data, the basis (also known as the dictionary) along with the data's sparse representation in the dictionary. 
This problem of dictionary learning or sparse coding was first formalized in the seminal work of Olshausen and Field~\cite{sparse1}, and has now become an integral approach in unsupervised learning for feature extraction and data modeling. 

The dictionary learning problem is to learn the unknown dictionary $A \in \R^{n \times m}$ and recover the sparse representation $X$ given data $Y$ that is generated as follows. The typical setting is the ``over-complete'' setting when $m>n$. 
Each column $A_i$ of $A$ is a vector in $\R^n$ and is part of the over-complete basis. Data is then generated by taking random sparse linear combinations of the columns of $A$. Hence the data matrix $Y \in \R^{n \times N}$ is generated as $Y = AX$, where $X \in \R^{m \times N}$ captures the representation of each of the $N$ data points\footnote{In general there can also be noise in the model where each column of $Y$ is given by $y=Ax+\psi$ where $\psi$ is a noise vector of small norm. In this paper we focus on the noiseless case, though our algorithms are also robust to inverse polynomial error in each sample. 
}. Each column of $X$ is a vector drawn from a distribution $\Ds \odot \Dv$. Here $\Ds$ is a distribution over $k$ sparse vectors in $\{0,1\}^m$ and represents the {\em support distribution}. 
Conditioning on support of the column $x$, each non-zero value is drawn independently from $\Dv$, which represents the {\em value distribution}. 

The goal of recovering $(A,X)$ from $Y$ is particularly challenging in the over-complete setting --  
notice that even if $A$ is given, finding the matrix $X$ with sparse supports such that $Y = A X$ is the sparse recovery or compressed sensing problem which is NP-hard in general~\cite{davis1997adaptive}. A beautiful line of work~\cite{donoho1989uncertainty, donoho2001uncertainty, CTao05, candes2006stable} gives polynomial time recovery of $X$ (given $A$) under certain assumptions about $A$ like Restricted Isometry Property (RIP) and incoherence. 
See Section~\ref{sec:prelims} for formal definitions. 

While there have been several 
heuristics and algorithms proposed for dictionary learning, 
the first rigorous polynomial time guarantees were given by Spielman et al.~\cite{SWW} who focused on the {\em full rank} case, i.e., $m=n$. They assumed that the support distribution $\Ds$ is uniformly random (each entry is non-zero independently with probability $p=k/m=1/\sqrt{m}$) and the value distribution $\Dv$ is a symmetric sub-Gaussian distribution, and this has subsequently been improved by~\cite{qu2014finding} to handle almost linear sparsity. 
The first algorithmic guarantees for learning over-complete dictionaries ($m$ can be larger than $n$)  from polynomially many (in $m,n$) samples, and in polynomial time were independently given by Arora et al.~\cite{AGMM14} and Agarwal et al.~\cite{agarwal2013exact}.
In particular, the work of~\cite{AGMM15} and its follow up work~\cite{AGMM15} 
provide guarantees for sparsity up to $n^{1/2}/\log m$, and also assumes slightly weaker assumptions on the support distribution $\Ds$,  requiring it to be approximately $O(1)$-wise independent.
The works of \cite{BKS15} and \cite{ma2016polynomial} gives Sum of Squares~(SoS) based quasi-polynomial time algorithms (and polynomial time guarantees in some settings) to handle almost linear sparsity under similar distributional assumptions. See Section~\ref{sec:related} for a more detailed discussion and comparison of these works. 

While these algorithms give polynomial time guarantees even in over-complete settings, they crucially rely on strong distributional assumptions on both the support distribution $\Ds$ and the value distribution $\Dv$. Curiously, it is not known whether these strong assumptions are necessary to recover $A,X$ from polynomially many samples, even from an information theoretic point of view. \anote{add?: and it is unclear if these assumptions are necessary.} This motivates the following question that we study in this work:

\vspace{5pt}
\noindent {\em Can we design efficient algorithms for learning over-complete dictionaries when the support distribution is essentially arbitrary?}

\vspace{5pt}
As one might guess, the above question as stated, is ill posed since recovering the dictionary is impossible if there is a column that is involved in very few samples\footnote{See Proposition~\ref{prop:non-identifiability} for a more interesting example.}. In fact we do not have a good understanding of when there is a unique $(A,X)$ pair that explains the data (this is related to the question of identifiability of the model). 
However, consider the following thought-experiment: suppose we have an instance with a large number of samples, each of the form $y=Ax$ with $x$ being an arbitrary sparse vector. In addition, suppose we have a few samples ($N_0$ of them) that are drawn from the standard dictionary learning model where the supports are random. 
The mere presence of the samples with random supports will ensure that there is a unique dictionary $A$ that is consistent with {\em all} the samples (as long as $N_0=\Omega(n^2)$ for example). On the other hand, since most of the samples have arbitrary sparse supports, the aggregate distribution looks fairly arbitrary\footnote{since we do not know which of the samples are drawn with random support.}. This motivates a natural semirandom model towards understanding dictionary learning when the sparse supports are arbitrary.       


\anote{See if first couple of lines in model are repetitive and make it crisper. }

\paragraph{The semirandom model.} In this model we have $N$ samples of the form $y=Ax$ with most of them having arbitrary $k$-sparse supports for $x$, and a few samples ($N_0$ of them) that are drawn from the random model for dictionary learning. We will use $\tDs$ to represent the arbitrary distribution over $k$-sparse supports and $\Dsr$ to represent the random distribution over $k$-sparse supports (as considered in prior works) and a parameter $\beta$ to represent the fraction of samples from $\Dsr \odot \Dv$ (it will be instructive to think of $\beta$ as very small e.g., an inverse polynomial in $n,m$). 
$N$ samples from the semirandom model $\sModel$ are generated as follows. 
\begin{enumerate}
\item The supports of $N_0=\beta N$ samples $\xsamp{1}, \dots, \xsamp{N_0}$ are generated from the random distribution $\Dsr$ over $k$-sparse $\set{0,1}^m$ vectors \footnote{More generally, $\Dsr$ can be any distribution that is $\tau$-negatively correlated -- here $\forall S$ s.t. $|S|=O(\log m), i \notin S$,  the probability $\Pr[i \in \supp(x)~|~S \subset \supp(x)]\le \tau k/m$, and $\Pr[i \in supp(x)] \approx k/m$.}.   
\item The adversary chooses the $k$-sparse supports of $N_1=(1-\beta)N$ samples arbitrarily (or equivalently from an arbitrary distribution $\tDs$). Note that the adversary can also see the supports of the $N_0$ ``random'' samples.
\item The values of each of the non-zeros in $X=\set{\xsamp{\ell}: \ell \in [N]}$ are picked independently from the value distribution $\Dv$ e.g., a Rademacher distribution ($\pm 1$ with equal probability). 
\item The $\xsamp{1},\dots,\xsamp{N}$ are reordered randomly to form matrix $X \in \R^{m \times N}$ and the data matrix $Y=AX$. $Y$ is the instance of the dictionary learning problem. 
\end{enumerate}
The samples that are generated in step 1 will be referred to as the random portion (or random samples), and the samples generated in step 2 will be referred to adversarial samples. 
As mentioned earlier, the presence of just the random portion ensures that the model is identifiable (assuming $\beta N=n^{\Omega(1)}$) from known results, and there is unique solution $A$. The additional samples that are added in step 2 represent more $k$-sparse combinations of the columns of $A$ -- hence, intuitively the adversary is only helpful by presenting more information about $A$ (such adversaries are often called monotone adversaries). On the other hand, the fraction of random samples $\beta$ can be very small 
(think of $\beta=O(1/\poly(n))$) -- hence the adversarial portion of the data can completely overwhelm the random portion. Further, the support distribution $\tDs$ chosen by the adversary (or the supports of the adversarial samples) could have arbitrary correlations and also depend on the 
the support patterns in the random portion. 
Hence, the support distribution can look very adversarial, and this  
is challenging for existing algorithmic techniques, which seem to break down in this setting (see Sections~\ref{sec:related} and \ref{sec:techniques}).  

Semirandom models starting with works of \cite{BS92,FK99} have been a very fruitful paradigm for interpolating between average-case analysis and worst-case analysis. Further, we believe that studying such semirandom models for unsupervised learning problems will be very effective in identifying robust algorithms that do not use strong distributional properties of the instance. For instance, algorithms based on convex relaxations for related problems like compressed sensing~\cite{CTao05} and matrix completion~\cite{CT10} are robust in the presence of a similar monotone adversary where there are additional arbitrary observations in addition to the random observations.
\pnote{Not sure if we need the last line.}



\subsection{Our Results}\label{sec:results}

We present a new polynomial time algorithm for dictionary learning that works in the semirandom model and obtain new identifiability results under minimal assumptions about the sparse supports of $X$. 
We give an overview of our results for the simplest case, when the value distribution $\Dv$ is a Rademacher distribution i.e., each non-zero value $x_i$ is either $\set{+1, -1}$ with equal probability. These results also extend to a more general setting  where the value distribution $\Dv$ can be a mean-zero symmetric distribution supported in $[-C,-1] \cup [1,C]$ for a constant $C > 1$ -- this is called {\em Spike-and-Slab} model~\cite{goodfellow2012large} and has been considered in past works on sparse coding~\cite{AGMM14}.   
\anote{Does it seem misleading, by not mentioning $\gamma_0$?}
\pnote{I think it's ok.}
As with existing results on recovering dictionaries in the over-complete setting, 
we need to assume that the matrix satisfies some incoherence or Restricted Isometry Property (RIP) conditions (these are standard assumptions even in the sparse recovery problem when $A$ is given).
 A matrix $A$ is $(k,\delta)$-RIP iff $(1-\delta)\norm{x}_2 \le \norm{Ax}_2 \le (1+\delta) \norm{x}_2$ for all $k$-sparse vectors, and a matrix is $\mu$-incoherent iff $\abs{\iprod{A_i, A_j}}\le \mu/\sqrt{n}$ for every two columns $i \ne j \in [m]$.  
Random $n \times m$ matrices satisfy the $(k,\delta)$-RIP property as long as $k=O(\delta n/\log(\tfrac{n}{\delta k}))$~\cite{Baraniuk}, and are $\mu=O(\sqrt{\log m})$ incoherent. Please see Section~\ref{sec:prelims} for the formal model and assumptions.\anote{Added incoherence defn too.}

Our main result is a polynomial time algorithm for learning over-complete dictionaries when we are given samples from the semirandom model proposed above.

\begin{itheorem}[Polytime algorithm for semirandom model]
\label{ithm:sr-algorithm}
Consider a dictionary $A \in \R^{n \times m}$ that is $\mu$-incoherent with spectral norm $\sigma$.
There is a polynomial time algorithm that given $\poly(n,m,k,1/\beta)$ samples generated from the semirandom model (with $\beta$ fraction random samples) with sparsity $k \le \sqrt{n}/(\mu^{O(1)} (\sigma m/n)^{O(1)}\polylog m)$, recovers with high probability the dictionary $A$ up to arbitrary (inverse-polynomial) accuracy 
(up to relabeling the columns, and scaling by $\pm 1$)\footnote{We will recover a dictionary $\widehat{A}$ such that $\norm{\widehat{A}_i - b_i A_i}_2 \le \eta_0$ for some $b \in \set{-1,1}^m$, where $\eta_0$ is the desired inverse-polynomial accuracy. While we state our guarantees for the noiseless case of $Y=AX$, our algorithms are robust to inverse polynomial additive noise.}.
\end{itheorem}
Please see Theorem~\ref{thm:main-full-algorithm} for a formal statement. The above algorithm recovers the dictionary up to arbitrary accuracy in the semirandom model for sparsity $k=\widetilde{O}(n^{1/2})$ -- as we will see soon, this is comparable to the state-of-the-art polynomial time guarantees even when there are no adversarial samples. 
 By using standard results from sparse recovery~\cite{CTao05,candes2006stable}, one can then use our knowledge of $A$ to 
 recover $X$. 
We emphasize in the above bounds that the sparsity assumption and recovery error do not have any dependence on $\beta$ the fraction of samples generated from the random portion. The dependence on $1/\beta$ in the sample complexity simply ensures that there are a few samples from the random portion in the generated data. 

When there are no additional samples from the adversary i.e., $\beta=1$, our algorithm in fact handles a significantly larger sparsity of $k=\widetilde{O}(m^{2/3})$ 

\begin{itheorem}[Beyond $\sqrt{n}$ with no adversarial supports ($\beta=1$)]
\label{ithm:random-algorithm}
Consider a dictionary $A \in \R^{n \times m}$ that is $\mu$-incoherent and $(k,1/\polylog m)$-RIP with spectral norm $\sigma$.
There is a polynomial time algorithm that given $\poly(n,m,k)$ samples generated from the ``random'' model with sparsity $k \le n^{2/3}/(\mu^{O(1)}(\sigma m/n)^{O(1)} \polylog m)$, recovers with high probability the dictionary $A$ up to arbitrary accuracy. 
\end{itheorem}
 Please see Theorem~\ref{thm:random-recovery-main} for a formal statement. 
For the sake of comparison, consider the case when the amount of over-completeness is $\widetilde{O}(1)$ or even $n^{\eps}$ for some small constant $\eps>0$ i.e., $m/n, \sigma \le n^{\eps}$.\footnote{The parameter $\sigma$ is an analytic measure of over-completeness; for any dictionary $A$ of size $n \times m$, $\sigma \ge \sqrt{m/n}$. Conversely, one can also upper bound $\sigma$ in terms of $m/n$ under RIP-style assumptions. 
When the columns of $A$ are random, then $\sigma =O(\sqrt{m/n})$;  otherwise, $\sigma=O(\sqrt{m/k})$ when $A$ is $(k,O(1))$-RIP.} 
The results of Arora et al.~\cite{AGMM14,AGMM15} recover the dictionaries for sparsity $k =\widetilde{O}(\sqrt{n})$, when there are no adversarial samples. 
On the other hand, sophisticated algorithms based on Sum-of-Squares (SoS) relaxations ~\cite{BKS15,ma2016polynomial} give quasi-polynomial time guarantees in general (and polynomial time guarantees when $\sigma=O(1)$) 
for sparsity going up to $k=O(m/\polylog m)$ when there are no adversarial samples. 
Hence, our algorithm gives polynomial time guarantees in new settings when sparsity $k=\omega(\sqrt{n})$ even in the absence of any adversarial samples (Theorem~\ref{ithm:random-algorithm}), and at the same time gives polynomial time guarantees for $k=\widetilde{O}(\sqrt{n})$ in the semirandom model even when the supports are almost arbitrary. 
Please see Section~\ref{sec:related} for a more detailed comparison. 


A key component of our algorithm that is crucial in handling the semirandom model is a new efficient procedure that allows us to test whether a given unit vector is close to a column of the dictionary $A$. In fact this procedure works up to sparsity $k=O(n/\polylog(m))$.  

\begin{itheorem}[Test Candidate Column]\label{ithm:test}
Given any unit vector $z \in \R^n$, there is a polynomial time algorithm (Algorithm~\ref{ALG:test}) that uses $\poly(m,n,k,1/\eta_0)$ samples from the semirandom model with the sparsity $k \le n/\polylog(m)$ and the dictionary $A$ satisfying $(k,\delta=1/\polylog(m))$-RIP property, that with probability at least $1-\exp(-n^2)$: 
\begin{itemize}
\item {\em (Completeness)} Accepts $z$ if $\exists i \in [m], ~b \in \set{\pm 1}$ s.t. $\norm{z- bA_i}_2 \le 1/\polylog(m)$.   
\item {\em (Soundness)} Rejects $z$ if $\norm{z-bA_i}_2 > 1/\poly\log(m)$ for every $i \in [m], ~b \in \set{\pm 1}$. \end{itemize} 
Moreover in the first case, the algorithm also returns a vector $\widehat{z}$ s.t. $\norm{\widehat{z}-bA_i}_2 \le \eta_0$, where $\eta_0$ represents the desired inverse polynomial accuracy. 
\end{itheorem}
Please see Theorem~\ref{thm:sr:testing} for a formal statement\footnote{The above procedure is also noise tolerant -- it is robust to adversarial noise of $1/\polylog(n)$ in each sample.}.
Our test is very simple and proceeds by  computing inner products of the candidate vector $z$ with samples and looking at the histogram of the values. Nonetheless, this  provides a very powerful subroutine to discard vectors that are not close to any column. The full algorithm then proceeds by efficiently finding a set of candidate vectors (by simply considering appropriately weighted averages of all the samples), 
and running the testing procedure on each of these candidates. The analysis of the candidate-producing algorithm requires several ideas such as proving new concentration bounds for polynomials of rarely occurring random variables, which we describe in Section~\ref{sec:techniques}.  

In fact, the above test procedure  works under more general conditions about the support distribution. This immediately implies {\em polynomial identifiability} for near-linear sparsity $k=O(n/\polylog m)$, by simply applying the procedure to every unit vector in an appropriately chose $\epsilon$-net of the unit sphere. 


\begin{itheorem}[Polynomial Identifiability for Rademacher Value Distribution]\label{ithm:rad:identifiability}
Consider a dictionary $A \in \R^{n \times m}$ that is $(k,\delta=1/\polylog(m))$-RIP property for sparsity $k \le n/\polylog(m)$ and suppose we are given $N=\poly(n,m,k,1/\beta)$ samples with arbitrary $k$-sparse supports that satisfies the following condition:

$\forall i_1, i_2, i_3 \in [m]$, there at least a few samples (at least $1/\poly(n)$ fraction) $y=Ax$ such that $i_1, i_2, i_3 \in \supp(x)$.

Then, there is a algorithm (potentially exponential runtime) that recovers with high probability a dictionary $\widehat{A}$ such that $\norm{\widehat{A}_i - b_i A_i}_2 \le 1/\poly(m)$ for some $b \in \set{-1,1}^m$ (up to relabeling the columns).
\end{itheorem}
Please see Corollary~\ref{corr:rad:identifiability} for a formal statement, and Corollary~\ref{corr:sr:identifiability} for related polynomial identifiability results under more general value distributions. 

The above theorem proves polynomial identifiability for arbitrary set of supports as long as every triple of columns $i_1, i_2, i_3$ co-occur i.e., there are at least a few samples where they jointly occur (this would certainly be true if the support distribution has approximate three-wise independence). On the other hand in Proposition~\ref{prop:non-identifiability}, we complement this by proving a {\em non-identifiability} result using an instance that does not satisfy the ``triples'' condition, but where every pair of columns co-occur. 
Hence, Corollary~\ref{corr:rad:identifiability} gives polynomial identifiability under arguably minimal assumptions on the supports. To the best of our knowledge, prior identifiability results were only known through the algorithmic results mentioned above, or using $n^{O(k)}$ many samples. Hence, while designed with the semirandom model in mind, our test procedure also allows us to shed new light on the information-theoretic problem of polynomial identifiability with adversarial supports. 

\anote{4/22: Added this line.} Developing polynomial time algorithms that handle a sparsity of $k=\widetilde{O}(n)$ under the above conditions (e.g., Theorem~\ref{ithm:rad:identifiability}) that guarantee polynomial identifiability, or in the semirandom model, are interesting open questions.  

\subsection{Technical Overview} \label{sec:techniques}

We now give an overview of 
the technical ideas involved in proving our algorithmic and identifiability results. Some of these ideas are also crucial in handling sparsity of $k=\omega(\sqrt{n})$ in the random model. Further, as we will see in the discussion that follows, the algorithm will make use of samples from both the random portion and the semi-random portion for recovering the columns. \anote{Is the prev. line OK? Just added it...}For the sake of exposition, let us restrict our attention to the value distribution being Rademacher i.e., each non-zero $x_i$ is either $+1$ or $-1$ independently with equal probability.   

\paragraph{Challenges with semirandom model for existing approaches.} We first describe the challenges and issues that come in the semirandom model, and more generally when dealing with arbitrary support distributions. Many algorithms for learning over-complete dictionaries typically proceed by computing aggregate statistics of the samples e.g., appropriate moments of the samples $y=Ax$ (where $x \sim \calD$), and then extracting individual columns of the dictionary -- either using spectral approaches~\cite{AGMM15} or using tensor decompositions~\cite{BKS15,ma2016polynomial}. However, in the semirandom model, the adversary can generate many more samples with adversarial supports, and dominate the number of random samples (it can be $\poly(n)$ factor larger) --- this can completely overwhelm the contribution of the random samples to the aggregate statistic . In fact, the supports of these adversarial samples can depend on the random samples as well.  

To further illustrate the above point, let us consider the algorithm of Arora et al.~\cite{AGMM15}. They guess two fixed samples $u^{(1)}=A \upzeta{1}, u^{(2)}=A \upzeta{2}$ and consider the statistic
\begin{align}
B&= \E_{y=Ax}\Big[ \iprod{y,u^{(1)}} \iprod{y,u^{(2)}} y \otimes y \Big] = \sum_{i \in [m]} \bigparen{\E_{x \sim \calD}\big[ x_i^4 \big] \iprod{A_i, \upzeta{1}} \iprod{A_i, \upzeta{2}}} \cdot A_i \otimes A_i + \nonumber\\
&\quad + \sum_{i \ne i'} \E_{x \sim \calD}[ x_{i}^2 x_{i'}^2] \bigparen{\iprod{A_{i},\upzeta{1}} \iprod{A_{i'}, \upzeta{2}} A_{i} \otimes A_{i'} + \iprod{A_{i'},\upzeta{1}} \iprod{A_{i}, \upzeta{2}} A_{i'} \otimes A_{i}}+\dots  \label{eq:intro:challenges}
\end{align}
To recover the columns of $A$ there are two main arguments involved. For the correct guess of $u^{(1)}, u^{(2)}$ with $\supp(\upzeta{1}),\supp(\upzeta{2})$ containing exactly one co-ordinate in common e.g., $i=1$, they show that one gets $B=q_1 A_1 A_1^T+E$ where $\norm{E} = o(q_1)$. In this way $A_1$ can be recovered up to reasonable accuracy (akin to {\em completeness}). To argue that $\|E\| = o(q_1)$, one can use the randomness in the support distribution to get that $\E[x_{i}^2 x_{i'}^2]=O(k^2/m^2)$ is significantly smaller (by a factor of approximately $k/m$) compared to $\E[x_{1}^2] \approx k/m$. On the other hand, one also needs to argue that for the wrong guess of $u^{(1)}, u^{(2)}$, the resulting matrix $B$ is not close to rank $1$ ({\em soundness}). The argument here, again relies crucially on the randomness in the support distribution.      

In the semirandom model, both completeness and soundness arguments are affected by the power of the adversary. For instance, if the adversary generates samples such that a subset of co-ordinates $T \subseteq [m]$ co-occur most of the time, then for every $i, i' \in T,~ \E[x_{i}^2 x_{i'}^2] =\Omega(\E[x_{i}^2])$. Hence, completeness becomes harder to argue since the cross-terms in \eqref{eq:intro:challenges} can be much larger (particularly for $k = \Omega(m^{1/8})$). The more critical issue is with soundness, since it is very hard to control and argue about the matrices $B$ that are produced by incorrect guesses of $u^{(1)}, u^{(2)}$ (note that they can also be from the portion with adversarial support). For the above strategy in particular, there are adversarial supports and choices of samples 
such that $B$ is close to rank $1$ but whose principal component is not aligned along any of the columns of $A$ (e.g., it could be along $\sum_{i \in T} A_i$). We now discuss how we overcome these challenges in the semirandom model.    



\paragraph{Testing for a Single Column of the Dictionary.}
A key component of our algorithm is a new efficient procedure, which when given a candidate unit vector $z$ tests whether $z$ is indeed close to one of the columns of $A$ (up to signs) or is far from every column of the dictionary i.e., $\norm{z - bA_i}_2 > \eta$ for every $i \in [m], b\in \set{-1,1}$ ($\eta$ can be chosen to be $1/\poly\log(n)$ and the accuracy can be amplified later). Such a procedure can be used as a sub-routine with any algorithm in the semirandom model since it addresses the challenge of ensuring soundness. We can feed a candidate set of test vectors generated by the algorithm and discard the spurious ones. 

The test procedure (Algorithm~\ref{ALG:test}) is based on the following observation: if $z=b A_i$ for some column $i \in [m]$ and $b \in \set{\pm 1}$, then the distribution of $\abs{\iprod{z,Ax}}$ will be bimodal, depending on whether $x_i$ is non-zero or not. This is because  
$$\abs{\iprod{b A_i, Ax}} = \abs{x_i} \pm \bigabs{\sum_{j \ne i} \iprod{A_i, A_j} x_j}=\abs{x_i} \pm o(1)~~~ \text{ with high probability},$$
when $A$ satisfies the RIP property (or incoherence). Hence Algorithm {\sc TestColumn} (Algorithm~\ref{ALG:test}) just computes the inner products $\abs{\iprod{z,Ax}}$ with polynomially many samples (it could be from the random or adversarial portion), and checks if they are always close to $0$ or $1$, with a non-negligible fraction of them (roughly $k/m$ fraction, if each of the $i$ occur equally often) taking a value close to $1$. 

The challenge in proving the correctness of this test is the soundness analysis: if unit vector $z$ is far from any column of $A_i$, then we want to show that the test fails with high probability. Consider a candidate $z$ that passes the test, and let $\alpha_i:=\iprod{z,A_i}$. Suppose $\abs{\alpha_i}=o(1)$ for each $i \in [m]$ (so it is far from every column). For a sample $y=Ax$ with $\supp(x)=S$, 
\begin{equation}\iprod{z,Ax}=\sum_{i \in S} x_i \iprod{z,A_i}=\sum_{i \in S} \alpha_i x_i.\label{eq:overview:sum}
\end{equation}
The quantity $\iprod{z,Ax}$ is a weighted sum of symmetric, independent random variables $x_i$, and the variance of $\iprod{z,Ax}$ equals $\norm{\alpha_S}_2^2 = \sum_{i \in S} \alpha_i^2$. When $\norm{\alpha_S}_2 = \Omega(1)$, Central Limit Theorems like the Berry-Ess\'{e}en theorem tells us that the distribution of the values of $\iprod{z,Ax}$ is close to a Normal distribution with $\Omega(1)$ variance. In this case, we can use anti-concentration of a Gaussian to prove that $\abs{\iprod{z,Ax}}$ takes a value bounded away from $0$ or $1$ (e.g., in the interval $[\tfrac{1}{4}, \tfrac{3}{4}]$) with constant probability. However, the variance $\norm{\alpha_S}_2^2=\sum_{i \in S}\alpha_i^2$ can be much smaller than $1$ (for a random unit vector $z$, we expect $\norm{\alpha_S}_2^2 =O(k/n)$). In general, we have very little control over the $\alpha_S$ vector since the candidate $z$ is arbitrary. For an arbitrary spurious vector $z$, we need to argue that either $\abs{\iprod{z,Ax}}$ (almost) never takes large values close to $1$, or takes values bounded away from $0$ and $1$ (e.g., in $[0.1, 0.9]$) for a non-negligible fraction of the samples. 

The correctness of our test relies crucially on such an anti-concentration statement, which may be of independent interest.  
\begin{claim}[See Lemma~\ref{lem:weakanticonc} for a more general statement]
Let $X_1, X_2, \dots, X_\ell$ be independent Rademacher random variables and let $Z = \sum_{i=1}^\ell a_i X_i$ where $\norm{a}_2 = 1$. There exists constants $c,c'>0$ s.t. for any $\eta',\kappa \in (0,1)$, $\beta \in (\tfrac{1}{16}, \tfrac{7}{16})$ and any $t \ge \max\set{1,c' \norm{a}_\infty}$, 
\begin{equation} 
\Pr_{X} \Big[ Z \in \big[(1-\eta')t , (1+\eta')t\big] \Big] \ge \kappa ~\implies~ \Pr_{X} \Big[ Z \in \big[\tfrac{\beta}{2}(1-\eta')t , \tfrac{3}{2} \beta (1+\eta')t\big]   \Big] \ge \Omega(\kappa). 
\end{equation}
\end{claim}
Note that in the above statement $a$ is normalized; we will apply the above claim with $a=\alpha/ \norm{\alpha}_2$. 

When $\kappa$ is large e.g., a constant or $t=\Omega(\norm{a}_2)$, one can use CLTs together with Gaussian anti-concentration to prove the claim.
However, even when the weights are all equal, such bounds do not work when $\kappa=1/\poly(n,m) \ll 1/\sqrt{k}$ or $t \gg \norm{a}_2$, which is our main setting of interest (this regime of $\kappa$ corresponds to the tail of a Gaussian, as opposed to the central portion of a Gaussian where CLTs can be applied for good bounds). In this regime near the tail, we prove the claim by using an argument that carefully couples occurrences of $\abs{Z} \approx t/3$ with occurrences of $\abs{Z} \approx 1$.   

The above test works for $k=O(n/\poly\log(m))$, only uses the randomness in the non-zero values, and works as long as the co-efficients $\abs{\alpha_i}$ are all small compared to $t=1$ i.e., $\norm{\alpha_S}_\infty < c t$.\footnote{There are certain configurations where $\abs{\alpha_i}$ are large, for which the above anti-concentration statement is not true. For example when $\alpha_1=\alpha_2=1/2$ and $0$ for rest of $i \in S$, then any $\pm 1$ combination of $\alpha_1, \alpha_2$ is in $\set{-1,0,1}$. In fact, in Proposition~\ref{prop:non-identifiability} we construct instances that are non-identifiable instances for which there are bad candidates $z$ which precisely result in such combinations. However, Lemma~\ref{lem:unnormal} shows that this is essentially the only bad situation for this test.} The full proof of the test uses a case analysis depending on whether there are some large co-efficients, and uses such a large coefficient in certain ``nice'' samples that exist under mild assumptions (e.g., in a semirandom model), along with the above lemma (Lemma~\ref{lem:weakanticonc}) to prove that a unit vector $z$ that is far from every column fails the test with high probability (the failure probability can be made to be $\exp(-n^2)$). Given the test procedure, to recover the columns of the dictionary, it now suffices (because of Algorithm~\ref{ALG:test}) to design an algorithm that produces a set of candidate unit vectors that includes the columns of $A$.

\paragraph{Identifiability.}  The test procedure immediately implies polynomial identifiability (i.e., with polynomially many samples) for settings where the test procedure works, by simply running the test procedure on every unit vector in an $\eps$-net of the unit sphere. When the value distribution $\Dv$ is a Rademacher distribution, we prove that just a condition on the co-occurrence of every triple suffices to construct such a test procedure (Algorithm~\ref{ALG:rad}). This implies the polynomial identifiability results of Theorem~\ref{ithm:rad:identifiability} for sparsity up to $k=O(n/\polylog n)$. For more general value distributions defined in Section~\ref{sec:prelims}, it just needs to hold that that there are a few samples where a given column $i$ appears, but a few other $O(\log n)$ given columns do not appear (e.g., for example when a subset of samples satisfy very weak pairwise independence; see  Lemma~\ref{lem:semirandom:randomsamples}). This condition suffices for Algorithm~\ref{ALG:test} to work for sparsity $k=O(n/\polylog n)$-- and implies the identifiability results in Corollary~\ref{corr:sr:identifiability}. 


\paragraph{Efficiently Producing Candidate Vectors.}
\newcommand{\cof}{\gamma}

Our algorithm for producing candidate vectors is inspired by the initialization algorithm of \cite{AGMM15}. We guess $2L-1$ samples $u^{(1)}=A\upzeta{1}, u^{(2)}=A\upzeta{2}, \dots, u^{(2L-1)}=A\upzeta{2L-1}$ for some appropriate constant $L$, and simply consider the weighted average of all samples given by
\begin{align*}
v&=\E\Big[ \iprod{y,A\upzeta{1}} \iprod{y,A\upzeta{2}} \dots \iprod{y,A\upzeta{2L-1}} ~ y \Big]
\end{align*}
and consider the unit vector along $v$. Let us consider a ``correct'' guess of $\upzeta{1}, \dots, \upzeta{2L-1}$ where all of them are from the random portion, and their supports all contain a fixed co-ordinate (say coordinate $1$). 
In this case we show that with at least a constant probability the vector $v=q_1 A_1 + \tilde{v}$ where $\norm{\tilde{v}}_2 = o(q_{\max}/\log m)$. Here $q_i$ is the fraction of samples $x$ with $i$ in its support and $q_{\max} = \max_i q_i$.  
Hence, by running over all choices of $2L-1$ tuples in the data we can hope to produce candidate vectors that are good approximations to frequently appearing columns of $A$. Notice that any spurious vectors that we produce will be automatically discarded by our test procedure. While the above algorithm is very simple, its analysis requires several new technical ideas including improved concentration bounds for polynomials of {\em rarely occurring} random variables.
To see this, we note that vector $v$ can be written as $v=\sum_{i \in [m]} \cof_i A_i$ where
$$\forall i \in [m],~ \cof_i=\sum_{j_1,\dots,j_{2L-1} \in [m]} \Paren{\sum_{i_1, \dots, i_{2L-1} \in [m]}\E\big[x_i x_{i_1} \dots x_{i_{2L-1}} x_i \big] \prod_{\ell \in [2L-1]}M_{i_\ell,j_\ell}} ~\upzeta{1}_{j_1} \dots \upzeta{2L-1}_{j_{2L-1}}.$$
Here $M$ denote the matrix $A^T A$. 
To argue that we will indeed generate good candidate vectors, we need to prove that $\norm{\sum_{i \ne 1} \gamma_i A_i}_2 = o(q_{\max}/\log m)$. For the fully random case this corresponds to proving that $\abs{\gamma_i} = o(k/(m\sqrt{m}))$ for each $i \in [m] \setminus \set{1}$. Both these statements boil down to proving concentration bounds for multilinear polynomials of the random variables $\upzeta{1}, \dots, \upzeta{2L-1}$, where $\set{\upzeta{\ell}: \ell \in [2L-1]}$ are {\em rarely occurring} mean-zero random variables i.e., they are non-zero with probability roughly $p=k/m$.   
Concentration bounds for multilinear degree-$d$ polynomials of $O(1)$ hypercontractive random variables are known, giving bounds of the form $\Pr[g(x)>t \norm{g}_2] \le \exp(-c t^{2/d})$ ~\cite{Ryanbook}. More recently, sharper bounds (analogous to the Hanson-Wright inequality for quadratic forms~\cite{HansonWright}) that do not necessarily incur a $d$ factor in the exponent and get bounds of the form $\exp(-\Omega(t^2))$ have also been obtained by Latala, Adamczak and Wolff~\cite{Latala, AdamczakWolff} for sub-gaussian random variables and more generally, random variables of bounded Orlicz $\psi_2$ norm. However, these seem to give sub-optimal bounds for rarely occurring random variables, as we demonstrate below. On the other hand, bounds that apply in the rarely occurring regime \cite{KimVu, SchudyS} typically apply to polynomials of non-negative random variables with non-negative coefficients, and do not seem directly applicable in our settings. 
\anote{Maybe the above description could be shorter?}



There are several different terms that arise in these calculations; we give an example of one such term to motivate the need for better concentration bounds in this setting with rarely occurring random variables. One of the terms that arises in the expansion of $\gamma_i$ is 
$$Z=\sum_{j_1,j_2 \in [m]} B_{j_1,j_2}  \upzeta{1}_{j_1}\upzeta{2}_{j_2}:=\sum_{i \in [m]} \sum_{j_1,j_2 \in [m] \setminus \set{i}} M_{ij_1} M_{ij_2} \upzeta{1}_{j_1}\upzeta{2}_{j_2}.$$
Using the fact that the columns of $A$ are incoherent, for this quadratic form we get that $\norm{B}_F =\widetilde{\Omega}(\sqrt{m})$. We can then apply Hanson-Wright inequality to this quadratic form, and conclude that the $\abs{Z} \le \sqrt{m} \poly\log(n)$ with high probability\footnote{The random variables $\upzeta{\ell}_{j}$ has its $\psi_2$ Orlicz-norm bounded by $K\le \log(1/p)=O(\log m)$; Hanson-Wright inequality shows that $\Pr[\abs{Z} > t] \le \exp\bigparen{-c \min\bigset{\frac{t^2}{K^4 \norm{B}_F^2}, \frac{t}{K^2 \norm{B}}} }$. Using Hypercontractivity for these distributions also gives similar bounds up to $\poly\log n$ factors.}. On the other hand, the $\zeta$ random variables are non-zero with probability at most $p=k/m$ and are $\tau=O(1)$-negatively correlated, and hence we get that $\Var[Z] \le m \sigma^4 (k/m)^2= \widetilde{O}(k^2/m)$ (and $\E[Z]=0$). Here $\sigma$ is the spectral norm of $A$. Hence, in the ideal case, we can hope to show a much better upper bound of $\abs{Z} \le k\poly\log(n)/\sqrt{m}$ (smaller by a factor of $k/m$).
Obtaining bounds that take advantage of the small probability of occurrence seems crucial in handling $k=\Omega(\sqrt{m})$ for the semirandom case, and $k=\omega(\sqrt{m})$ for the random case.   

To tackle this, we derive general concentration inequalities for multilinear degree-$d$ polynomials of rarely occurring random variables. 

\begin{iproposition}[Same as Proposition~\ref{prop:concentration-degree-d}]
Consider a degree $d$ multilinear polynomial $f$ in $\upzeta{1}, \dots, \upzeta{d} \in \R^m$ of the form
$$f(\upzeta{1}, \dots, \upzeta{d})= \sum_{(j_1, \dots, j_d) \in [m]^d} T_{j_1, \dots, j_d} \upzeta{1}_{j_1} \dots \upzeta{d}_{j_d},$$
where each of the random variables $\zeta_j$ are independent, bounded and non-zero with probability at most $p$. Further for any $\Gamma \subset [d]$, let $M_{\Gamma,\Gamma^c}$ be the $m^{|\Gamma|} \times m^{d-|\Gamma|}$ be the matrix obtained by flattening along $\Gamma$ and $[d] \setminus \Gamma$ respectively and 
\begin{equation} \label{eq:overview:flattening}
\imbal= \sum_{\Gamma \subset [d]}~ \frac{\norm{M_{\Gamma,\Gamma^c}}_{2 \to \infty}^2}{\norm{T}_F^2} \cdot p^{-|\Gamma|} = \Big(\frac{\norm{M_{\Gamma,\Gamma^c}}_{2 \to \infty}^2}{m^{d-|\Gamma|}}\Big) \Big( \frac{\norm{T}_F^2}{m^{d}}  \Big)^{-1} \cdot \frac{1}{(p m)^{|\Gamma|}}, 
\end{equation}
where $\norm{\cdot}_{2 \to \infty}$ is the maximum $\ell_2$ norm of the rows.
Then, 
for any $\eta>0$, 
we have
\begin{equation}\label{eq:concentration-degree-d}
\Pr\Big[ \abs{f(\upzeta{1}, \dots, \upzeta{d})} \ge \log (2/\eta)^{d}  \sqrt{\imbal} \cdot p^{d/2} \norm{T}_F \Big] \le \eta. 
\end{equation}
\end{iproposition}

Here $\imbal$ is a measure of how well-spread out the corresponding tensor $T$ is: it depends in particular, on the maximum row norm ($\norm{\cdot}_{2 \to \infty}$ operator norm) of different ``flattenings'' of the tensor $T$ into matrices. This is reminiscent of how the bounds of Latala~\cite{Latala,AdamczakWolff} depends on the spectral norm of different ``flattenings'' of the tensor into matrices, but they arise for different reasons. We defer to Section~\ref{subsec:conc-bounds} for a formal statement and more background. To the best of our knowledge, we are not aware of similar concentration bounds for arbitrary multilinear (with potentially non-negative co-efficients) for rarely occurring random variables, and we believe these bounds may be of independent interest in other sparse settings.


The analysis for both the semirandom case and random case proceeds by carefully analyzing various terms that arise in evaluating $\set{\cof_i : i \in [m]}$, and using Proposition~\ref{prop:concentration-degree-d} in the context of each of these terms along with good bounds on the norms of various tensors and their flattenings that arise (this uses sparsity of the samples, the incoherence assumption and the spectral norm bound among other things). We now describe one of the simpler terms that arise in the random case, to demonstrate the advantage of considering larger $L$ i.e., more fixed samples. Consider the expression
\begin{equation}\label{eq:overview:nofixterm}
Z = \sum_{\substack{(j_1,\dots,j_{2L-1})\\ \in [m]^{2L-1}}} M_{i,j_{2L-1}} \upzeta{2L-1}_{j_{2L-1}} \sum_{i_1, \dots, i_{L-1}} \E\big[x_i^2 x_{i_1}^2 \dots x_{i_{L-1}}^2 \big] \prod_{\ell \in [L-1]}M_{i_\ell,j_{2\ell-1}} M_{i_\ell, j_{2\ell}} \upzeta{2\ell-1}_{j_{2\ell-1}} \upzeta{2\ell}_{j_{2\ell}}.
\end{equation}
In the random case, $\E[x_i^2 x_{i_1}^2\dots x_{i_{L-1}}^2] \approx \E[x_i^2]\E[x_{i_1}^2]\dots \E[x_{i_{L-1}}^2] \le (k/m)^L$, since the support distribution is essentially random (this also assumes the value distribution is Rademacher). Further, for the corresponding tensor $T$ of co-efficients, one can show a bound of $\norm{T}_F = O\bigparen{m^{(L-1)/2}}$. Hence, applying Proposition~\ref{prop:concentration-degree-d}, we would get an ideal bound (assuming the imbalance factor $\imbal=O(1)$ ) of roughly $c \cdot (k/m)^{L} \sqrt{m}^{L-1}\cdot (k/m)^{L-1/2}= c \bigparen{\tfrac{k^2}{m\sqrt{m}}}^{L-1} \cdot (k/m)^{3/2} $, which becomes $o(k/(m\sqrt{m}))$ as required for $L$ being a sufficiently large constant when $k=o(m^{3/4 - \eps})$ \footnote{The bound that we actually get in this case is off by a $c=\sqrt{m} \poly\log n$ factor since $\imbal=\omega(1)$, but this also becomes small for large $L$.}. 
On the other hand, with higher values of $L$ there are some lower-order terms that start becoming larger comparatively, for which Proposition~\ref{prop:concentration-degree-d} becomes critical. Balancing out these terms allows us to handle a sparsity of $k=\widetilde{O}(m^{2/3})$ for the random case. This is done in Section~\ref{sec:random}. 

The semirandom model presents several additional difficulties as compared to the random model. Firstly, as most of the data is generated with arbitrary supports, we cannot assume that the $x$ variables are $\tau = O(1)$-negatively correlated. As a result, the term $\E[x_i^2 x_{i_1}^2\dots x_{i_{L-1}}^2]$ does not factorize as the adversary can make the joint probability distribution of the non-zeros very correlated. 
Hence, to bound various expressions that appear in the expansion of $\gamma_i$, we need to use inductive arguments to upper bound the magnitude of each inner sum and eliminating the corresponding running index (this needs to be done carefully since these quantities can be negative). We bound each inner sum using Proposition~\ref{prop:concentration-degree-d}, 
using the fact that $\sum_{i_d \in [m]} \E[x_i^2 x_{i_1}^2 \dots x_{i_d}^2] \le k \E[x_i^2 x_{i_1}^2 \dots x_{i_{d-1}}^2]$, and some elegant linear algebraic facts. This is done in Section~\ref{subsec:instan-conc}.   

Finally, the above procedure can be used to recover all the columns $A_i$ of the dictionary whose corresponding occurrence probabilities $q_i=\E[x_i^2]$ are close to the largest  i.e., $q_i =\widetilde{\Omega}( \max_{j \in [m]} q_{j})$. To recover all the other columns, we use a linear program and subsample the data (just based on columns recovered so far), so that one of the undiscovered columns has largest occurrence probability. We defer to the details in Sections~\ref{sec:semi-random-recovery} and \ref{sec:semi-random-full-alg}.     

\subsection{Related Work}\label{sec:related}

\paragraph{Polynomial Time Algorithms.}
Spielman et al.~\cite{SWW} were the first to provide a polynomial time algorithm with rigorous guarantees for dictionary learning. They handled the full rank case, i.e, $m=n$, and assumed the following distributional assumptions about $X$: each entry is chosen to be non-zero independently with probability $k/m=O(1)/\sqrt{n}$  (the support distribution $\Ds$ is essentially uniformly random) and conditioned on the support, each non-zero value is set independently at random from a sub-Gaussian distribution e.g.,  Rademacher distribution (the value distribution $\Dv$). 
Their algorithm uses the insight that w.h.p. in this model, 
the sparsest vectors in the row space of $Y$ correspond to the rows of $X$, and solve a sequence of LPs to recover $X$ and $A$. 
Subsequent works~\cite{luh2015random,blasiok2016improved,qu2014finding} have focused on improving the sample complexity and sparsity assumptions in the full-rank setting. 
However in the presence of the semirandom adversary, the sparsest vectors in the row space of $Y$ may not contain rows of $X$ and hence the algorithmic technique of~\cite{SWW} breaks down. 

For the case of over-complete dictionaries the works of Arora et al.~\cite{AGMM14} and Agarwal et al.~\cite{agarwal2013exact} provided polynomial time algorithms when the dictionary $A$ is $\mu$-incoherent. 
\anote{removed the incoherence assumption, since it should be in prelims} 
In particular, the result of~\cite{AGMM14} also holds under a  weaker assumption that the support distribution $\Ds$ is approximately $\ell=O(1)$-wise independent i.e., $\Pr_{x \sim \Ds}[i_1,i_2,\dots,i_\ell \in \supp(x)]\le \tau^\ell (k/m)^\ell$ for some constant $\tau>0$.   
Under this assumption they can handle sparsity up to $\widetilde{O}(\min(\sqrt{n},m^{1/2-\eps}))$ for any constant $\eps>0$ with $\ell=O(1/\eps)$. 
Their algorithm computes a graph $G$ over the samples in $Y$ by connecting any two samples that have a high dot product -- these correspond to pairs of samples whose supports have at least one column in common. 
Recovering columns of $A$ then boils down to identifying communities in this graph with each community identifying a column of $A$. Subsequent works have focused on extending this approach to handle mildly weaker or incomparable assumptions on the dictionary $A$ or the distribution of $X$~\cite{arora2014more,AGMM15}. For example, the algorithm of \cite{AGMM15} only assumes $O(1)$-wise independence on the non-zero values of a column $x$. 
The state of the art results along these lines can handle $k=\widetilde{O}(\sqrt{n})$ sparsity for $\mu=\widetilde{O}(1)$-incoherent dictionaries. 
Again, we observe that in the presence of the semirandom adversary, the community structure present in the graph $G$ could become very noisy and one might not be able to extract good approximations to the columns of $A$, or worse still, find spurious columns.

The work of Barak at al.~\cite{BKS15} reduce the problem of recovering the columns of $A$ to a (noisy) tensor decomposition problem, which they solve using Sum-of-Squares~(SoS) relaxations. 
Under assumptions that are similar to that of~\cite{AGMM14} (assuming approximate $\widetilde{O}(1)$-wise independence), these algorithms based on SoS relaxations~\cite{BKS15,ma2016polynomial} handle almost linear sparsity $k=\widetilde{O}(n)$ and recover incoherent dictionaries with quasi-polynomial time guarantees in general, and polynomial time guarantees when $\sigma=O(1)$ 
(this is obtained by combining Theorem 1.5 in \cite{ma2016polynomial} with \cite{BKS15}). The recent work of Kothari et al.~\cite{kothari2017outlier} also extended these algorithms based on tensor decompositions using SoS, to a setting when a small fraction of the data can be adversarially corrupted or arbitrary. This is comparable to the setting in the semirandom model when $\beta=1-\epsilon$~(for a sufficiently small constant $\epsilon$), but the non-zero values for these samples can also be arbitrary. \anote{edited: 5/5/18.}However in the semirandom model, the reduction from dictionary learning to tensor decompositions breaks down because the supports can have arbitrary correlations in aggregate, particularly when $\beta$ is small. Hence these algorithms do not work in the semirandom model.

Moreover, even in the absence of any adversarial samples, Theorem~\ref{ithm:random-algorithm} and the current state-of-the-art guarantees~\cite{ma2016polynomial,AGMM15} are incomparable, and are each optimal in their own setting. For instance, consider the setting when the over-completeness $m/n, \sigma =O(n^{\eps})$ for some small constant $\eps>0$. In this case, Arora et al.~\cite{AGMM15} can handle a sparsity of $\widetilde{O}(\sqrt{n})$ in polynomial time and Ma et al.~\cite{ma2016polynomial} handle $\widetilde{O}(n)$ sparsity in quasi-polynomial time, while Theorem~\ref{ithm:random-algorithm} handles a sparsity of $\widetilde{O}(n^{2/3})$ in polynomial time. On the other hand, \cite{AGMM15} has a better dependence on $\sigma$, while \cite{ma2016polynomial} can handle $\widetilde{O}(n)$ sparsity when $\sigma=O(1)$. Further, both of these prior works do not need full independence of the value distribution $\Dv$
and the SoS-based approaches work even under mild incoherence assumptions to give some weak recovery guarantees\footnote{However, to recover $A$ and $X$ to high accuracy, incoherence and RIP assumptions of the kind assumed in our work and \cite{AGMM15} seem necessary.}  However, we recall that in addition our algorithm works in the semirandom model (almost arbitrary support patterns) up to sparsity $\widetilde{O}(\sqrt{n})$, and this seems challenging for existing algorithms. 


\paragraph{Heuristics and Associated Guarantees.} Many iterative heuristics like $k$-SVD, method of optimal direction~(MOD), and alternate minimization have been designed for dictionary learning, and recently there has also been interest in giving provable guarantees for these heuristics. 
Arora et al.~\cite{AGMM14} and Agarwal et al.~\cite{Alekhsparse} gave provable guarantees for $k$-SVD and alternate minimization assuming initialization with a close enough dictionary. 
Arora et al.~\cite{AGMM15} provided guarantees for a heuristic that at each step computes the current guess of $X$ by solving sparse recovery, and then takes a gradient step of the objective $\|Y-AX\|^2$ to update the current guess of $A$. They initialize the algorithm 
using a procedure that finds the principal component of the matrix $E[\iprod{u^{(1)},y}\iprod{u^{(2)},y}~yy^T]$ for appropriately chosen samples $u^{(1)},u^{(2)}$ from the data set. A crucial component of our algorithm in the semirandom model is a procedure to generate candidate vectors for the columns of $A$ and is inspired by the initialization procedure of~\cite{AGMM15}.

\paragraph{Identifiability Results.}
As with many statistical models, most identifiability results for dictionary learning follow from efficient algorithms. As a result identifiability results that follow from the results discussed above rely on strong distributional assumptions. On the other hand results establishing identifiability under deterministic conditions~\cite{aharon2006uniqueness,georgiev2005sparse} require exponential sample complexity as they require that every possible support pattern be seen at least once in the sample, and hence require $O(m^k)$ samples. 
To the best of our knowledge, our results (Theorem~\ref{ithm:rad:identifiability}) lead to the first identifiability results with polynomial sample complexity without strong distributional assumptions on the supports.

\paragraph{Other Related Work.} 

A problem which has a similar flavor to dictionary learning is Independent Component Analysis (ICA), which has been a rich history in signal processing and computer science~\cite{ICA1,ICA2,GVX14}. Here, we are given $Y=AX$ where each entry of the matrix $X$ is independent, and there are polynomial time algorithms both in the under-complete~\cite{ICA2} and over-complete case~\cite{Cardoso,GVX14} that recover $A$ provided each entry of $X$ is non-Gaussian. However, these algorithms do not apply in our setting, since the entries in each column of $X$ are not independent (the supports can be almost arbitrarily correlated because of the adversarial samples).   

Finally, starting with the works of Blum and Spencer~\cite{BS92}, semirandom models have been widely studied for various optimization and learning problems. Feige and Kilian~\cite{FK99} considered semi-random models involving monotone adversaries for various problems including graph partitioning, independent set and clique. Semirandom models have also been studied in the context of unique games~\cite{KMM}, graph partitioning problems~\cite{MMV12, MMV14} and learning communities~\cite{PW15,MPW15,MMVSBM}, correlation clustering~\cite{MS10,MMVCC}, noisy sorting~\cite{MMVfas}, coloring~\cite{DF16} and clustering~\cite{AV18}.

\section{Preliminaries}
\label{sec:prelims}
We will use $A$ to denote an $n \times m$ over-complete~($m > n$) dictionary with columns $A_1, A_2, \ldots A_m$. Given a matrix or a higher order tensor $T$, we will uses $\|T\|_F$ to denote the Frobenius norm of the tensor. For matrices $A$ we will use $\|A\|_2$ to denote the spectral norm of $A$.  We first define the standard random model for generating data from an over-complete dictionary. 

Informally, a vector $y=Ax$ is generated as a random linear combination of a few columns of $A$. We first pick the support of $x$ according to a {\em support distribution} denoted by $\Ds$, and then draw the values of each of the non-zero entries in $x$ independently according to the {\em value distribution} denoted by $\Dv$. 
$\Ds$ is a distribution that is over the set of vectors in $\{0,1\}^m$ with at most $k$ ones. 

\paragraph{Value Distribution:}
As is standard in past works on sparse coding~\cite{AGMM14,AGMM15}, we will assume that the value distribution $\Dv$ is any mean zero symmetric distribution supported in $[-C,-1] \cup [1,C]$ for a constant $C > 1$. This is known as the {\em Spike-and-Slab} model~\cite{goodfellow2012large}. For technical reasons we also assume that $\Dv$ has non-negligible density in $[1,1+\eta]$ for $\eta ={1}/{(\poly\log n)}$. Formally we assume that 
\begin{align}
\exists \gamma_0 \in (0,1) \text{ s.t. } \forall \eta \geq \frac{1}{\log^c n}, \Psymb_{\Dv}([1,1+\eta]) \geq \gamma_0.
\label{eq:gamma-0-assumption}
\end{align}
In the above definition, we will think of $\gamma_0$ as just being non-negligible (e.g., $1/\poly(n)$). This assumption is only used in Section~\ref{sec:testing}, and the sample complexity will only involve inverse polynomial dependence on $\gamma_0$. 
The above condition captures the fact that the value distribution has some non-negligible mass close to $1$ \footnote{If the value distribution has negligible mass in $[1,1+\eta] \cup [-1-\eta,-1]$, one can arguably rescale the value distribution by $(1+\eta)$ so that all of the value distribution is essentially supported on $[1, C/(1+\eta)] \cup [-C/(1+\eta), -1]$.}.
Further, this is a benign assumption that is satisfied by many distributions including the Rademacher distribution that is supported on $\{+1,-1\}$ (with $\gamma_0=1/2$), and the uniform distribution over $[-C,-1] \cup [1,C]$ (with $\gamma_0=1/(2C)$).

\paragraph{Random Support Distribution $\Dsr$.}

Let $\xi \in \mathbb{R}^m$ be drawn from $\Dsr$. To ensure that each column appears reasonably often in the data so that recovery is possible information theoretically we assume that each coordinate $i$ in $\xi$ is non-zero with probability $\frac{k}{m}$. We do not require the non-zero coordinates to be picked independently and there could be correlations provided that  they are negatively correlated up to a slack factor of $\tau$.
\begin{definition}
For any $\tau\ge 1$, a set of non-negative random variables $Z_1, Z_2, \dots, Z_m$ where $P(Z_i \neq 0) \leq p$ is called $\tau$-negatively correlated if for any $i \in [m]$ and any $S \subseteq [m]$ such that $i \notin S$ and $|S| = O(\log m)$ we have that for a constant $\tau > 0$,
\begin{align}
P \big(Z_i \neq 0 \big| \bigcap_{j \in S} Z_j \neq 0 \big) \leq \tau p.
\label{eq:support-assumption-1}
\end{align}
 \pnote{Relate it to $(k,\tau)$ nice distributions of BKS.} 
\end{definition}
In the random model the variables $\xi_1, \xi_2, \dots, \xi_m$ are $\tau$-negatively correlated with $p = \frac{k}{m}$. We remark that for our algorithms we only require the above condition (for the random portion of the data) to hold for sets $S$ of size up to $O(\log m)$. Of course in the semi-random model described later, the adversary can add additional data from supports distributions with arbitrary correlations; hence they are not $\tau$-negatively correlated, and each co-ordinate of $x$ need not be non-zero with probability at most $p=k/m$. 


\paragraph{Random model for Dictionary Learning.}

Let $\Dsr \odot \Dv$ denote the distribution over $\mathbb{R}^m$ obtained by first picking a support vector from $\Dsr$ and then independently picking a value for each non zero coordinate from $\Dv$. Then we have that a sample $y$ from the over complete dictionary is generated as

$$
y = \sum_{i \in [m]} x_i A_i,
$$ 
where $(x_1, x_2, \dots, x_m)$ is generated from $\Dsr \odot \Dv$. Given $S = \{\ysamp{1}, \ysamp{2}, \dots,\ysamp{N} \}$ drawn from the model above, the goal in standard dictionary learning is to recover the unknown dictionary $A^*$, up to signs and permutations of columns. 


\subsection{Semi-random model}
\label{subsec:model}
We next describe the {\em semi-random} extension of the above model for sparse coding. In the semi-random model an initial set of samples
is generated from the standard model described above. A semi-random adversary can then an arbitrarily number of additional samples with each sample 
 $y=Ax$ generated by first picking the support of $x$ arbitrarily and then independently picking values of the non-zeros according to $\Dv$. 
 Formally we have the following definition
\begin{definition}[Semi-Random Model: $\sModel$] 
\label{def:semi-random}
A semi-random model for sparse coding, denoted as $\sModel$, is defined via the following process of producing $N$ samples
\begin{enumerate}
\item Given a $\tau$-negatively correlated support distribution $\Dsr$,  $N_0=\beta N$ ``random'' support vectors $\upxi{1}, \upxi{2}, \dots, \upxi{N_0}$ are generated from $\Dsr$.
\item Given the knowledge of the supports of $\upxi{1}, \dots, \upxi{N_0}$, the semi-random adversary generates $(1-\beta)N$ additional support vectors $\upxi{N_0+1}, \upxi{N_0+2}, \dots, \upxi{N}$ from an arbitrary distribution $\tDs$. The choice of $\tDs$ can depend on $\upxi{1}, \upxi{2}, \dots, \upxi{N_0}$.
\item Given a value distribution $\Dv$ that satisfies the Spike-and-Slab model, the vectors $x^{(1)},x^{(2)},\dots,x^{(N_0)},x^{(N_0+1)},\dots,x^{(N)}$ are form by picking each non-zero value (as specified by $\upxi{1}, \dots, \upxi{N}$ respectively) independently from the distribution $\Dv$.  
\item $x^{(1)},x^{(2)},\dots,x^{(N)}$ are randomly reordered as columns of an $m \times N$ matrix $X$. Then the output of the model is $Y=AX$.
\end{enumerate}
\end{definition}
We would like to stress that the amount of semi-random data can overwhelm the initial random set. In other words, $\beta$ need not be a constant and can be a small inverse polynomial factor. The number of samples needed for our algorithmic results will have an inverse polynomial dependence on $\beta$. While the above description of the model describes a distribution from which samples can be drawn, one can also consider a setting where there a fixed number of samples $N$, of which $\beta N=N_0$ samples were drawn with random supports i.e., from $\Dsr$. These two descriptions are essentially equivalent in our context since the distribution $\tDs$ is arbitrary. However, since there are multiple steps in the algorithm, it will be convenient to think of this as a generative distribution that we can draw samples from (in the alternate view, we can randomly partition the samples initially with one portion for each step of the algorithm).     
\anote{Added the equivalence statement. }

\begin{definition}[Marginals and Expectations]
\label{def:q}
Given $(x_1, x_2, \dots, x_m)$ generated from $\tDs \odot \Dv$ and a subset of indices $i_1, i_2, \dots, i_R \in [m]$ we denote $q_{i_1, i_2, \dots, i_R}$ as the marginals of the support distribution, i.e.
\begin{align}
q_{i_1, i_2, \dots, i_R} = P_{\tDs}(\xi_{i_1} \neq 0 \text{ and, } \xi_{i_2} \neq 0 \text{ and, } \dots, \xi_{i_R} \neq 0).
\label{eq:q-definition}
\end{align}
Here $\tDs$ is an arbitrary distribution over $k$-sparse vectors in $\{0,1\}^m$, and the notation $P_{\tDs}$ denotes that the randomness is over the choice of the support distribution and not the value distribution. We will also be interested in analyzing low order moments of subsets of indices w.r.t. the value distribution $\Dv$. Hence we define
\begin{align}
q_{i_1, i_2, \dots, i_R}(d_1, d_2, \dots, d_R) = E_{\tDs \odot \Dv}[x^{d_1}_{i_1}x^{d_2}_{i_2} \dots x^{d_R}_{i_R}].
\label{eq:qd-definition}
\end{align}
\end{definition}
Here $d_1, d_2, \dots, d_R \geq 0$. Notice that the above expectation is non-zero only if all $d_j$s are even numbers. This is because conditioned on the support the values are drawn independently from a mean $0$ distribution. Furthermore, it is easy to see that when all $d_j$s are even we have that
\begin{align}
1 \leq q_{i_1, i_2, \dots, i_R}(d_1, d_2, \dots, d_R) \leq C^{\sum_{j=1}^R d_j} q_{i_1, i_2, \dots, i_R}
\label{eq:qd-bound}
\end{align}
We next state two simple lemmas about the marginals and expectations defined above that we will use repeatedly in our algorithmic results and analysis. The proofs can be found in the Appendix.
\begin{lemma}
\label{lem:q-sum}
For any $R \geq 2$ and any subset of indices $i_1, i_2, \dots, i_R \in [m]$ we have that
$$
1 \leq \sum_{i_R \in [m]} \frac{q_{i_1, i_2, \dots, i_R}}{q_{i_1, i_2, \dots, i_{R-1}}} \leq k.
$$
Furthermore, if the support distribution satisfies \eqref{eq:support-assumption-1} then we also have that
$$
\frac{q_{i_1, i_2, \dots, i_R}}{q_{i_1, i_2, \dots, i_{R-1}}} \leq \frac{\tau k}{m}.
$$
\end{lemma}
\begin{lemma}
\label{lem:q-d-sum}
For any $R \geq 2$, any subset of indices $i_1, i_2, \dots, i_R \in [m]$ and any even integers $d_1, d_2, \dots, d_R$ we have that
\begin{align}
1 \leq \sum_{i_R \in [m]} \frac{q_{i_1, i_2, \dots, i_R}(d_1, d_2, \dots, d_R)}{q_{i_1, i_2, \dots, i_{R-1}}(d_1, d_2, \dots, d_{R-1})} \leq k C^{d_R}.
\end{align}
Furthermore, if the support distribution satisfies \eqref{eq:support-assumption-1} then we also have that
$$
\frac{q_{i_1, i_2, \dots, i_R}(d_1, d_2, \dots, d_R)}{q_{i_1, i_2, \dots, i_{R-1}}(d_1, d_2, \dots, d_{R-1})} \leq \frac{\tau k C^{d_R}}{m}.
$$
\end{lemma}

\subsection{Properties of the dictionary}
\label{sec:prelim-dict}

\anote{Define the different distributions.}


Our results on dictionary learning will make two assumptions on the structure of the unknown dictionary. These assumptions, namely {\it incoherence} and {\it Restricted Isometry Property} are standard in the literature on sparse recovery and dictionary learning. Next we define the two assumptions, discuss relationships among them and state simple consequences that will be used later in the analysis. All the proofs can be found in the Appendix.
\begin{definition}[Incoherence]
We say that an $n \times m$ matrix with unit length columns is $\mu$-incoherent if for any two columns $A_i, A_j$, we have that
$$
\iprod{A_i,A_j} \leq \frac{\mu}{\sqrt{n}}
$$
\end{definition}
The $\sqrt{n}$ factor above is a natural scaling since random $n \times m$ matrices are $O(\sqrt{\log m})$ incoherent. Notice that every matrix is $\sqrt{n}$-incoherent. Hence values of $\mu = o(\sqrt{n})$ provide non-trivial amount of incoherence. In general, smaller values of $\mu$ force the columns of $A$ to be more uncorrelated. In this work we will think of $\mu$ as $\poly\log n$\footnote{Although our results also extend to values of $\mu$ upto $n^{\epsilon}$ for a small constant $\epsilon$. The sparsity requirement will weaken accordingly.}.
Next we state a simple lemma characterizing the spectral norm of incoherent matrices.
\begin{lemma}
\label{lem:spectral-norm-incoherent}
Let $A$ be an $n \times m$ matrix with unit length columns that is $\mu$-incoherent. Then we have that
$$
\sqrt{\frac m n} \leq \|A\|_2 \leq \sqrt{1+\frac{m \mu}{\sqrt{n}}}
$$
\end{lemma}
\begin{definition}[Restricted Isometry Property~(RIP)]
We say that an $n \times m$ matrix satisfies $(k,\delta)$-RIP if for any $k$-sparse vector $x \in \R^m$, we have that
$$
(1-\delta) \leq \frac{\|Ax\|}{\|x\|} \leq (1+\delta).
$$
\end{definition}
In other words, RIP matrices preserve norms of sparse vectors. In this work we will be interested in matrices that satisfy $(k,\delta)$-RIP for $\delta = 1/\poly\log n$. It is well known that a random $n \times m$ matrix will be $(k,\delta)$-RIP when $k \le O(\delta n /\log(\frac{n}{\delta k}))$~\cite{Baraniuk}. The following lemma characterizes the spectral norm of matrices that have the RIP property.
\begin{lemma}
\label{lem:spectral-norm-rip}
Let $A$ be an $n \times m$ matrix with unit length columns that satisfies the $(k,\delta)$-RIP property. Then we have that
$$
\sqrt{\frac m n} \leq \|A\|_2 \leq (1+\delta) \sqrt{\frac m k}.
$$
\end{lemma}
The two notions of incoherence and RIP are also intimately related to each other. The next well known fact shows that when $k = o(\sqrt{n})$, incoherence implies the $(k,\delta)$-RIP property. 
\begin{lemma}
\label{lem:incoherence-implies-rip}
Let $A$ be an $n \times m$ matrix with unit length columns that is $\mu$-incoherent. Then for any 
$\delta \in (0,1)$, we have that $A$ also satisfies $(k,\delta)$-RIP for $k= \frac{\sqrt{n} \delta}{2\mu}$.
\end{lemma}
In fact incoherent matrices are one of a handful ways to explicitly construct RIP matrices~\cite{bourgain2011explicit}. Conversely we have that RIP matrices have incoherent columns for non-trivial values of $\mu$. In fact the following lemma implies a much stronger statement, and will be used in Section~\ref{sec:testing} for analyzing the test procedure. 
\begin{lemma}\label{lem:incoherencefromRIP}
Let $A$ be an $n \times m$ matrix that satisfies the $(k,\delta)$-RIP property for $\delta < 1$. Then for any column $i \in [m]$ and $(k-1)$ other columns $T \subset [m]$, we have
$$ \sum_{j \in T} \iprod{A_i, A_j}^2 \le 2\delta+\delta^2. $$
\end{lemma}

\anote{Can we just club this with the incoherence from RIP lemma?}

We next state a useful consequence of the RIP property that we will crucially rely on in our testing procedure in Section~\ref{sec:testing}.

\begin{lemma}\label{lem:largevalues}
Let $A$ be a $(k,\delta)$-RIP matrix and $z$ be any unit vector. Then for any $\gamma$ with $\frac{1}{\sqrt{k-1}}<\gamma<1$,
\begin{align*}
\forall T \subseteq [m]~\text{ s.t. } |T|\le k,~ \sum_{i \in T} \iprod{z,A_i}^2 & \le 1+\delta,~~\text{ and }\\
\Abs{ \Set{ i \in [m]: \abs{\iprod{z,A_i}}> \gamma}} &< \frac{1+\delta}{\gamma^2}. 
\end{align*}
\end{lemma}

\section{Testing Procedure and Identifiability} \label{sec:testing}

\anote{Should we change $\eta$ to $\gamma$ and vice-versa to be consistent with the next section?}

In this section we describe and prove the correctness of our testing procedure that checks if a given unit vector $z$ is close to any column of the dictionary $A$. 
\anote{Need to set what $\eta$ is. }
The procedure works as follows: it takes a value $\eta$ as input  and checks if the inner product $|\iprod{z,Ax}|$ only takes values in $[0, \eta] \cup [1-\eta, C(1+\eta)]$ for most samples $x$, and if $|\iprod{z,Ax}| \in [1-\eta,C(1+\eta)]$ for a non-negligible fraction of samples. In other words, a vector $z$ 
is rejected only if $|\iprod{z,Ax}| \in (2\eta, 1-2\eta)$ for a non-negligible fraction of the samples, or if $|\iprod{z,Ax}| \in [1-\eta, C(1+\eta)]$ for a negligible fraction of samples. For any $\eta \in (0,1)$, we will often use the notation $I_\eta$ to denote the set $\set{t \in \R:  \abs{t} \in [1-\eta,C(1+\eta)] \cup [0, \eta]}$, i.e. the range of values close to $0$ or $1$.

 \begin{figure}[htb]
 \begin{center}
 \fbox{\parbox{0.98\textwidth}{
{\bf Algorithm \textsc{TestColumn}}$(z,Y=\set{\ysamp{1},\dots,\ysamp{N}}, \kappa_0, \kappa_1,\eta)$
\begin{enumerate}
\item Let $\widetilde{\kappa}_1$ be the fraction of samples such that $\abs{\iprod{z,\ysamp{r}}} \in [1-\eta,C(1+\eta)]$ and $\widetilde{\kappa}_0$ be the fraction of samples such that $\abs{\iprod{z,\ysamp{r}}} \notin [1-\eta,C(1+\eta)] \cup [0,C\eta]$.
\item If $\widetilde{\kappa}_0 < \kappa_0$ and $\widetilde{\kappa}_1 \ge \kappa_1$, return $(\text{YES}, \widehat{z})$, where $z' = \text{mean}\big(\set{ \ysamp{r}: r \in [N] \text{ s.t. } \iprod{\ysamp{r},z}  \ge \tfrac{1}{2} } \big)$ and $\widehat{z} = z'/\norm{z'}_2$.
\item Else return $(\text{NO},\emptyset)$.
\end{enumerate}
 }}
 \end{center}
 \caption{\label{ALG:test}}
 \end{figure}

We show the following guarantees for Algorithm {\sc TestColumn}.
We will prove the guarantees in a slightly broader setup so that it can be used both for the identifiability results and for the algorithmic results. We assume that we are given $N$ samples $\set{\ysamp{r}=A \xsamp{r}: r\in [N]}$, when the value distribution (distribution of each non-zero co-ordinate of a given sample $\xsamp{r}$) is given by $\Dv$ (see \eqref{eq:gamma-0-assumption} in Section~\ref{sec:prelims}). We make the following mild assumption about the sparsity pattern (support); for any $i$ and any $T \subset [m] \setminus \set{i}$, we assume that there are at least $q_{\min} N$ samples which contain $i$ but do not contain $T$ in the support. Note that for the semi-random model, if $\beta$ fraction of the samples come from the random portion, then $q_{\min} \ge \tfrac{1}{2}\beta k/m $ with high probability. 

In what follows, it will be useful to think of $\eta=O(1/\poly\log(n)), \gamma_0 = n^{-\Omega(1)}$, the desired accuracy $\eta_0 = 1/\poly(n)$, sparsity $k=O(n/\poly\log(n))$, and the desired failure probability to be $\gamma=\exp(-n)$. Hence, in this setting $\kappa_0 = n^{-\Omega(1)}$ and $\delta = O(1/\poly\log(n))$ as well.   

\begin{theorem}[Guarantees for {\sc TestColumn}]\label{thm:sr:testing}
There exists constants $c_0, c_1,c_2,c_3,c_4,c_5>0$ (potentially depending on $C$) such that the following holds for any $\gamma\in (0,1), \eta_0<\eta \in (0,1)$ satisfying $\sqrt{\frac{c_3 k}{m}} <\eta <\frac{c_1}{\log^2\bigparen{\tfrac{mn}{q_{\min} \eta_0}}}$. Set $\kappa_0:= c_4 \gamma_0 \eta q_{\min}/(km)$.
Suppose we are given $N \ge \frac{c_2 knm \log(1/\gamma)}{\eta_0^3 \gamma_0 \kappa_0}$ samples $\ysamp{1}, \dots, \ysamp{N}$ satisfying
\begin{itemize}
\item the dictionary $A$ is $(k,\delta)$-RIP for $\delta< \bigparen{\frac{\eta}{16C \log(1/\kappa_0) }}^2$,
\item $\forall i \in [m], T \subset [m] \setminus \set{i}$ with $|T| \le c_3/\eta^2$, there at least $q_{\min} N$ samples whose supports all contain $i$, but disjoint from $T$.
\end{itemize}
Suppose we are given a unit vector $z\in \R^n$, then $\text{\sc TestColumn}(z,\set{\ysamp{1},\dots,\ysamp{N}},2\kappa_0,\kappa_1= c_5 q_{\min} \gamma_0 \eta, \eta)$ runs in times $O(N)$ time, and we have with probability at least $1-\gamma$ that 
\begin{itemize}
\item {\em (Completeness)} if $\norm{z- bA_i}_2 \le \eta'=\eta/(8C \log(1/\kappa_0))$ for some $i \in [m], b \in \set{-1,1}$, then Algorithm {\sc Testvector} outputs $(\text{YES}, \widehat{z})$. 
\item {\em (Soundness)} if unit vector $z \in \R^n$ passes {\sc TestColumn}, then there exists $i \in [m], b \in \set{-1,1}$ such that $\norm{z-bA_i}_2 \le \sqrt{8\eta}$. Further, in this case  $\norm{\widehat{z} - bA_i}_2 \le \eta_0$. 
\end{itemize}
\end{theorem}

\begin{remark}
We note that the above algorithm is also robust to adversarial noise. In particular, if we are given samples of the form $\ysamp{r}=A \xsamp{r}+\psi^{(r)}$, where $\norm{\psi^{(r)}}_2 \le O(\eta)$, then it is easy to see that the completeness and soundness guarantees go through since the contribution to $\iprod{\ysamp{r},z}$ is at most $\abs{\iprod{\psi^{(r)},z}} \le \norm{\psi^{(r)}} =O(\eta)$.   
\end{remark}

The above theorem immediately implies an identifiability result for the same model (and hence the semi-random model). By applying Algorithm {\sc TestColumn} to each $z$ in an $\widetilde{\Omega}(\eta)$-net over $\R^n$ dimensional unit vectors and choosing $\gamma= \exp\bigparen{-\Omega(n \log(1/\eta))}$ in Theorem~\ref{thm:sr:testing} and performing a union bound over every candidate vector $z$ in the net, we get the following identifiability result as long as $k<n/\poly\log(n)$. 

\begin{corollary}[Identifiability for Semi-random Model]\label{corr:sr:identifiability}
There exists constants $c_0, c_1,c_2,c_3,c_4,c_5,c_6>0$ (potentially depending on $C$) such that the following holds for any $k< n/\log^{2c_1} m$, $\eta_0 \in (0,1)$. 
Set $\kappa_0:= c_0 \gamma_0 \log^{-c_1} m q_{\min}$.
Suppose we are given $N \ge \frac{c_2 knm \log^{c_1} m\log(1/\kappa_0)}{\eta_0^3 \gamma_0 q_{\min}}$ samples $\ysamp{1}, \dots, \ysamp{N}$ satisfying
\begin{itemize}
\item the dictionary $A$ is $(k,\delta)$-RIP for $\delta< \frac{c_5}{\log(1/\kappa_0) \log^{c_6} m}$,
\item $\forall i \in [m], T \subset [m] \setminus \set{i}$ with $|T| \le c_4 \log^{2c_1} m$, there at least $q_{\min} N$ samples whose supports all contain $i$, but disjoint from $T$.
\end{itemize}
Then there is an algorithm that with probability at least $1-\exp(-n)$ finds the columns $\widehat{A}$ such that $\norm{\widehat{A}_i - b_i A_i}_2 \le \eta_0$ for some $b \in \set{-1,1}^m$. 
\end{corollary}



The second condition in the identifiability statement is a fairly weak condition on the support distribution. Lemma~\ref{lem:semirandom:randomsamples} for instance shows that a subset of samples which satisfy a weak notion of pairwise independence in the samples suffices for this to hold. 

 The guarantees for the test (particularly the soundness analysis) relies crucially on an anti-concentration statement for weighted sums of independent symmetric random variables, which may be of independent interest. 
Assuming that a weighted sum of independent random variables that are symmetric and bounded take a value close to $t$ with non-negligible probability $\kappa$, then we would like to conclude that it also takes values in $[t/3, 2t/3]$ with non-negligible probability that depends on $\kappa$. Central limit theorems like the Berry-Esse\'en theorem together with Gaussian anti-concentration imply such a statement when $\kappa$ is large e.g., $\kappa=\Omega(1)$; 
however even when the weights are all equal, they do not work when when $\kappa=1/\poly(n,m) \ll 1/\sqrt{k}$, which is our main setting of interest (this interest of $\kappa$ corresponds to the tail of a Gaussian, as opposed to the central portion of a Gaussian where CLTs can be applied for good bounds).


We first recall the Berry-Esse\'{e}n central limit theorem (see e.g., \cite{BerryEsseen}).

\begin{theorem}[Berry-Esse\'{e}n]\label{thm:berryesseen}
Let $Z_1, \dots, Z_\ell$ be independent r.v.s satisfying $\E[Z_i]=0$, $|Z_i| \le \eta ~\forall i \in [\ell]$, and $\sum_{i \in [\ell]} \E[Z_i^2] = 1$. If $Z=\sum_{i=1}^\ell Z_i$ and $F$ is the cdf of $Z$ and $\Phi$ is the cdf of a standard normal distribution, then 
$\sup_{x} |F(x) - \Phi(x)| \le \eta$.
\end{theorem}

The following is a simple consequence the Berry-Esse\'{e}n theorem by using the properties of a normal distribution. 
 \begin{fact}\label{fact:BE}
 Under the conditions of Theorem~\ref{thm:berryesseen}, for any $a<b$ we have
 $$\Pr_Z[ a \le Z \le b] \ge \Phi(b) - \Phi(a)- 2\eta.$$
 \end{fact}


We now proceed to the anti-concentration type lemma which will crucial in the analyis of the test; the setting of parameters that is of most interest to us is when $\kappa =1/\poly(n)$, $\eta =O(1/\poly\log(m))$ and $\ell = k$ (this corresponds to the support of a sample $x$).

\begin{lemma}\label{lem:weakanticonc}
For any constant $C\ge 1$, there exists constants $c_0=c_0(C) \in (0,1), c_1=c_1(C) \in (0,1)$, such that the following holds. Let $\eta' \in (0,\tfrac{1}{32C}), \kappa \in (0,1)$ and let $X_1, X_2, \dots, X_\ell$ be independent zero mean symmetric random variables taking any distribution over values in $[-C,-1] \cup [1,C]$ and $Z = \sum_{i=1}^\ell a_i X_i$ where $\norm{a}_2 = 1$ and $\norm{a}_\infty \le \eta$. For any $t \ge 1$ and $\beta \in (\tfrac{1}{16C}, \tfrac{7}{16C})$ with $\eta < c_1 t$, 
\begin{equation} 
\Pr_{X} \Big[ Z \in \big[(1-\eta')t , (1+\eta')Ct\big] \Big] \ge \kappa ~\implies~ \Pr_{X} \Big[ Z \in \big[\tfrac{\beta}{2}(1-\eta')t , \tfrac{3}{2} \beta (1+\eta')Ct\big]   \Big] \ge \min\Bigset{\frac{\kappa}{2}, c_0}. 
\end{equation}
\end{lemma}

In the above lemma, $c_0, c_1>0$ are appropriately chosen small constants such that  
$$c_0 = \min_{ r \in [1/(32C^2), 10 C] } \tfrac{1}{2}(\Phi(3r)-\Phi(r)) \ge \tfrac{1}{2}\left(\Phi\Big(10C+\tfrac{1}{16C^2}\Big)-\Phi(10C)\right),~ c_1 = \tfrac{c_0}{80 C^2}$$ 
where $\Phi(t)$ is the c.d.f. of a standard normal at $t>0$. Note that for our choice of $C$, the two intervals $[(1-\eta')t, (1+\eta')Ct]$ and $[\tfrac{\beta}{2}(1-\eta')t, \tfrac{3\beta}{2}(1+\eta')Ct]$ are non-overlapping. 
We also remark that the desired interval $[\beta t/2, 3C\beta/2]$ can be improved to a smaller interval around $[(\beta- \varepsilon) t, (\beta+\varepsilon) Ct ]$ with corresponding losses both in various constants.

\begin{proof}
 We will denote the two intervals of interest by $I'_1=[1-\eta', (1+\eta') C]$ and $I'_\beta= [\tfrac{1}{2}\beta (1-\eta') t, \tfrac{3}{2} \beta (1+\eta')Ct]$. Let $c' \ge 1$ be a sufficiently large absolute constant that depends on $C$ ($c'=40C^2$ suffices). We have two cases depending on how large $t$ is compared to the variance (remember $\norm{a}_2=1$). The first case when $t \ge c'$ corresponds to the tail of the distribution, while the second case when $t< c'$ corresponds to the central portion around the mean. 

\paragraph{Case $t \ge c'$:}
In this case, we will couple the event when $Z \in I'_1$ to the event when $Z \in I'_\beta$. Let $(x_1, \dots, x_\ell) \in ([-C,-1] \cup [1,C])^\ell$ be a fixed instantiation of $X$ and let $\lambda_i=a_i x_i$ and $\sum_i \lambda_i=\lambda$.

Choose a random partition $T \subseteq \set{1,\dots, \ell}$ by picking each index $i \in [\ell]$ i.i.d. with probability $\mu=1-2\beta$ . Let $Y_i$ be the corresponding indicator random variable; note that $\E[\sum_i \lambda_i Y_i]=\mu \lambda$ where $\mu \in (\tfrac{1}{8}, \tfrac{7}{8})$ for $C \ge 1$, and $\text{Var}[\sum_i \lambda_i Y_i]=\mu(1-\mu) \sum_i \lambda_i^2 \le C^2 \norm{a}_2^2/2 \le C^2/2$. 
By Chebychev inequality, we have that when $\lambda \in I'_1$, 
\begin{align} \label{eq:goodevent}
\Pr\Big[\big| \sum_{i \in T}\lambda_i - \mu \lambda \big| > \frac{\beta \lambda}{2} \Big]=\Pr_Y\Big[ \big| \sum_i \lambda_i Y_i - \mu \lambda \big| > \frac{\beta\lambda}{2} \Big] &< 
\frac{2C^2}{\beta^2 \lambda^2}< \frac{1}{2},
\end{align}
where the last inequality holds since $\lambda \ge (1-\eta')t$ and $\beta \lambda \ge 0.9 \beta t \ge 0.9 c'/(16C) \ge 2C$, from our choice of $c'$. Hence, the contribution to the sum $Z$ from the random partition $T$ is around $(1-2\beta)$ fraction of the total sum with probability $1/2$ (over just the randomness in the partition). 


For $X=(X_1, X_2, \dots, X_\ell)$ consider the following r.v. coupled to $X$ based on the random set $T$ that is chosen beforehand:
$$ X'=(X'_1, X'_2, \dots, X'_\ell), \text{ where } X'_i=\begin{cases} -X_i & \text{ if } i \in T\\ X_i & \text{otherwise} \end{cases}.$$
 Each of the $X_i$ are symmetric with $\E[X_i]=0$ and $X_i$ are independent of $a_i$ and mutually independent, so $X$ and $X'$ are identically distributed, and the corresponding map is bijective. Hence if
 $$\sum_i \alpha_i X_i = \lambda \in I'_1, ~\text{ then }  \sum_i \alpha_i X'_i \in [(\beta-\tfrac{\beta}{2})\lambda, (\beta+\tfrac{\beta}{2})\lambda]= [\tfrac{\beta}{2}\lambda, \tfrac{3\beta}{2}\lambda], \text{with probability} \ge \tfrac{1}{2},$$
over just the randomness in the partition $T$.
Let $E_0$ represent the event that $Z \in [(1-\eta')t, (1+\eta')Ct]$ and $E_*$ be the event that $Z \in [\tfrac{\beta}{2}(1-\eta')t, \tfrac{3}{2}\beta(1+\eta')Ct]$. Let $x \in \R^\ell$ be an occurrence of $E_0$; for every fixed occurrence $x \in E_0$, since \eqref{eq:goodevent} holds with probability at least $1/2$, we have that the corresponding coupled occurrence $x'\in E_*$ with probability at least $1/2$ (the map from $x$ to $x'$ corresponding to the coupling bijective). Hence, $\Pr[E_*] \ge \kappa/2$, as required.

\paragraph{Case $t \le c'$:} 
In this case, we cannot use the above coupling argument since the variance from the random partition is too large compared to the sum $Z$. However, here we will just use the Berry-Esse\'{e}n theorem to argue the required concentration. Firstly $\norm{a}_\infty \le \eta < \tfrac{1}{80C^2} c_0 t \le c_0 c'/(80C^2) \le c_0/2$. Note that $[\tfrac{\beta}{2} t, \tfrac{3\beta}{2} Ct]$ corresponds to an interval of size at least $\beta t$ around $\beta t \ge t/(16C)$. Further $\beta \in (\tfrac{1}{16C}, \tfrac{7}{16C})$. We also have if $\sigma^2=\Var[Z]$, then $1 \le \sigma^2 \le C^2$ since $\norm{a}_2 =1$. Let $Z'=Z/\sigma$.  

Hence, from Fact~\ref{fact:BE} applied to $Z'$ we conclude that
\begin{align*}
\Pr\Big[Z \in [\tfrac{\beta}{2}(1-\eta')t, \tfrac{3\beta}{2}(1+\eta')Ct] \Big] &\ge  \Pr\Big[Z' \in [\tfrac{\beta t}{2\sigma}, \tfrac{3\beta C t}{2\sigma}] \Big] \ge \Phi\Big(\frac{3\beta t}{2\sigma}\Big) - \Phi\Big(\frac{\beta t}{2\sigma}\Big) - 2\eta \\
&\ge   \min_{ r=\frac{\beta t}{2 \sigma} \in [\tfrac{1}{32C^2}, 10 C] } (\Phi(3r)-\Phi(r)) - 2\eta 
 \ge 2c_0 - 2\eta \ge c_0, 
\end{align*}
where the last line uses our choice of $c_0$.  

\end{proof}

\subsection{Analysis and Identifiability of the Semirandom model for $k= \widetilde{\Omega}(m)$}

We start with the soundness analysis for the test. 

\begin{lemma}[Soundness]\label{lem:sr:soundness}
There exists constants $c_0, c_1,c_2,c_3,c_4>0$ (potentially depending on $C$) such that the following holds for any $\eta, \gamma, \kappa \in (0,1)$ satisfying $\sqrt{c_3 k/m} <\eta <c_1$. 
Suppose $\forall i \in [m], T \subset [m] \setminus \set{i}$ with $|T| \le c_3/\eta^2$ , there at least $N_0 \ge c_2 C \eta^{-1} \gamma_0^{-1}\cdot  \log(1/\gamma)$ samples whose supports contain $i$, but not any of $T$.   Suppose $z$ is a given unit vector such that $\abs{\iprod{z,A_i}} < 1-4\eta$ for all $i \in [m]$.
Furthermore, suppose there are at least $\kappa N \ge c_2 \log(1/\gamma)$ samples such that $|\iprod{z,Ax}| \in [1 - \eta,C(1+\eta)]$. Then with probability at least $1-\gamma$, there are at least $\min\set{\kappa N/4, c_4 \gamma_0 \eta N_0}$ samples such that $\abs{\iprod{z,Ax}} \in [\eta/(36C), 1-2\eta]$.  
\end{lemma}
In the above lemma, a typical setting of parameters is $\eta=1/\poly\log(n), \kappa=1/\poly(n)$, and $\gamma$ will be chosen depending on how many candidate unit vectors $z$ we have; for instance, to be $\exp(-O(n))$. 
\anote{Deleted remark about $N_0$.}
We first present a couple of simple lemmas which will be useful in the soundness analysis. 
The following lemma shows that in a semirandom set, among the ``random'' portion of the samples, given a fixed $i \in [m]$ and $T \subseteq [m]\setminus \set{i}$ of small size, there are many samples $x \in \R^m$ whose support contains $i$ and not $T$. This only uses approximate pairwise independence of the support distribution. This will be used crucially by our testing procedure.    
\begin{lemma} \label{lem:semirandom:randomsamples}
For any $s \ge 2(t+1) \log m$, suppose we have $N_0 \ge 8 s m/k$ samples drawn from the ``random'' model $\Dsr \odot \Dv$ (random support). Then with probability at least $1-\exp(-s)$, we have that for all $i \in [m]$, and all $T \subset [m]\setminus \set{i}$ such that $|T|\le t \le m/(2\tau k)$ we have at least $s$ samples that all contain $i$ but do not contain $T$ in their support. 
\end{lemma}
\begin{proof}
Consider a fixed $i\in [m]$, and a fixed set $T \subseteq [m]\setminus \set{i}$ with $|T|\le t$. Then 
\begin{align*}
\Pr_{x \sim \Ds}\Big[ \text{supp}(x) \ni i ~\wedge~ \text{supp}(x) \cap T = \emptyset \Big] &\ge \Pr_{\Ds}\Big[ i \in \supp(x)\Big] - \sum_{j \in T} \Pr_{\Ds}\Big[ i \in \supp(x) \wedge j \in \supp(x) \Big] \\
&\ge \frac{k}{m} \Bigparen{1 -  \frac{\tau k |T|}{m} } \ge \frac{k}{2m}.
\end{align*}
since $t \le m/(2\tau k)$. Hence, if we have $N_0$ samples, the expected number of samples that do not contain $T$ but contain $i$ in its support is at least $N_0 k/(2m) \ge 2s$. Hence, by using Chernoff bounds, and a union bound over all possible choices (at most $m^{t+1}$ of them),   the claim follows. 
\end{proof}

The following lemma is a simple consequence of Berry-Ess\'{e}en theorem that lower bounds the probability that the sum of independent random variables is very close to $0$. 

\begin{lemma}\label{lem:anticonczero}
Let $C\ge 1$ and $\eta_1 \in (0,1/(32C)]$ be constants. 
Let $Z=\sum_{i=1}^{\ell} \alpha_i X_i$ where $X_i$ are mean zero, symmetric, independent random variables taking values in $[-C,-1] \cup [1,C]$ and let $\norm{\alpha}_2 \le 1$ and $\norm{\alpha}_\infty < \eta_1$. Then there exists a constant $c_1=c_1(C)>0$ (potentially depending on $C$) such that
$$\Pr\Big[ \sum_{i=1}^\ell \alpha_i X_i \in [0,9C\eta_1) \Big] \ge  c_1 \eta_1.$$
\end{lemma}
\begin{proof}
If $\sigma_1^2$ is the variance of $Z$, then
$\sigma_1^2=\sum_i \alpha_i^2 \Var[x_i]$ . Hence,  $\norm{\alpha}_2 \le \sigma_1 \le C \norm{\alpha}_2.$ 
We split the elements depending on how large they are compared to the variance. Let $\eta'=\min\set{\eta_1, \sigma/(16C)}$. Let $T_g=\set{i \in [\ell]: \abs{\alpha_i} \le \eta'}$, and let $T_b= \set{i \in [\ell]: \abs{\alpha_i} > \eta'}$. Firstly, when $\eta'=\eta_1$ we have $|T_b|=0$ since $\norm{\alpha}_\infty \le \eta_1$. Otherwise, $|T_b| \le 256 C^2$.  

Applying Fact~\ref{fact:BE} due to the Berry-Ess\'{e}en theorem to the sum restricted to the small terms i.e., in $T_g$,
\begin{align*}
\Pr \Big[ \sum_{i \in T_g} \alpha_i x_i \in [0, 8 \eta_1 C] \Big] &\ge \Pr \Big[ \sum_{i \in T_g} \alpha_i x_i \in [0, 8 \eta'C] \Big] \ge \Phi\left(\frac{8 \eta'C}{\sigma_1}\right)-\Phi(0) -  \frac{2\eta' C}{\sigma_1} \\
&\ge \tfrac{1}{2}\text{erf}\left( \frac{4 \sqrt{2} \eta' C}{\sigma_1} \right) - \frac{2\eta' C}{\sigma_1} \ge  \frac{\eta' C}{8\sigma_1} \ge \frac{\eta'}{8 \norm{\alpha}_2} \\
&\ge \min\Set{\frac{\eta_1}{8 \norm{\alpha}_2}, \frac{\sigma_1}{128C\norm{\alpha}_2}} \ge \frac{\eta_1}{128C \norm{\alpha}_2}, 
\end{align*}
where the second line uses the fact that $\tfrac{1}{2}\text{erf}(4 \sqrt{2} \delta) \ge (2+\tfrac{1}{8})\delta$ for all $\delta \le 0.2$. 

Further, since each $x_i$ is independent and symmetric, we have $\sum_{i \in T_b} \alpha_i x_i \in [0, \eta_1 C]$ with probability at least $2^{-|T_b|} \ge 2^{-256C^2}$. Since $\sum_{i \in T_b} \alpha_i x_i$ and $\sum_{i \in T_g} \alpha_i x_i$ are independent, we have for some constant $c_1>0$ (e.g., $c_1=2^{-256C^2}/4$ suffices)  
$$\Pr \Big[ \sum_{i \in \ell} \alpha_i x_i \in [0, 9C \eta_1) \Big] \ge \Pr \Big[ \sum_{i \in T_g} \alpha_i x_i \in [0, 8 C\eta'] \Big] \times \Pr \Big[ \sum_{i \in T_b} \alpha_i x_i \in [0, C \eta_1 ] \Big]  \ge \frac{c_1 \eta_1}{C \norm{\alpha}_2}.$$

\end{proof}

We now proceed to the soundness proof, which crucially the weak anti-concentration statement in Lemma~\ref{lem:weakanticonc}. 
\anote{Check that $C+\eta$ vs $C(1+\eta)$ doesnt cause any issues.}
\begin{proof}[Proof of Lemma~\ref{lem:sr:soundness}]
For convenience, let $\eta'=\eta/(18C)$. 
Let $T_{\text{lg}}=\set{i \in [m]: \abs{\iprod{z,A_i}} > \eta'}$. Note that from Lemma~\ref{lem:largevalues}, we have that $|T_{\text{lg}}| \le 2/(\eta')^2$. Let $\alpha_i = \iprod{z,A_j}$. 

\paragraph{Case $|T_{\text{lg}}|=0$.} 
In this case, it follows by applying Lemma~\ref{lem:weakanticonc}, and stitching its guarantees across different supports.  Let $S$ be a fixed support, and condition on $x$ having a support of $S$.
Let 
$$\kappa(S)= \Pr_{x \sim \calD} \Big[\bigabs{\sum_{i \in S} \iprod{z,A_i} x_i } \in [1- \eta,C(1+\eta)] ~\big|~ \text{supp}(x)=S \Big].$$

Let $\alpha \in \R^{S}$ be defined by $\alpha_i=\iprod{z,A_i}$ for each $i \in S$. We will apply Lemma~\ref{lem:weakanticonc} to the linear form given by $Z=\sum_{i \in S} a_i X_i$, where random variable $X_i=x_i$, $a=\alpha/\norm{\alpha}_2$ and consider $t=1/\norm{\alpha}_2$. Also $\norm{a}_\infty =\gamma/\norm{\alpha}_2 < c_1 t$. Applying Lemma~\ref{lem:weakanticonc} with $\eta'=\eta$ and $\beta=1/(3C)$, we have

\begin{align*}
\Pr\Big[\bigabs{\sum_i \alpha_i x_i} \in [1-\eta,C(1+\eta)] \Big] \ge \kappa(S) ~\implies~~ \Pr\Big[\bigabs{\sum_{i \in S} \alpha_i x_i} \in [\tfrac{1}{6C}-\eta,\tfrac{1}{2}+\eta] \Big] &\ge \frac{c_0\kappa(S)}{2}.\\
\text{Since }\eta < \frac{1}{16C},~~ \Pr\Big[\bigabs{\sum_{i \in S} \alpha_i x_i} \in [2\eta,\tfrac{1}{2}+\eta] \Big] &\ge \frac{c_0 \kappa(S)}{2}.
\end{align*}   
Summing up over all $S$, and using $\kappa=\sum_S q(S) \kappa(S) $, we get that 
$$ \Pr_{x \sim \calD}\Big[ 2\eta \leq |\iprod{z,Ax}| \leq 1-4\eta \Big] \geq \frac{c_0}{2} \cdot\kappa .$$

Further, $c_0 \kappa N \ge \Omega(\log(1/\gamma))$.
Hence, using Chernoff bounds we have with probability at least $(1-\gamma)$ that if $\kappa N$ samples $x$ satisfy $\abs{\iprod{z,Ax}} \in [1-\eta,C+\eta]$, then $\tfrac{c_0}{4} \kappa N$ samples satisfy $\abs{\iprod{z,Ax}} \in (2\eta, 1-2\eta)$. Hence $z$ fails the test with probability at least $1-\gamma$.

\paragraph{Case $|T_{\text{lg}}|\ge 1$.} Let $j \in T_{\text{lg}}$. Let $\tcalD$ be the distribution over vectors $x$ conditioned on $\supp(x) \cap T_{\text{lg}}=\set{j}$. Since $|T_{\text{lg}}| \le 2/(\eta')^2 \le c_3/\eta^2$, we have that at least $N_0$ samples $x$ s.t. $j \in \supp(x)$ and $T_{\text{lg}} \cap \supp(x)=\set{j}$. 

On the other hand for any given sample with given support $S$ ($|S|\le k$),  $\norm{\alpha_S}_2^2 \le \norm{A_S}_2^2 \le 1+\delta$ . Further, $\norm{\alpha}_\infty \le \eta' \le \eta/(18C)$. Applying Lemma~\ref{lem:anticonczero}, we have for some constant $c'>0$ (potentially depending on $C$) that
\begin{equation}\label{eq:genC:soundness:1}
\Pr_{x \sim \tcalD} \Big[ \sum_{i \in S\setminus \set{j}} \alpha_i x_i \in [0, \eta/2] \Big] \ge c' \eta.
\end{equation}
Further, since $j \in T_{\text{lg}}$, $\eta' \le |\alpha_j| \le 1- 4\eta$. However, recall that $|x_j| \in [1,C]$, hence $|\alpha_j x_j|$ can be as large as $(1-4\eta)C \ge 1$. However, since $j \in \supp(x)$, we are given that with probability at least $\gamma_0$, $x_j \in [1-\eta, 1+\eta]$ (and similarly $[-1-\eta,-1+\eta]$). Hence with probability at least $\gamma_0$, we have that $\eta' (1-\eta) \le \alpha_j x_j \le (1-4\eta) (1+\eta) \le 1-3\eta$. Further from \eqref{eq:genC:soundness:1} and independence of $x_i$, we have
with probability at least $c' \gamma_0 \eta$ that 
$$\frac{\eta}{36C} \le \eta' (1-\eta) \le \sum_{i \in S} \alpha_i x_i = \alpha_j x_j + \sum_{i \in S \setminus T_{\text{lg}}} \alpha_i x_i  \le 1-3\eta + \frac{\eta}{2} \le 1-5\eta/2.$$ 

Note that the value distribution is independent of the sparsity. Hence as before, applying Chernoff bounds for the $N_0$ samples 
we get that with probability at least $1-\gamma$ that 
$\tfrac{c'}{18C} \delta_0 \eta N_0  \ge \Omega(\log(1/\gamma))$ samples have $\eta/(36C) \le \abs{\iprod{z,Ax}} \le 1-2\eta$.  

\end{proof}


We now present a simple lemma that is useful for completeness. 

\begin{lemma}
\label{lem:close-vector-product}
For any $\eta \in (0,1), \kappa_0 \in (0,\tfrac{1}{2})$, suppose for some $i \in [m]$ and $b \in \set{-1,1}$, let $\Ahat_i$ be a vector such that $\norm{\Ahat_i - b A_i}_2 \leq \eta' < \frac{\eta}{8C\log(1/\kappa_0)}$. Let $\set{\ysamp{1}, \ysamp{2}, \ldots \ysamp{N}}$ be a set of samples generated from the model where $\ysamp{r} = A \xsamp{r}$, where $\xsamp{r}$ is a $k$-sparse vector with arbitrary sparsity pattern, and the non zero values drawn independently from the distribution $\Dv$. Furthermore, assume that 
$A$ is $(k,\delta)$-RIP for $0<\sqrt{\delta}<\eta/(16C \log(1/\kappa_0))$. Then for a fixed sample $r \in [N]$
\begin{equation}
\Pr\Big[\abs{\iprod{\ysamp{r},\Ahat_i}- b x_i} \ge \eta \Big] \le \kappa_0.  \label{eq:fixedsample}
\end{equation} 
Further, we have with probability at least $1-\kappa_0 N$ that
\begin{align} 
\forall r \in [N],~ \abs{\iprod{\ysamp{r},\Ahat_i} - \iprod{\ysamp{r},bA_i}} &\leq 4C \eta' \log (1/\kappa_0) < \eta/2.\label{eq:close:1}\\
\abs{\iprod{\ysamp{r},\Ahat_i}- b x_i} & \leq 4C \log(1/\kappa_0)(\eta'+2\sqrt{\delta}) < \eta \label{eq:close:2}.
\end{align}
\end{lemma}

\begin{proof}
Consider a fixed sample $\ysamp{r}=A \xsamp{r}$. Define the random variable $Q_r = \iprod{\ysamp{r},bA_i - \Ahat_i}$; here the support of $\xsamp{r}$ is fixed, but the values of the $k$ non-zero entries of $\xsamp{r}$ are independent and picked from $\Dv$.
Similarly, let $R_r=\iprod{\ysamp{r}, bA_i} - bx_i$. 
For each $r \in [N]$, let $E_r$ represent the event 
$\big[|Q_r| \ge 4C \eta' \log(1/\kappa_0)  ~\text{ or }~ |R_r|\ge 8C \sqrt{\delta} \log(1/\kappa_0) \big].$ 

Let $T$ denote the support of $\xsamp{r}$, and let $x=\xsamp{r}$ for convenience. Let $\psi=bA_i - \Ahat_i$.\\ 
Note that $\psi,A$ are fixed, and $x_j$ are picked independently. If $A_T$ represents the submatrix of $A$ formed by the columns corresponding to $T$,  then we have using the $(k,\delta)$-RIP property of $A$ that if we denote by
\begin{align*}
Q_r &=\iprod{\psi,Ax} = \sum_{j \in T} \iprod{\psi,A_j} x_j\\
\text{Var}[Q_r]&= \sum_{j \in T} \iprod{\psi,A_j}^2 \text{Var}[x_i] \le C^2 \sum_{j \in T} \iprod{\psi,A_j}^2 \le C^2 \norm{A_T}^2 \norm{\psi}_2^2 \le (1+\delta)^2 C^2 (\eta')^2 .
\end{align*}
Further each entry is at most $\abs{\iprod{\psi,A_j}x_j} \le C \norm{\psi}_2 \le C\eta'$. Applying Bernstein's inequality with $t=4C \eta' \log(1/\kappa_0) $
\begin{align*}
\Pr\Big[ |Q_r| \ge t \Big] \le 2\exp\left( -\frac{t^2}{2 \text{Var}[Q_r]+ C \eta' t}\right) \le 2 e^{-2\log(1/\kappa_0)} \le \frac{\kappa_0}{2}.
\end{align*}
\noindent Similarly, we analyze  
$R_r - bx_i= \iprod{bA_i, Ax}-bx_i = \sum_{j \in T \setminus \set{i}} \iprod{A_i, A_j} x_j$,
where $x_i=0$ if $i \notin T$. From Lemma~\ref{lem:incoherencefromRIP} we have that the $\text{Var}[R_{r}] \le 3C^2 \delta$, $\E[R_r]=0$ and $\abs{\iprod{A_i, A_j} x_j} \le 2C \sqrt{\delta}$. Hence, from Bernstein inequality with $t=8C\log(1/\kappa_0) \sqrt{\delta}$ we again get 
\begin{align*}
\Pr\Big[ |R_r| &\ge 8C \sqrt{\delta} \log(1/\kappa_0)   \Big] \le  2 e^{-2\log(1/\kappa_0)} < \kappa_0/2. \\
\text{Hence } \Pr[E_r] & \le \Pr\Big[|Q_\ell| > 4C \eta' \log(1/\kappa_0) \Big]+\Pr\Big[|R_r|>8C \sqrt{\delta} \log(1/\kappa_0) \Big] \le  \kappa_0.
\end{align*}

Hence, performing a union bound over all the $s$ events $E_1, \dots, E_N$ for the $N$ samples, we have that both \eqref{eq:close:1}, and \eqref{eq:close:2} hold with probability at least $1- \kappa_0 N$. 



\end{proof}


The completeness analysis follows in straightforward fashion from Lemma~\ref{lem:close-vector-product}. In what follows, given $N$ samples $\ysamp{1}, \dots, \ysamp{N}$, let $q^{(1)}=\min_{i \in [m]} \tfrac{1}{N} \sum_{r \in [N]} \I[\xsamp{r}_i \ne 0]$. Note that $q^{(1)} \ge q_{\min}$. 
\begin{lemma}[Completeness]\label{lem:sr:completeness}
There exists constants $c_2,c_3>0$ (potentially depending on $C$) such that the following holds for any $\eta, \gamma \in (0,1), \kappa_0 \in (0,q^{(1)}/2)$ satisfying $\sqrt{c_3 k/m} <\eta <c_1$. Let $A$ be $(k,\delta)$-RIP for $\sqrt{\delta}<\eta/(16C \log(1/\kappa_0))$ and $z \in \R^n$ be a given unit vector such that $\norm{z-bA_i}  \le \eta' \le \eta/(8C \log(1/\kappa_0))$ for some $i \in [m]$, $b \in \set{-1,1}$.
Suppose we are given $N \ge c_2  \log(1/\gamma)/\min\set{\kappa_0,q^{(1)}}$ samples of the form $\set{\ysamp{r}=A\xsamp{r} : r\in [N]}$ drawn with arbitrary sparsity pattern and each non-zero value drawn randomly from $\Dv$ (as in Section~\ref{sec:prelims}). Then, we have with probability at least $1-\gamma$  
\begin{align}
\Abs{\bigset{\abs{\iprod{z,Ax}} \notin I_\eta= [0,\eta) \cup [1-\eta, C(1+\eta)]}} &\le 2\kappa_0 N  \label{eq:sr:complete:eq1}
\\
\Abs{\bigset{r \in [N]: \abs{\iprod{z,A\xsamp{r}}} \in [ 1-\eta, C(1+\eta) ] }} &\ge \tfrac{1}{4} q^{(1)} N. \label{eq:sr:complete:eq2}
\end{align}
\end{lemma}
\begin{proof}
Let $z=bA_i + \psi$ where $\norm{\psi}_2 \le \eta$.
We have from \eqref{eq:fixedsample} of Lemma~\ref{lem:close-vector-product} that
$\abs{\iprod{z,Ax}-bx_i} \ge \eta$ with probability at most $\kappa_0$. Further $\abs{x_i} \in \set{0} \cup [1,C]$; hence, for a fixed sample $r \in [N]$, $\abs{\iprod{z,A\xsamp{r}}} \notin I_\eta$ with probability at most $\kappa_0$. Hence, at most $\kappa_0 N =\Omega(\log(1/\gamma))$ samples have $\abs{\iprod{z,A\xsamp{r}}} \notin I_\eta$ in expectation.  Note that the value distribution is independent of the sparsity. Hence applying Chernoff bounds for the $N$ independent samples 
we get that with probability at least $1-\gamma$ that \eqref{eq:sr:complete:eq1} holds. 

Similarly, $|x_i| \ne 0$ for at least $q^{(1)} N$ fraction of the samples, and the value distribution is symmetric. Further, from \eqref{eq:fixedsample} of Lemma~\ref{lem:close-vector-product} $\abs{\iprod{z,Ax}} \ge 1-\eta$ with probability at least $q^{(1)} - \kappa_0/2 \ge q^{(1)}/2$. Hence, using a similar argument involving Chernoff bounds, 
\eqref{eq:sr:complete:eq2} holds. 
\end{proof}

\begin{lemma}[Amplifying Accuracy]\label{lem:purify}
There exists constants $c_1,c_2,c_3>0$ (potentially depending on $C$) such that the following holds for any $\eta_0<c_1, \gamma \in (0,1)$.
Let $A$ be $(k,\delta)$-RIP for $\sqrt{\delta}<c'/(16C \log(\tfrac{mn}{q^{(1)} \eta_0}))$ and $z \in \R^n$ be a given unit vector such that $\norm{z-bA_i}  \le \eta_2 := c'/(8C \log(\tfrac{mn}{q^{(1)} \eta_0}))$ for some $i \in [m]$, $b \in \set{-1,1}$.
Suppose we are given $N \ge c_2 knm\eta_0^{-3} \log(1/\gamma)/q^{(1)}$ samples of the form $\set{\ysamp{r}=A\xsamp{r} : r\in [N]}$ drawn with arbitrary sparsity pattern and each non-zero value drawn randomly from $\Dv$ (as in Section~\ref{sec:prelims}). Then, we have with probability at least $1-\gamma$ that if  
\begin{equation}
\widehat{z}= \frac{\sum_{r \in [N]} \ysamp{r} \I\big[ \iprod{z,\ysamp{r}} \ge \tfrac{1}{2}\big]}{\sum_{r \in [N]}\I\big[ \iprod{z,\ysamp{r}} \ge \tfrac{1}{2}\big]}, ~\text{ then } \Norm{\frac{\widehat{z}}{\norm{\widehat{z}}_2} - bA_i}_2 \le \eta_0.
\end{equation}
\end{lemma}
\begin{proof}
Let $z^*=\E_{x \sim \calD}\big[ y \big| \iprod{y,z} \ge \tfrac{1}{2}\big]$. We will show $\norm{\widehat{z}-z^*}_2 \le \eta_0/2$ and $\norm{z^* - bA_i}_2 \le \eta_0/2$.
For the former, we will use concentration bounds for each of the $n$ co-ordinates. Let $\ell \in [n]$. Observe that for any sample $y$, $\abs{y(\ell)} \le Ck$, and $\Var[y(\ell)]\le C^2 k^2$. By applying Hoeffding bounds, we see that with $N \ge c_2 C n k\eta_0^{-2} \log(n/\gamma)$, we have that with probability at least $1-\gamma/2$, $\forall \ell \in [n], ~ \abs{\widehat{z}(\ell) - z^*(\ell)}< \eta_0/(2\sqrt{n})$; hence $\norm{\widehat{z}-z^*}_2 \le \eta_0/2$.   

Set $\kappa_0= \eta_0 q^{(1)}/(16Ckm)$, and $c_1<1/4$, and let $\mu_j=\E_{\calD}[x_j ~|~ x_j\ge 1]$ for each $j \in [m]$. From Lemma~\ref{lem:close-vector-product} we have 
\begin{align*}
\Pr\Big[\bigabs{ \iprod{y, z} - bx_i} \ge c_1 \Big]  \le \kappa_0 &\implies~~ \Pr\Big[ \I[\iprod{y,z} \ge \tfrac{1}{2}] \ne \I[bx_i \ge 1] \Big] \le \kappa_0 \\
\E_{\calD}\Big[ y ~\big|~ bx_i \ge 1 \Big]&=\sum_{j \in [m]} \E_{\calD}\Big[ x_j ~\big|~ bx_i \ge 1 \Big]A_j = b \mu_i A_i
\end{align*}
where the last line follows from symmetry. 
Further, $b x_i \ge 1$ with probability at least $q^{(1)}/2$.  
If $\tcalD$ be the conditional distribution  of $\calD$ conditioned on $\iprod{y,z} \ge \tfrac{1}{2}$, 
\begin{align*}
z^* = \sum_{j \in [m]} \E_{\tcalD}[x_j]  A_j &= b \mu_i A_i + \sum_{j \in [m]} \Paren{\E_{\tcalD} [x_j] - \E_{\calD}[x_j ~|~ bx_i \ge 1]} A_j\\
\norm{z^* - b\mu_i A_i}_2 & \le \sum_{j \in [m]} \Abs{\E_{\tcalD} [x_j] - \E_{\calD}[x_j ~|~ bx_i \ge 1]} \le \frac{4\kappa_0 Ck m}{\tfrac{1}{2}q^{(1)}} < \eta_0/2,
\end{align*}
since $\norm{A_j}_2 =1$.
Hence, the lemma follows. 
\end{proof}

We now wrap up the proof of Theorem~\ref{thm:sr:testing} and Corollary~\ref{corr:sr:identifiability}.

\begin{proof}[Proof of Theorem~\ref{thm:sr:testing}]
The proof follows in a straightforward way by combining Lemma~\ref{lem:sr:completeness} and Lemma~\ref{lem:sr:soundness}. Set $\kappa_1=\tfrac{1}{2}c_4 \gamma_0 \eta$. Firstly, note that $q_{\min} \le q^{(1)}$. If $\norm{z-bA_i}_2 \le \eta$ for some $i \in [m], b \in \set{-1,1}$ then from Lemma~\ref{lem:sr:completeness}, we have that with probability at least $1-\gamma/2$ that $\abs{\iprod{z,\ysamp{r}}} \notin I_\eta$ for at most $2\kappa_0 N$ samples, and $\abs{\iprod{z,\ysamp{r}}} \in [1-\eta, C(1+\eta)]$ for at least $q_{\min} N/4$ samples. Hence it passes the test, proving the completeness case.

On the other hand, from Lemma~\ref{lem:sr:soundness} applied with $\kappa=q_{\min}/8$ and since $\min\set{\tfrac{1}{32}, c_4 \gamma_0 \eta} q_{\min} \ge 2 \kappa_1=2c_5 \gamma_0 \eta q_{\min}$ (picking $c_5=c_4/2$), 
we also get that if $z$ passes the test, then with probability at least $1-\gamma/2$, we have $\abs{\iprod{z,A_i}} \le 1-4\eta$ i.e., $\norm{z-bA_i}_2 \le \sqrt{8\eta}$ for some $i \in [m], b \in \set{-1,1}$ as required. Further, from our choice of parameters $\sqrt{8 \eta}< \eta_2:=c_1/(8C \log(\tfrac{mn}{q_{\min}\eta_0}))$. Hence applying Lemma~\ref{lem:purify} we also get that $\norm{\widehat{z}-bA_i}_2  \le \eta_0$ with probability at least $1-\gamma/2$. Combining the two, we get the soundness claim. 

\end{proof}

\begin{proof}[Proof of Corollary~\ref{corr:sr:identifiability}]
Consider a $\eta'$-net over $\R^n$ dimensional unit vectors where $\eta'= c' \eta/(C \log n)$ for some constant $c'>0$. Since we have $k/m < \log^{-2c_1} m$, we can set $\eta=\log^{-c_1} m, \gamma= (\eta'/4)^n$ for $c_1>2$. Applying Theorem~\ref{thm:sr:testing} and performing a union bound over every candidate vector $z$ in the $\eta'$-net, we get with probability at least $1-\exp(-n)$, that only vectors that are $O(\sqrt{\eta})$ close to a column passes {\sc TestColumn}, and there is at least one candidate in the net $\eta'$-close to each column that passes the test. Further $\norm{A_i - A_j}_2 \ge 1/2$ for each $i \ne j$. Hence, we can cluster the candidates into exactly $m$ clusters of radius $O(\sqrt{\eta})$ around each true column. Picking one such candidate $z$ for each column $A_i$, and looking at its corresponding $\widehat{z}$ returns each column up to $\eta_0$ accuracy. 
\end{proof}




\section{Stronger Identifiability for Rademacher Value Distribution}

\newcommand{\Small}{S_{\text{small}}}
\newcommand{\Shalf}{S_{1/2}}

In the special case when the value distribution is a Rademacher distribution (each $x_i$ is $+1$ or $-1$ with probability $1/2$ each), we can obtain even stronger guarantees for the testing procedure. We do not need to assume that there are non-negligible fraction of samples $y=Ax$ where the support distribution is ``random'' \footnote{In particular, we don't need to assume for any $i, T \subseteq [m]\setminus\set{i}$ of small size, that we have many samples that contain $i$ but not $T$.}. Here, we just need that for every triple $i_1, i_2, i_3 \in [m]$ of columns, they jointly occur in at least a non-negligible number of samples.
On the other hand, we remark that the triple co-occurrence condition is arguably the weakest condition under which identifiability is possible. Proposition~\ref{prop:non-identifiability} shows a non-identifiability statement even when the value distribution is a Rademacher distribution. In this example, for every pair of columns there are many samples where these two columns co-occur.

\begin{theorem}[Rademacher Value Distribution]\label{thm:rad:testing}
There exists constants $c_0, c_1,c_2,c_3,c_4>0$ such that the following holds for any $\gamma\in (0,1), \eta_0<\eta \in (0,1)$ satisfying $\sqrt{\frac{c_3 k}{m}} <\eta <\frac{c_1}{\log^2\bigparen{\tfrac{mn}{q_{0} \eta_0}}}$. Set $\kappa_0:= c_4 \eta q_{0}/(km)$.
Suppose we are given $N \ge \frac{c_2 knm \log(1/\gamma)}{\eta_0^3 \kappa_0}$ samples $\ysamp{1}, \dots, \ysamp{N}$ satisfying
\begin{itemize}
\item the dictionary $A$ is $(k,\delta)$-RIP for $\delta< \bigparen{\frac{\eta}{16\log(1/\kappa_0) }}^2$,
\item $\forall i_1, i_2, i_3 \in [m]$, there at least $q_{0} N$ samples whose supports all contain $i_1, i_2, i_3$.
\end{itemize}
Suppose we are given a unit vector $z\in \R^n$, then Algorithm~\ref{ALG:rad} i.e., {\sc TestCol\_Rad} called with parameters $(z, \set{\ysamp{1},\dots,\ysamp{N}}, 2\kappa_0, \kappa_1=c_5 \eta q_{0}, \eta)$ runs in times $O(N)$ time, and we have with probability at least $1-\gamma$ that 
\begin{itemize}
\item {\em (Completeness)} if $\norm{z- bA_i}_2 \le \eta'=\eta/(8\log(1/\kappa_0))$ for some $i \in [m], b \in \set{-1,1}$, then Algorithm~\ref{ALG:rad} outputs $(\text{YES}, z')$. 
\item {\em (Soundness)} if unit vector $z \in \R^n$ passes Algorithm~\ref{ALG:rad}, then there exists $i \in [m], b \in \set{-1,1}$ such that $\norm{z-bA_i}_2 \le \sqrt{8\eta}$. Further, in this case  $\norm{z' - bA_i}_2 \le \eta_0$. 
\end{itemize}
\end{theorem}
As before, we note that the above algorithm is also robust to adversarial noise of the order of magnitude $O(\eta)$ in every sample. Further, the above theorem again implies an identifiability result by applying it to each candidate unit vector $z$ in an $\widetilde{\Omega}(\eta)$-net over $\R^n$ dimensional unit vectors and choosing $\gamma= \exp\bigparen{-\Omega(n \log(1/\eta))}$ for $k<n/\poly\log(n)$. 

\begin{corollary}[Identifiability for Rademacher Value Distribution]\label{corr:rad:identifiability}
There exists constants $c_0, c_1,c_2,c_3,c_4,c_5,c_6>0$ such that the following holds for any $k< n/\log^{2c_1} m$, $\eta_0 \in (0,1)$. 
Set $\kappa_0:= c_0 \log^{-c_1} m q_{0}$.
Suppose we are given $N \ge c_2 knm \eta_0^{-3} q_{0}^{-1}\log^{c_1} m\log(1/\kappa_0)$ samples $\ysamp{1}, \dots, \ysamp{N}$ satisfying
\begin{itemize}
\item the dictionary $A$ is $(k,\delta)$-RIP for $\delta< \frac{c_5}{\log(1/\kappa_0) \log^{c_6} m}$,
\item $\forall i_1, i_2, i_3 \in [m]$, there at least $q_{0} N$ samples whose supports all contain $i_1, i_2, i_3$.
\end{itemize}
Then there is an algorithm that with probability at least $1-\exp(-n)$ finds the columns $\widehat{A}$ (up to renaming columns) such that $\norm{\widehat{A}_i - b_i A_i}_2 \le \eta_0$ for some $b \in \set{-1,1}^m$. 
\end{corollary}

The test procedure for checking whether unit vector $z$ is close to a column is slightly different. In addition to Algorithm {\sc TestColumn}, there is an additional procedure that 

\anote{Modifying algorithm slightly to compensate for the $16$ factor loss in soundness.}

 \begin{figure}[htb]
 \begin{center}
 \fbox{\parbox{0.98\textwidth}{
 {\bf Algorithm \textsc{TestCol\_Rad}}$(z,Y=\set{\ysamp{1},\dots,\ysamp{N}}, \kappa_0, \kappa_1,\eta)$
\begin{enumerate}
\item Let $\widetilde{\kappa}_1$ be the fraction of samples such that $\abs{\iprod{z,\ysamp{r}}} \in [1-\eta,1+\eta]$ and $\widetilde{\kappa}_0$ be the fraction of samples such that $\abs{\iprod{z,\ysamp{r}}} \notin [1-\eta,1+\eta] \cup [0,\tfrac{1}{32}\eta]$.
\item Check if $\widetilde{\kappa}_0 < \kappa_0$ and $\widetilde{\kappa}_1 \ge \kappa_1$. 
\item If Yes, then compute $z'=\text{mean}\bigparen{\set{\ysamp{r}:  \iprod{\ysamp{r}, z} \in (1-10\eta,1+10\eta)}}$, and check if $\norm{z'}_2 \le 1.1$. If yes, return $(\text{YES}, \widehat{z})$, where   
$z' = \text{mean}\big(\set{ \ysamp{r}: r \in [N] \text{ s.t. } \iprod{\ysamp{r},z}  \ge \tfrac{1}{2} } \big)$ and $\widehat{z} = z'/\norm{z'}_2$.
\item Else in other cases, return $(\text{NO},\emptyset)$.
\end{enumerate}
 }}
 \end{center}
 \caption{\label{ALG:rad}}
 \end{figure}


\subsection{Analysis for Rademacher value distribution.}

In the following lemmas, a typical setting of parameters is $\eta=1/\poly\log(n), \kappa=1/\poly(n)$, and $\gamma$ will be chosen depending on how many candidate unit vectors $z$ we have; for instance, to be $\exp(-O(n))$.

The completeness analysis mainly follows along the same lines as Lemma~\ref{lem:sr:completeness} (and uses Lemma~\ref{lem:close-vector-product}); but it also has an additional component that argues about the extra test. The following lemma is stated for a fixed unit vector $z$ and a single sample $x$ drawn from $\calD$.

\begin{lemma}[Completeness]\label{lem:rad:completeness}
There exists absolute constants $c_1,c_2,c_3>0$ such that the following holds for any $\eta_0 < \eta \in (0,c_1), \gamma \in (0,1/2)$ and $0<\kappa_0 <\min\Bigset{ \frac{\eta_0 q_{\min}}{16 km},\frac{\eta_0 q_{\min}^2}{2 \sqrt{k}}}$. 
Let $A$ be $(k,\delta)$-RIP for $\sqrt{\delta}<\eta/(16\log(1/\kappa_0))$ and $z \in \R^n$ be a given unit vector such that $\norm{z-bA_i}  \le \eta' \le \eta/(8 \log(1/\kappa_0))$ for some $i \in [m]$, $b \in \set{-1,1}$.
Suppose we are given $N \ge c_2 \eta_0^{-2} \log(1/\gamma)/\min\set{\kappa_0,q_{\min}}$ samples of the form $\set{\ysamp{r}=A\xsamp{r} : r\in [N]}$ drawn with arbitrary sparsity pattern and each non-zero value drawn from a Rademacher distribution. Then, we have with probability at least $1-\gamma$  
\begin{align}
\Abs{\bigset{\abs{\iprod{z,Ax}} \notin I_\eta= [0,\tfrac{1}{16}\eta) \cup [1-\eta, 1+\eta]}} &\le 2\kappa_0 N  \label{eq:rad:complete:eq1}
\\
\Abs{\bigset{r \in [N]: \abs{\iprod{z,A\xsamp{r}}} \in [ 1-\eta, 1+\eta ] }} &\ge \tfrac{1}{4} q_{\min} N.
\label{eq:rad:complete:eq2}\\
\norm{z'}_2&\le 1+\eta_0 <1.1 \label{eq:rad:complete:eq3},
\end{align}
where $z'=\text{mean}\bigparen{\set{\ysamp{r}:  \iprod{\ysamp{r},z} \in (1-10\eta, 1+10\eta)},~ r \in [N]}$ is the statistic considered in step 3 of Algorithm~\ref{ALG:rad}.
\end{lemma}
\begin{proof}
The first two parts \eqref{eq:rad:complete:eq1}, \eqref{eq:rad:complete:eq2} follow by just applying Lemma~\ref{lem:sr:completeness} with $C=1$. 

We now prove \eqref{eq:rad:complete:eq3}. Let $z^*=\E_{x \sim \calD}\big[ Ax \big| \iprod{Ax,z} \ge \tfrac{1}{2}\big]$. We will show $\norm{\widehat{z}-z^*}_2 \le \eta_0/2$ and $\norm{z^* - bA_i}_2 \le \eta_0/2$.
For the former, we will use concentration bounds for each of the $n$ co-ordinates. Let $\ell \in [n]$. Observe that for any sample $y$, $\abs{y(\ell)} \le Ck$, and $\Var[y(\ell)]\le C^2 k^2$. By applying Hoeffding bounds, we see that with $N \ge c_2 C n k\eta_0^{-2} \log(n/\gamma)$, we have that with probability at least $1-\gamma/2$, $\forall \ell \in [n], ~ \abs{\widehat{z}(\ell) - z^*(\ell)}< \eta_0/(2\sqrt{n})$; hence $\norm{\widehat{z}-z^*}_2 \le \eta_0/2$.   

Again from \eqref{eq:fixedsample}, we have with probability at least $1-\kappa_0$, $\iprod{z,Ax} \in (1-10\eta, 1+10\eta)$ {\em if and only} if $x_i=b$. If $E'$ is the event $\iprod{z,Ax} \in (1-10\eta, 1+10\eta)$ and $E''$ is the event $x_i=b$, then with probability at least $1-\gamma$,  
\begin{align*}
\E_{x \sim \calD}\big[ Ax ~|~ E' \big] - \E_{x \sim \calD}\big[ Ax ~|~ E''\big] & =  \sum_x  Ax \cdot \Paren{\frac{\Pr[x]}{\Pr[E']} - \frac{\Pr[x]}{\Pr[E'']}}  \\
\Bignorm{ z^* - \E_{x \sim \calD}\big[ Ax ~|~ x_i=b \big] }_2 & \le \max_x \norm{Ax}_2 \cdot \frac{\bigabs{\Pr[E''] - \Pr[E']}}{\Pr[E''] Pr[E']} \le \frac{2  \kappa_0 \sqrt{k}}{q_{\min}^2} \le \frac{\eta_0}{2}\\
\text{Further } \E_{x \sim \calD}\big[Ax ~\big|~ x_i=b \big]&=A_i+ \sum_{j \ne i} \big( \E_{x \sim \calD} \big[x_j ~\big|~ x_i=b\big] \big) A_j =A_i.\\ 
\text{Hence } \norm{z^* - bA_i}_2 &\le \frac{\eta_0}{2}, ~~ \norm{z' - bA_i}_2 \le \eta_0.  
\end{align*}
\end{proof}

\begin{lemma}[Soundness]\label{lem:rad:soundness}
There exists constants $c_0, c_1,c_2,c_3,c_4>0$ such that the following holds for any $\eta_0<\eta \in (0,c_1), \gamma, \kappa \in (0,1)$. 
Suppose $\forall i_1, i_2, i_3 \in [m]$, the probability that 
$i_1, i_2, i_3$ are all in the support is at least $q_0$.   Given any unit vector $z$ such that $\abs{\iprod{z,A_i}} < 1-4\eta$ for all $i \in [m]$.
Furthermore, suppose there are at least $N \ge c_2 \eta_0^{-2}\log(1/\gamma) \max\set{\kappa^{-1},q_0^{-1} \eta^{-1} }$ samples such that $\abs{\iprod{z,Ax}} \in [1 - \eta,1+\eta]$. Then with probability at least $1-O(\gamma)$, at least one of the following two statements hold:
\begin{enumerate}
\item[(i)] There are at least $\min\Bigset{\frac{\kappa}{m^4} , q_0 } \cdot c_4 N$ samples such that $\abs{\iprod{z,Ax}} \in [\tfrac{1}{16}\eta, 1-2\eta]$.  
\item[(ii)] If $z'$ is the vector output by Algorithm {\sc TestCol\_Rad}, then $\norm{z'}_2 > (1+\sqrt{2})/2 - \eta_0>1.1$.
\end{enumerate}
\end{lemma}


As for the case of more general distributions, the soundness analysis for Rademacher distributions uses the anti-concentration type statement in Lemma~\ref{lem:weakanticonc} about weighted sums of independent Rademacher random variables. The following lemma specializes it for the Rademacher case, but generalizes it to also handle the case when the weights $\alpha_i$ can be relatively large. While this lemma conditions on a fixed support $S$, the final soundness claim will proceed by stitching together this claim over different $S$, since we expect very few samples with the same fixed support $S$.   

\begin{lemma}\label{lem:unnormal}
There exists a universal constant $c_0>0$ such that the following holds. 
Let $\eta \in (0, c_0/40), \kappa \in (0,1), \eps \in (0,\tfrac{1}{2})$, let $X_1, X_2, \dots, X_m$ be i.i.d. Rademacher r.v.s and let $S \subset [m]$ be a fixed subset of $\R^m$. Suppose $Z = \sum_{i \in S} \alpha_i X_i$ where $\alpha_S \in \R^S$ is any vector with $\norm{\alpha_S}_2 \le 1$ and $\norm{\alpha_S}_\infty < 1- 2\eta$. Suppose 
$$\Pr_{X} \Big[ \bigabs{ \sum_{i \in S} \alpha_i X_i} \in [1-\eta, 1+ \eta] \Big] \ge \kappa,$$ 
such that at least one of the following two cases holds:
\begin{enumerate}
\item[(a)] 
\begin{equation} \label{eq:impl1}
\Pr_{X} \Big[ \bigabs{\sum_{i \in S} \alpha_i X_i} \notin [0, 2\eta) \cup (1-2\eta, 1+2\eta) \Big] \ge \min\Set{\frac{\eps \kappa}{16},\frac{c_0}{2}}, 
\end{equation}
\item[(b)]
there exists $i^*_1, i^*_2 \in S$ such that $\abs{\alpha_{i^*_1}}, \abs{\alpha_{i^*_2}} \in [\tfrac{1}{2} - 2\eta, \tfrac{1}{2} + 2\eta]$ and $\abs{\alpha_i} \le 2\eta ~ \forall i \in S \setminus \set{i^*_1, i^*_2}$,  and 
\begin{equation}\label{eq:impl2}
\Pr_{X} \Big[ \bigabs{\sum_{i \in S \setminus \set{i^*_1, i^*_2}} \alpha_i X_i } > 8\eta \Big] \le \eps \kappa.
\end{equation}
\end{enumerate}
\end{lemma}

The above lemma shows that if $z$ is not close to a column, then either we have the case that $Z$ takes values outside of $I_{2\eta}$ for a non-negligible fraction of samples \eqref{eq:impl1} (so Algorithm {\sc TestColumn} would work), or we have a very particular structure -- there are two coefficients which are both close to $1/2$ in absolute value, and the rest of the terms do not contribute much. In fact this is unavoidable --  the instances in Proposition~\ref{prop:non-identifiability} that exhibit {\em non-identifiability} precisely result in combinations of this form. 

Lemma~\ref{lem:unnormal} involves a careful case analysis depending on the magnitude of the $\alpha_i =\iprod{z,A_i}$. On the one hand, when all the $\alpha_i$ are small, then Lemma~\ref{lem:weakanticonc} shows that $\iprod{z,Ax} \notin I_{2\eta}$ with reasonable probability. However, when there are some large $\alpha_i$, it involves a technical case analysis. Let $T_{1/2}=\set{i \in [m]: \abs{\iprod{z,A_i}} \in [\tfrac{1}{2}-2\eta, \tfrac{1}{2}+2\eta]}$.  \anote{made it $2\eta$ from $\eta$.}
Before we proceed to the proof of Lemma~\ref{lem:unnormal}, we prove the following helper lemma that handles the case when there is non-negligible contribution from terms $i$ such that $\abs{\alpha_i}$ is small. 

\begin{lemma}\label{lem:soundness:helper}
In the above notation, let $\Shalf=\set{i \in S: \abs{\alpha_i} \in (\tfrac{1}{2} - 2\eta, \tfrac{1}{2}+2\eta)}$, and let $\Small=\set{i \in S: \abs{\alpha_i} \le 2 \eta}$ and suppose $S=\Shalf \cup \Small$.  Also suppose
$$\Pr_{X} \Big[\bigabs{\sum_{i \in \Small} \alpha_i X_i} \ge 8 \eta \Big] \ge \kappa'.$$
Then there exists a universal constant $c\in (0,1)$ such that
\begin{equation}\label{eq:soundness:helper}
\Pr_{X} \Big[\bigabs{\sum_{i \in \Small} \alpha_i X_i} \notin I_{2\eta} \Big] \ge c \kappa'.
\end{equation}
\end{lemma}
\begin{proof}
We split the proof up into cases depending on $\abs{\Shalf}$; note that $|\Shalf| \le 4$. 

\paragraph{Case $\abs{\Shalf} \in \set{2,3,4}$:}
We have that either 
$$\Pr_X \Big[\sum_{i \in \Small} \alpha_i X_i \ge 8 \eta \Big] \ge \frac{\kappa'}{2} \text{ or } \Pr_X \Big[\sum_{i \in \Small} \alpha_i X_i \le -8\eta \Big] \ge \frac{\kappa'}{2}.$$
From the independence of $X_i$, with probability at least $1/16$ the signs of $\alpha_i X_i$ for all $i \in \Shalf$ match, and $\sum_{i \in \Shalf} \alpha_i X_i \ge 1-4\eta$ (similarly, it's negative with probability $1/16$). Hence, \eqref{eq:soundness:helper} follows.  

\paragraph{Case $|\Shalf|=1$:} At least one of the two cases hold: either 
$$\Pr_X \Big[\bigabs{\sum_{i \in \Small} \alpha_i X_i} \in [8 \eta, \tfrac{1}{2}-4\eta] \cup [\tfrac{1}{2}+4\eta, 1+2\eta) \Big] \ge \frac{\kappa'}{4} ~\text{ or  } \Pr_X \Big[\bigabs{\sum_{i \in \Small} \alpha_i X_i} \in (\tfrac{1}{2}-4\eta, \tfrac{1}{2}+4\eta) \Big] \ge \frac{\kappa'}{4} .$$
In the first case, we have from the symmetry and independence of the $X_i$ that $\sum_{i \in \Shalf} \alpha_i X_i$ and $\sum_{i \in \Small} \alpha_i X_i$ are aligned with probability $1/2$, thus giving \eqref{eq:soundness:helper} as required. 
In the second case, we apply Lemma~\ref{lem:weakanticonc} with $a=\frac{\alpha_S}{\norm{\alpha_S}_2}, t=\frac{1}{2\norm{\alpha_S}_2}, \beta=1/2$ and $\eta'=4 \eta/\norm{\alpha_S}_2$ to conclude that
$$\Pr_X \Big[\bigabs{\sum_{i \in \Small} \alpha_i X_i} \in (\tfrac{1}{8}-3\eta, \tfrac{3}{8}+3\eta) \Big] \ge \min\Bigset{\frac{\kappa'}{8}, c_0}.$$
Again, using the independence of $X_i$ and since $\eta < 1/80$, we have 
$$\Pr_X \Big[\bigabs{\sum_{i \in S} \alpha_i X_i} \in (2\eta, 1-2\eta) \Big] \ge \min\Bigset{\frac{\kappa'}{16}, \frac{c_0}{2}}.$$

\paragraph{Case $|\Shalf|=0$:} 
Our analysis will be very similar to the case when $|\Shalf|=1$. If 
$$\Pr_X \Big[\bigabs{\sum_{i \in \Small} \alpha_i X_i} \in [8 \eta, 1-2\eta] \cup [1+2\eta, 1+2\eta) \Big] \ge \frac{\kappa'}{4},$$
then this already gives \eqref{eq:soundness:helper}. Otherwise we have 
$$\Pr_X \Big[\bigabs{\sum_{i \in \Small} \alpha_i X_i} \in (1-2\eta, 1+2\eta) \Big] \ge \frac{\kappa'}{4} . $$
In this case, we apply Lemma~\ref{lem:weakanticonc} with $a=\alpha_S/\norm{\alpha_S}_2, t=1/\norm{\alpha_S}_2, \beta=1/2$ and $\eta'=2 \eta/\norm{\alpha_S}_2$ to conclude that
$$\Pr_X \Big[\bigabs{\sum_{i \in \Small} \alpha_i X_i} \in (\tfrac{1}{4}-3\eta, \tfrac{3}{4}+3\eta) \Big] \ge \min\Bigset{\frac{\kappa'}{8}, c_0},$$
thus establishing \eqref{eq:soundness:helper}.

\end{proof}

\anote{Made change from $\eta$ to $2\eta$ in a few places. Removed normalizing by $\norm{\alpha_S}_2$.}
We now proceed to the proof of Lemma~\ref{lem:unnormal}.
\begin{proof}[Proof of Lemma~\ref{lem:unnormal}]

For convenience, let us overload notation and denote $\alpha=\alpha_S$. From the assumptions of the lemma, $\norm{\alpha}_\infty < 1- \eta$. Let $\Small=\set{i \in S: \abs{\alpha_i} < 2\eta}$. We now have a case analysis depending on the contribution from $\Small$, and whether there are some large co-efficients $|\alpha_i|$ ($i \in S$). Finally let $\Shalf=\set{i \in S \big|~ \abs{\alpha_i} \in (\tfrac{1}{2} - 2\eta, \tfrac{1}{2}+2\eta)}$.

\noindent \emph{{\bf Case 1:} $\Small=S$.} In this case, it follows directly from Lemma~\ref{lem:weakanticonc}. Set $a=\alpha/\norm{\alpha}_2$ and $t=1/\norm{\alpha}_2$ and $Z= \sum_{i=1}^\ell a_i X_i$. Also $\norm{a}_\infty =\eta/\norm{\alpha}_2 < c_0 t/20$. Applying Lemma~\ref{lem:weakanticonc} with $\beta=1/3$, 
we have
\begin{align*}
\Pr\Big[\bigabs{\sum_i \alpha_i X_i} \in [1-\eta,1+\eta] \Big] \ge \kappa ~\implies~~ \Pr\Big[\bigabs{\sum_i \alpha_i X_i} \in [\tfrac{1}{6}-\eta,\tfrac{1}{2}+\eta] \Big] \ge \min\Bigset{\frac{\kappa}{2}, c_0}.
\end{align*}   
Hence, in this case \eqref{eq:impl1} follows. 

\noindent \emph{{\bf Case 2:} Suppose $\exists i^* \in S$ s.t. $\abs{\alpha_{i^*}} \in (2\eta, \tfrac{1}{2}-2\eta) \cup (\tfrac{1}{2}+2\eta, 1-2\eta)$.}

Consider the simple coupling $X'$ where $X'_1=-X_1$ and $X'_i=X_i$ for $i \ge 2$. 
\begin{align*}
\Bigabs{\sum_i \alpha_i X'_i - \sum_i \alpha_i X_i} &= 2 \abs{\alpha_1}  \in (4 \eta, 1- 4\eta) \cup (1+4\eta, 2-4\eta] \\ 
\text{Hence, } \Bigabs{\sum_i \alpha_i X_i}&\in [1-\eta,1+\eta] ~\implies~~ \Bigabs{\sum_i \alpha_i X'_i} \in [2 \eta, 1-2\eta] \cup [1+2\eta, \infty).
\end{align*}   
Hence, in this case \eqref{eq:impl1} follows, as 
$\Pr_{X} \Big[ \bigabs{\sum_{i \in S} \alpha_i X_i } \in [ 2\eta, 1- 2\eta) \cup (1+2\eta, \infty) \Big] \ge \kappa$.

\noindent Otherwise, $\Shalf \cup \Small=S$. 
First note that $|\Shalf| \le 4$. \\

\noindent \emph{{\bf Case 3:} $\Shalf \cup \Small=S$ and }
$$\Pr_X \Big[\bigabs{\sum_{i \in \Small} \alpha_i X_i} \ge 8 \eta \Big] < \eps \kappa.$$
If $|\Shalf|=2$, we have \eqref{eq:impl2}. Otherwise $|\Shalf| \in \set{1,3,4}$. Then with probability at least $1/8$, we have that 
$\bigabs{\sum_{i \in \Shalf} \alpha_i X_i} \in \cup_{b \in \set{1,3,4}}[\tfrac{b}{2}-4 \eta, \tfrac{b}{2}+4\eta]$. Since $\eta < 1/20$, and $X_i$ are independent we get \eqref{eq:impl1} since
$$\Pr_X \Big[\bigabs{\sum_{i \in S} \alpha_i X_i} \notin [0,2\eta] \cup [1-2\eta, 1+2\eta] \Big] \ge \Pr_X \Big[\bigabs{\sum_{i \in S} \alpha_i X_i} \in [\tfrac{b}{2}-12 \eta,  \tfrac{b}{2}+12\eta] \Big] \ge \frac{(1-\eps)\kappa}{8} \ge \frac{\kappa}{16}.$$

\noindent \emph{{\bf Case 4:} $\Shalf \cup \Small=S$ and}
$$\Pr_X \Big[\bigabs{\sum_{i \in \Small} \alpha_i X_i} \ge 8 \eta \Big] \ge \eps \kappa.$$
In this case we just apply Lemma~\ref{lem:soundness:helper} with $\kappa'=\eps \kappa$ to obtain \eqref{eq:impl1}.


\end{proof}

\anote{Write down the test and the completeness case.}
We now show the soundness analysis of Step 2 in Algorithm~\ref{ALG:rad}. This will be useful to handle the case when there are most of the contribution to $\iprod{z,Ax}$ comes from two columns. 

\anote{Move the test and analysis elsewhere? Also, need to account for flipping. Is there flipping problems? }
\begin{lemma}\label{lem:pjango}
Let $\eta \in (0,\tfrac{1}{80})$ and $\kappa'>0$ satisfy $\kappa'= 1/(4m^2)$. Let $i_1, i_2 \in [m]$ satisfy $\abs{\iprod{z,A_{i_1}}}, \abs{\iprod{z,A_{i_2}}} \in (\tfrac{1}{2}-\eta, \tfrac{1}{2}+\eta)$, and $q_0= \Pr_{x \sim \calD} \big[ |x_{i_1}|=|x_{i_2}|=1 \big]$, and suppose 
\begin{equation}\label{eq:pjango:cond} \Pr_{x \sim \calD} \Big[ \bigabs{ \sum_{i \neq i_1, i_2} x_i \iprod{z,A_i}} \ge 8 \eta \Big]< \kappa' q_0 .\end{equation}
If $\tcalD$ denotes the conditional distribution conditioned on $\iprod{z,Ax} \in (1 - 10\eta, 1+10\eta)$, then
$\norm{\E_{x \sim \tcalD} Ax} > (1+\sqrt{2})/2. $
\end{lemma}
\begin{proof}
Let us denote by $\alpha_{i_1}=\iprod{z,A_{i_1}}, \alpha_{i_2}=\iprod{z,A_{i_2}}$, and $\sigma_{i_1}=\sign(\alpha_{i_1}), \sigma_{i_2}=\sign(\alpha_{i_2})$. Note that $\abs{\alpha_{i_1}}, \abs{\alpha_{i_2}} \in (\tfrac{1}{2}-\eta, \tfrac{1}{2}+\eta)$. Firstly, $\Pr_{x \sim \calD}[x_{i_1}=\sigma_{i_1} ~\wedge~  x_{i_2}=\sigma_{i_2}] = q_0/4$. 

For all samples $x$ such that $x_{i_1}=\sigma_{i_1}, x_{i_2}=\sigma_{i_2}$ (and hence $\text{supp}(x) \ni i_1, i_2$), if $\big| \sum_{i \ne i_1, i_2} \iprod{z,A_i} \big|\le 8\eta$, then $\iprod{z,Ax} \in (1-10\eta, 1+10\eta)$. Hence, we have from \eqref{eq:pjango:cond} that
\begin{align}
\Pr_{x \sim D}\Big[ x_{i_1}=\sigma_{i_1} ~\wedge~ x_{i_2}=\sigma_{i_2} ~\wedge~ \iprod{z,Ax} \in (1-10\eta, 1+10\eta) \Big] &\ge q_0 - \kappa'q_0 \ge q_0(1-\kappa') . \nonumber\\
\Pr_{x \sim D}\Big[ \iprod{z,Ax} \in (1-10\eta, 1+10\eta) ~ \big|~x_{i_1}=\sigma_{i_1} ~\wedge~ x_{i_2}=\sigma_{i_2} \Big] &\ge 1-\kappa' .
\label{eq:pjango:oneside}
\end{align}

\noindent From \eqref{eq:pjango:cond} since $x_{i_1}, x_{i_2} \in \set{-1,0,1}$ and our choice of $18\eta<1/4$, we have for all but $\kappa' q_0$ fraction of all the samples 
\begin{align*}
\iprod{z, Ax} \in (1 - 10 \eta, 1+10\eta) ~ &\implies~ 1-18\eta\le \frac{1}{2}(\sigma_{i_1}x_{i_1}+\sigma_{i_2} x_{i_2}) \le 1+18\eta\\ &\implies~ x_{i_1}=\sigma_{i_1}, x_{i_2}=\sigma_{i_2},\\
\text{Hence,} \Pr_{x \sim \tcalD} \Big[  x_{i_1}=\sigma_{i_1} \wedge x_{i_2}=\sigma_{i_2}  \Big] &\ge 1-\frac{\kappa' q_0}{\Pr_{x \sim D}\Big[|\iprod{y,Ax}| \in (1-10 \eta, 1+10\eta)\Big]} \\
&\ge 1-\frac{\kappa' q_0}{q_0(1-\kappa')} \ge 1-2\kappa'.\\
\text{Combined with \eqref{eq:pjango:oneside} we have},~ &\norm{\tcalD - \calD_{| x_{i_1}=\sigma_{i_1} , x_{i_2}=\sigma_{i_2}}}_{TV} \le 3\kappa'.
\end{align*}
Suppose we denote the vector $u=\sigma_{i_1}A_{i_1}+\sigma_{i_2}A_{i_2}$ and $\ubar=\E_{x \sim \tcalD} [Ax]=\sum_{i \in [m]} A_i \E_{x \sim \tcalD} [x_i]$, then
\begin{align*}
\norm{u-\ubar}_2 &\le 
\sum_{i \in [m] } \norm{A_{i}}_2 \cdot \Abs{\E_{\calD }[x_i | x_{i_1}=\sigma_{i_1}, x_{i_2}=\sigma_{i_2}]- \E_{\tcalD}[x_i]} \\
&\le \sum_{i \in [m]} 6\kappa' \le 6\kappa' m.
\end{align*}

\end{proof}

We now give the soundness analysis of Algorithm {\sc TestCol\_Rad}
\anote{Take care of the flipping problem.}

\begin{proof}[Proof of Lemma~\ref{lem:rad:soundness}]
Let $T_{1/2}=\set{i \in [m]: \abs{\alpha_i} \in (\tfrac{1}{2}-\eta, \tfrac{1}{2}+\eta)}$. Firstly from Lemma~\ref{lem:largevalues}, $|T_{1/2}| \le 4$. Let 
$$\kappa(S)= \Pr_{x \sim \calD} \Big[\bigabs{\sum_{i \in S} \iprod{z,A_i} x_i } \in (1- \eta,1+\eta) ~\big|~ \supp(x)=S \Big ].$$
Note that $\kappa=\sum_S q(S) \kappa(S)$. 

\paragraph{Case $|T_{1/2}| \le 1$ or more generally, if} 
\begin{equation}\label{eq:sound:case1}
\Pr_{x \sim \calD}\Big[ |\text{supp}(x) \cap T_{1/2}| \le 1~ \wedge~ |\iprod{z,Ax}| \in (1 - \eta,1+\eta) \Big] \ge \eps \kappa,
\end{equation}
for $\eps\in (0,1/2)$ being a sufficiently small constant (we can choose $\eps=1/2$). \anote{choose later.}
For any fixed support $S$ such that $|S \cap T_{1/2}|\le 1$, applying Lemma~\ref{lem:unnormal} we get from \eqref{eq:impl1} that 
\begin{align*}
& \Pr_{x \sim \calD} \Big[ \abs{\iprod{z,Ax}} \notin [0, 2\eta) \cup (1-2\eta, 1+2\eta) ~\big|~ \supp(x)=S \Big] \ge \min\Set{\frac{\eps \kappa(S)}{32},\frac{c_1}{2}}.\\
\text{Hence }& \Pr_{x \sim \calD} \Big[ \big|\iprod{z,Ax} \big| \notin [0, 2\eta) \cup (1-2\eta, 1+2\eta) \Big] \ge \sum_{S : |S \cap T_{1/2}|\le 1} q(S) \cdot \frac{\eps c_1 \kappa(S)}{64} 
\ge \frac{c_1 \kappa}{128}.
\end{align*}

Further, $c_1 \kappa N \ge \Omega(\log(1/\gamma))$.
Hence, using Chernoff bounds we have with probability at least $(1-\gamma)$ that if $\kappa N$ samples $x$ satisfy $\abs{\iprod{z,Ax}} \in [1-\eta,1+\eta]$, then $\tfrac{c_1}{256} \kappa N$ samples satisfy $\abs{\iprod{z,Ax}} \in [2\eta, 1-2\eta]$. Hence $z$ fails the test with probability at least $1-\gamma$.

\paragraph{Case $|T_{1/2}| = 2$.}
Let $T_{1/2}=\set{i_1, i_2}$. 
Since the lemma is true when \eqref{eq:sound:case1} holds, we can assume
\begin{equation}\label{eq:sound:case3}
\Pr_{x \sim \calD}\Big[ i_1, i_2 \in \text{supp}(x) ~ \wedge~ |\iprod{z,Ax}| \in (1 - \eta,1+\eta) \Big] \ge (1-\eps)\kappa.
\end{equation}
In particular, we have $\sum_{S \ni i_1, i_2} q(S) \ge (1-\eps)\kappa \ge \kappa/2$. 

Suppose \eqref{eq:pjango:cond} holds, then we have that if $\tcalD$ is the conditional distribution given by Lemma~\ref{lem:pjango} $\bignorm{\E_{x \sim \tcalD}[ Ax]}_2 \ge (1+\sqrt{2})/2$. Further, as before in Lemma~\ref{lem:purify}, we can apply Hoeffding bounds since $N \ge c_2 C n k\eta_0^{-2} \kappa^{-1} \log(n/\gamma)$ to conclude that with probability at least $1-\gamma$, 
the vector $z'$ in Algorithm~\ref{ALG:rad} has norm at least $(1+\sqrt{2})/2- \eta_0>1.1$; hence $z$ fails the test. 
So, we may assume that
\eqref{eq:pjango:cond} does not hold i.e., 
\begin{equation}
\Pr_{x \sim \calD} \Big[\abs{ \sum_{i \neq i_1, i_2} x_i \iprod{z,A_i}} \ge 8 \eta \Big] \ge \kappa' q^{(2)},
\end{equation}
where $q^{(2)}=\Pr_{x \sim \calD} \big[ i_1, i_2 \in \text{supp}(x) \big]$. Note that the support distribution $\calD_s$ and the value distribution $\calD_v$ are independent. Let $E_S$ be the event that $\Big[\Pr_{x_S \sim \calD_v}\big[\abs{\sum_{i \in S\setminus \set{i_1, i_2}} \iprod{z,A_i}x_i} \ge 8\eta \big] \ge \kappa'/2 \Big]$.
Hence we have
\begin{align*}
\sum_{S \ni i_1, i_2} q(S) \times \Pr_{x_S \sim \calD_v}\Big[ \abs{\sum_{i \in S\setminus \set{i_1, i_2}} \iprod{z,A_i}x_i} \ge 8\eta \Big]  &\ge \kappa' \sum_{S \ni i_1, i_2} q(S)\\
\text{By a simple averaging argument,}~ \Pr_{S \sim \calD_s} [E_S] &\ge \frac{\kappa'}{2} \cdot q^{(2)}.
\end{align*}

Consider any fixed set $S$ such that $E_S$ is true i.e., 
$\Pr_{x_S \sim \calD_v}\big[\abs{\sum_{i \in S\setminus \set{i_1, i_2}} \iprod{z,A_i}x_i} \ge 8\eta \big] \ge \kappa'/2$. From Lemma~\ref{lem:soundness:helper}, for some absolute constant $c>0$, 
$$\Pr_{x_S \sim \calD_v}\big[\abs{\iprod{z,Ax}} \notin I_\eta ~\big|~ \text{supp}(x)=S \big] \ge c\kappa'. $$

Combined with $\Pr[E_S] \ge \kappa' q^{(2)}/2$ and $q^{(2)} \ge \kappa/2$, we get
\begin{align*}
\Pr\Big[ |\iprod{z,Ax}| \notin I_{2\eta} ~\wedge~ i_1,i_2 \in \text{supp}(x) \Big] &\ge \frac{\kappa' q^{(2)}}{2} \times c \kappa'\ge c'' (\kappa')^2\kappa=\frac{c'\kappa}{m^4}.
\end{align*}

As before, we can now apply Chernoff bounds since $c' \kappa N/m^4 \ge \Omega(\log(1/\gamma))$, to conclude that if $\kappa N$ samples $x$ satisfy $\abs{\iprod{z,Ax}} \in [1-\eta,1+\eta]$, then $c' \kappa N/m^4$ samples satisfy $\abs{\iprod{z,Ax}} \in (2\eta, 1-2\eta)$. Hence $z$ fails the test with probability at least $1-\gamma$.

\paragraph{Case $|T_{1/2}| \in \set{3,4}$.}
Let us suppose $|T_{1/2}|=3$ (an almost identical argument works for $|T_{1/2}|=4$). Let $i_1, i_2, i_3 \in T_{1/2}$. Consider samples that contain $i_1,i_2, i_3$ in their support i.e., $i_1, i_2,i_3 \in S$ 
For any $S \ni i_1, i_2, i_3$, let $q(S)=\Pr_{x \sim \calD} [\text{supp}(x)=S]$. Hence, $\sum_{S \ni i_1, i_2, i_3} q(S) \ge q_0$. Since the $x_i$ are independent, we have 
$$\forall S \supset \set{i_1, i_2, i_3},~\Pr_{x \leftarrow \calD}\Big[\sum_{i \in S \cap T_{1/2}} x_i \iprod{z,A_i} \in [\tfrac{3}{2} - 4\eta,\tfrac{3}{2}+4 \eta] ~ \big| ~ \text{supp}(x)=S\Big] \ge \frac{1}{8}  .$$

\noindent Consider a fixed support $S \ni i_1, i_2, i_3$ and suppose on the one hand that for $I'=[-\tfrac{1}{2}-6\eta, -\tfrac{1}{2} + 6\eta] \cup [-\tfrac{3}{2}-6\eta, -\tfrac{3}{2} + 6\eta] \cup [-\tfrac{5}{2}-6\eta, -\tfrac{5}{2} + 6\eta]$,
\begin{align*}
\Pr_{x \sim \calD} \Big[ \sum_{i \in S \setminus T_{1/2}} \iprod{z,A_i} x_i \in I' ~\big|~ \text{supp}(x)=S \Big] &< \tfrac{1}{32}\\
\text{then, } \Pr_{x \leftarrow \calD}\Big[\text{supp}(x)=S ~\wedge~ \sum_{i \in S \cap T_{1/2}} x_i \iprod{z,A_i} \in [\tfrac{3}{2} - 4\eta,\tfrac{3}{2}+4 \eta] \Big] &\ge \frac{q(S)}{32}
\end{align*}
Otherwise, for some $b \in \set{-\tfrac{1}{2},-\tfrac{3}{2},-\tfrac{5}{2}}$, we have that 
$$\Pr_{x \sim \calD} \Big[ \sum_{i \in S \setminus T_{1/2}} \iprod{z,A_i} x_i \in [b-6\eta,b + 6\eta]  ~\big|~ \text{supp}(x)=S \Big] \ge \tfrac{1}{96}.$$
Applying Lemma~\ref{lem:unnormal} with $\beta=\tfrac{1}{15}$,
\begin{align*}
\Pr_{x \sim \calD} \Big[ \sum_{i \in S \setminus T_{1/2}} \iprod{z,A_i} x_i \in [\tfrac{b}{30}-6\eta, \tfrac{b}{10} + 6\eta]  ~\big|~ \text{supp}(x)=S \Big] &\ge \min \Set{\tfrac{1}{192}, c_0} \ge c_0.\\
\text{Since } \tfrac{|b|}{10} < \tfrac{1}{4}, ~ \Pr_{x \sim \calD} \Big[ \text{supp}(x)=S ~\wedge ~\iprod{z,Ax} \ge 1+2\eta \Big] &\ge c_0 q(S).
\end{align*}

Combining the two cases, and since $c_0 < 1/32$, 
\begin{align*}
\forall S\supset \set{i_1, i_2, i_3},~\Pr_{x \sim \calD} \Big[ \text{supp}(x)=S ~\wedge ~\iprod{z,Ax} \ge 1+2\eta \Big] &\ge c_0 q(S).\\
\text{Summing over }S\supset \set{i_1, i_2, i_3} ~\Pr_{x \sim \calD} \Big[\iprod{z,Ax} \ge 1+2\eta \Big] &\ge c_0 q_0.
\end{align*}

Again, we can use Chernoff bounds since $c_0 q_0 N =\Omega(\log(1/\gamma))$ to conclude that with probability at least $1-\gamma$, $\iprod{z,Ax} \notin I_{2\eta}$ for at least $c_0 q_0 N/2$ samples, thus failing the test.  


\end{proof}

We now wrap up the proof of Theorem~\ref{thm:rad:testing}. 
The proof of Corollary~\ref{corr:rad:identifiability} is identical to the proof of Corollary~\ref{corr:sr:identifiability} (we just use Theorem~\ref{thm:rad:testing} as opposed to Theorem~\ref{thm:sr:testing}). So we omit it here.

\begin{proof}[Proof of Theorem~\ref{thm:rad:testing}]
The proof follows in a straightforward way by combining Lemma~\ref{lem:rad:completeness} and Lemma~\ref{lem:rad:soundness}. Firstly, note that $q^{(1)} \ge q_0$ . If $\norm{z-bA_i}_2$ for some $i \in [m], b \in \set{-1,1}$ then from Lemma~\ref{lem:rad:completeness}, we have that with probability at least $1-\gamma/2$ that $\abs{\iprod{z,\ysamp{r}}} \notin I_\eta$ for at most $2\kappa_0 N$ samples, $\abs{\iprod{z,\ysamp{r}}} \in [1-\eta, C(1+\eta)]$ for at least $q_{0} N/4$ samples, and finally $\norm{z'}_2 \le 1+\eta_0<1.1$ where $z'$ is the vector computed in step 3 of Algorithm~\ref{ALG:rad}. Hence it passes the test, proving the completeness case. 

On the other hand, from Lemma~\ref{lem:rad:soundness} applied with $\kappa=q_{0}/8$ and since $\min\set{\tfrac{1}{32}, c_4 \eta} q_{0} \ge 2 \kappa_1=2c_5 q_{0} \eta$ (for our choice of $c_5$), we also get that if $z$ passes the test, then with probability at least $1-\gamma/2$, we have $\abs{\iprod{z,A_i}} \ge 1-4\eta$ for some $i \in [m]$ as required. As before, from our choice of parameters $\sqrt{8 \eta}< \eta_2:=c_1/(8C \log(\tfrac{mn}{q_{\min}\eta_0}))$. Hence applying Lemma~\ref{lem:purify} we also get that $\norm{\widehat{z}-bA_i}_2  \le \eta_0$ with probability at least $1-\gamma/2$. Combining the two, we get the soundness claim.


\end{proof}

\subsection{Non-identifiability for arbitrary support distribution}


\begin{proposition}[Non-identifiability] \label{prop:non-identifiability}
There exists two different incoherent dictionaries $A,B$ which are far apart i.e., $\min_{\pi \in \text{perm}_m, b \in \set{-1,1}^m } \sum_{i \in [m]} \norm{A_i - b_i B_i}_2^2 = \Omega(1)$, and corresponding support distributions $\Ds_A, \Ds_B$ (the value distribution in both cases is the Rademacher distribution), such that if $P_A$ is the distribution over the samples  $y=Ax$ when $x \sim \calD_A=\Ds_A \odot \Dv$ and $P_B$ is the distribution over samples $y=Bx$ when $x \sim \calD_B=\Ds_B \odot \Dv$, then $P_A$ and $P_B$ are identical. 

Moreover every pair of columns $i_i, i_2 \in [m]$ occur with non-negligible probability in both the support distributions $\Ds_A$ and $\Ds_B$.
\end{proposition}

\begin{proof}
Our construction will be based on a symmetric construction with $4$ vectors. This can be extended to the case of larger $m$ by padding this construction with $m-4$ other random columns, or combining with many (randomly rotated) copies of the same construction. 

Let $A_1, A_2,\dots, A_4$ be a set of $4$ orthogonal unit vectors in $n$ dimensions. Consider four unit vectors given by 
\begin{align}
B_1 = \tfrac{1}{2}(A_1+A_2+A_3+A_4) ,& ~~B_2 = \tfrac{1}{2}(A_1+A_2 - A_3 - A_4),\nonumber\\
B_3 = \tfrac{1}{2}(A_1 - A_2 - A_3 + A_4),& ~~ B_4 = \tfrac{1}{2}(A_1 - A_2+A_3 - A_4). 
\end{align}

Note that these columns $B_1, B_2, B_3, B_4$ are also pairwise orthonormal. Further, $A_1, A_2, A_3, A_4$ can also be represented similarly as balanced $\set{+\tfrac{1}{2}, -\tfrac{1}{2}}$ combinations (since the inverse of the normalized Hadamard matrix is itself). The alternate dictionary $B$ is comprised of the columns $(B_1,B_2, B_3, B_4, A_5, \dots, A_m )$. 

The non-identifiability of the model follows from the following simple observation, which can be verified easily.

\begin{observation}
Any $\set{+1, -1}$ weighted combination of exactly two out of the four columns $\set{B_1,B_2,B_3,B_4}$ has one-one correspondence with a $\set{+1, -1}$ weighted combination of exactly two out of the four columns $\set{A_1,A_2,A_3,A_4}$.  
\end{observation}

Let $T:([4] \times \set{-1,1}) \times  ([4] \times \set{-1,1}) \to ([4] \times \set{-1,1}) \times ([4] \times \set{-1,1})$ represent this mapping as follows: for $i_1, i_2 \in [m], b_1, b_2 \in \set{1,-1}$, let $T(i_1,b_1, i_2, b_2)= (i'_1, b'_1, i'_2, b'_2)$. Note that this mapping is bijective. 

Now consider a support distribution $\Ds_A$ in which every sample $x$
contains {\em exactly two} of the co-ordinates $\set{1,2,3,4}$ in its support, and $k-2$ of the other $m$ co-ordinates at random.  
In other words, $\abs{\supp(x) \cap \set{1,2,3,4} }=2$ for every $x$ generated by $\calD_A$, and each of these pairs occur with equal probability i.e.,  
\begin{equation}\label{eq:equalpairs}
\forall i_1\ne i_2 \in \set{1,2,3,4}, ~ \Pr_{x \sim \calD_A}\Big[\supp(x) \cap \set{1,2,3,4}]=\set{i_1, i_2}\Big] = \frac{1}{{4 \choose 2}}=\frac{1}{6}.
\end{equation}

Consider any sample $x$ generated by $\calD_A$ such that $\supp(x) \cap \set{1,2,3,4} =\set{i_1,i_2}$, and let $x_{i_1}=b_{i_1}, x_{i_2}=b_{i_2}$. Hence, 
$$y=Ax=b_{i_1} A_{i_1}+b_{i_2}A_{i_2} + \sum_{j \in [m] \setminus \set{1,2,3,4}} x_j A_j = b'_{i'_1} B_{i'_1}+b'_{i'_2} B_{i'_2} + \sum_{j \in [m] \setminus \set{1,2,3,4}} x_j A_j.$$

Hence for each sample $y=Ax$ with $x \sim \calD_A$, there is a corresponding sample $y=B x'$ that is given by the bijective mapping $T$, and vice-versa. Note that since each of the pairs occur equally likely, and each non-zero value is $\pm 1$ with equal probability, the distribution of $x'$ is given by $\calD_B=\Ds_B \odot \Dv$ where the support distribution $\Ds_B$ has each of these pairs from $\set{1,2,3,4}$ occurring with equal probability $1/6$, analogous to \eqref{eq:equalpairs}.
Hence, it is impossible to tell if the dictionary is $\set{A_1, A_2, A_3, A_4, A_5,\dots,A_m}$ or $\set{B_1, B_2, B_3, B_4, A_5,\dots, A_m}$, thus establishing non-identifiability.   
\end{proof}
\begin{remark}
We note that the above construction can be extended to show non-identifiability in a stronger sense. By having $k/4$ blocks of $4$ vectors, where each block is obtained by applying a random rotation to the block $B_1, B_2, B_3, B_4$ of $4$ vectors used in the above construction having a support distribution that has exactly two out of the four indices in each block, we can conclude it is impossible to distinguish between $2^{\Omega(k)}$ different dictionaries. 
\end{remark}


\section{Efficient Algorithms by Producing Candidate Columns}
\label{sec:semi-random-efficient}

The main theorem of this section is a polynomial time algorithm for recovering incoherent dictionaries when the samples come from the semirandom model. 

\anote{4/13: Change T to N below.  4/22: Added for any $\eps>0$.}
\begin{theorem}
\label{thm:main-full-algorithm}
Let $A$ be a $\mu$-incoherent $n \times m$ dictionary with spectral norm $\sigma$. For any $\epsilon>0$, given $N= \poly(k,m,n,1/\eps,1/\beta)$ samples from the semi-random model $\sModel$, Algorithm \textsc{RecoverDict} with probability at least $1 - \frac{1}{m}$, outputs a set $W^*$ such that
\begin{itemize}
\item For each column $A_i$ of $A$, there exists $\hat{A}_i \in W^*$, $b \in \set{\pm 1}$ such that $\|A_i - b\hat{A}_i\| \leq \epsilon$.
\item For each $\hat{A}_i \in W^*$, there exists a column $A_i$ of $A$, $b \in \set{\pm 1}$ such that $\|\hat{A}_i - b{A}_i\| \leq \epsilon$,
\end{itemize}
provided 
$k \leq \sqrt{n}/\nu_1(\tfrac{1}{m}, 16)$.
Here $\nu_1(\eta, d) := c_1 \tau \mu^2\bigparen{C(\sigma^2 + \mu \sqrt{\frac m n}) \log^2(n/\eta)}^{d}$, 
$c_1>0$ is a constant (potentially depending on $C$), and the polynomial bound for $N$ also hides a dependence on $C$.
\end{theorem}
\anote{What's the dependence on $\mu$?}
\anote{4/9: I don't understand why there isn't a dependence on $1/\nu(\eta,2L)$ for $k$ -- if yes, that should give dependencies on $C^{O(L)} \mu^{O(L)} (m/n)^{^{O(L)}}$... Am I missing something?}
The bound above is the strongest when $m = \widetilde{O}(n)$ and $\sigma=\widetilde{O}(1)$, in which case we get guarantees for $k=\widetilde{O}(\sqrt{n})$, where $\widetilde{O}$ also hides dependencies on $\tau, \mu$. 
 However, notice that we can also handle $m = O(n^{1+\epsilon_0}), \sigma=O(n^{\eps_0})$, for a sufficiently small constant $\epsilon_0$ at the expense of smaller sparsity requirement -- in this case we handle $k=\widetilde{O}(n^{1/2-O(\eps_0)})$ (we do not optimize the polynomial dependence on $\sigma$ in the above guarantees).
The above theorem gives a polynomial time algorithm that recovers the dictionary (up to any inverse polynomial accuracy) as long as $\beta$, the fraction of random samples is inverse polynomial. In particular, the sparsity assumptions and the recovery error do not depend on $\beta$. In other words, the algorithm succeeds as long we are given a few ``random'' samples (say $N_0$ of them), even where there is a potentially a much larger polynomial number $N \gg N_0$ of samples with arbitrary supports. 
We remark that the above algorithm is also robust to inverse polynomial error in each sample; however we omit the details for sake of exposition. 

Our algorithm is iterative in nature and crucially relies on the subroutine \textsc{RecoverColumns} described in Figure~\ref{ALG:Single_Recovery}. Given data from a semi-random model \textsc{RecoverColumns} helps us efficiently find columns that appear frequently in the supports of the samples. More formally, if $q_{\max}$ is the probability of the most frequently appearing column in the support of the semi-random data, then the subroutine will help us recover good approximations to each column $A_i$ such that $q_i \geq q_{\max}/\log m$.

Our algorithm for recovering large frequency columns is particularly simple as it just computes an appropriately weighted mean of the data samples. \anote{4/22: Shall we delete the rest of this up until the equation? It is already in Technical overview...} 
It is loosely inspired by the initialization procedure in~\cite{AGMM15}. The intuition 
comes from the fact that if the data were generated from a completely random model and if $u^{(1)},u^{(2)}$ and $u^{(3)}$ are three samples that contain $A_i$ (with the same sign) then $\E[\iprod{u^{(1)},y}\iprod{u^{(2)},y}\iprod{u^{(3)},y}y]$ is very close to $A_i$ (provided $k \le n^{1/3}$). 
Using this one can recover good approximations to each column $A_i$ if the model is truly random. However, in the case of semi-random data one cannot hope to recover all the columns using this since the adversary can add additional data in such a way that a particular column's frequency becomes negligible or that the support patterns of two or more columns become highly correlated. 
Nonetheless, we show that by computing a weighted mean of the samples where the weights are computed by looking at higher order statistics, one can hope to recover columns with large frequencies. The guarantee of the subroutine (see Figure~\ref{ALG:Single_Recovery}) is formalized in Theorem~\ref{thm:single-recovery-main} below. 
In order to do this we will look at the statistic 
\begin{align}
\E_y[\iprod{u^{(1)},y}\iprod{u^{(2)},y}\iprod{u^{(3)},y}\dots \iprod{u^{(2L-1)},y}~ y]
\label{eq:alg-statistic}
\end{align}
for a constant $L \geq 8$. Here $u^{(1)}, u^{(2)}, \dots, u^{(2L-1)}$ are samples that all have a particular column, say $A_1$, in their support such that $A_1$ appears with the same sign in each sample. We will show that if $A_1$ is a high frequency column, i.e., $q_1 \geq \frac{q_{\max}}{\log m}$, then one can indeed recover a good approximation to $A_1$. Notice that while the adversarial samples added might not have $u^{(1)}, u^{(2)}, \dots, u^{(2L-1)}$ with certain desired properties (that are needed for procedure to work), the random portion of the data will contain such samples with high probability.
\anote{4/22: Added for any $c>0$, any $\epsilon>0$. Used to be "exists constants $c>0$"}
\begin{theorem}
\label{thm:single-recovery-main}
There exist constant $c_1>0$ (potentially depending on $C$) such that the following holds for any $\eps>0$ and constants $c>0$, $L \ge 8$.
Given 
$\poly(k,m,n,1/\eps,1/\beta)$ samples from the semi-random model $\sModel$, Algorithm \textsc{RecoverColumns}, with probability at least $1-\frac{1}{m^c}$, outputs a set $W$ such that
\begin{itemize}
\item For each $i$ such that $q_i \geq q_1/\log m$, $W$ contains a vector $\hat{A}_i$, and there exists $b \in \set{\pm 1}$ such that $\|A_i - b\hat{A}_i\| \leq \epsilon$. 
\item For each vector $\hat{z} \in W$, there exists $A_i$ and $b \in \set{\pm 1}$ such that $\|\hat{z} - bA_i\| \leq \epsilon$,
\end{itemize}
provided 
$k \leq \sqrt{n}/(\nu(\frac 1 m, 2L) \tau \mu^2).
$ Here 
$\nu(\eta,d):=c_1 \bigparen{C(\sigma^2 + \mu \sqrt{\frac m n}) \log^2(n/\eta)}^{d}$, and the polynomial bound also hides a dependence on $C$ and $L$.
\end{theorem}

Our overall iterative approach outlined in the \textsc{DictLearn} procedure in Figure~\ref{ALG:Full_Recovery} identifies frequently occurring columns and then re-weighs the data in order to uncover more new columns. We will show that such a re-weighting can be done by solving a simple linear program. While our recovery algorithm is quite simple to implement and simply outputs an appropriately weighted mean of the samples, its analysis turns out to be challenging. In particular, to argue that with high probability over the choice of samples $u^{(1)}, \dots, u^{(2L-1)}$, the statistic in \eqref{eq:alg-statistic} is close to one of the frequently occurring columns of $A$, we need to prove new concentration bounds~(discussed next) on polynomials  of random variables that involve {\it rarely occurring} events. 
\anote{4/22: Removed line about hope about "broad applications"...}

Next we provide a roadmap for the remainder of this section. In Section~\ref{subsec:conc-bounds} we develop and prove new concentration bounds for polynomials of rarely occurring random variables. The main proposition here is Proposition~\ref{prop:concentration-degree-d} which provides these concentration bounds in terms of the $\|\|_{2,\infty}$ norm of various ``flattenings'' of the tensor of the coefficients. In Section~\ref{subsec:instan-conc} we use Proposition~\ref{prop:concentration-degree-d} to derive various implications specific to the case of semi-random model for dictionary learning. This will help us argue about concentration of various terms that appear when analyzing \eqref{eq:alg-statistic}. Building on the new concentration bounds, in Section~\ref{sec:semi-random-recovery} we provide the proof of Theorem~\ref{thm:single-recovery-main}. Finally, in Section~\ref{sec:semi-random-full-alg} we prove Theorem~\ref{thm:main-full-algorithm}.



\subsection{Concentration Bounds for Polynomials of Rarely Occurring Random Variables}
\label{subsec:conc-bounds}
In this section we state and prove new concentration bounds involving polynomials of rarely occurring random variables. We first prove the general statement and then present its implications that will be useful for dictionary learning. We recall the distributional assumptions about the vectors. Consider the following distribution $\Zeta$ over sample space $[-C,C]^m$: a sample $\zeta=(\zeta_1, \zeta_2, \dots, \zeta_m) \sim \Zeta$ is sparsely supported with $\norm{\zeta}_0 = p m$. The support $S \subset {[m] \choose p m}$ is picked according to a support distribution that is $\tau$-negatively correlated as defined in Section~\ref{sec:prelims} ($\tau=1$ when it is uniformly at random)\footnote{Here, the each entry can also be chosen to be non-zero independently with probability at most $p$.}, and conditioned on the support $S \subset [m]$, the random variables $(\zeta_i: i \in S)$ are i.i.d. symmetric mean-zero random variables picked according to some distribution $\Dv$ which is supported on $[-C,-1] \cup [1,C]$. We emphasize that while the proposition will also handle these $\tau$-negatively correlated support distributions (where $\tau=\omega(1)$), the following statements are interesting even when $\tau=1$, or when each entry is non-zero with independent probability of $k/m$. 

\begin{definition}\label{def:flattenednorms}
Given any tensor $T \in \R^{n^{\times d}}$, and any subset $\Gamma \subseteq [d]$ of modes, we denote the $(\Gamma,\infty)$ flattened norm by 
$$\norm{T}_{\Gamma,\infty}= \norm{\Tgamma}_{2, \infty}= \max_{J_1 \subset [m]^{\Gamma}} \norm{\Tgamma_{J_1}}_2, $$ 
where $\Tgamma$ is the flattened matrix of dimensions $[m]^\Gamma \times [m]^{[d] \setminus \Gamma}$. Recall that for $M \in \R^{n_1 \times n_2}$, $\norm{M}_{2,\infty}=\max_{i \in [n_1]} \norm{(M^T)_i}_2$ is the maximum $\ell_2$ norm of any row of $M$.
\end{definition}
Note that when $\Gamma=\emptyset$ this corresponds to the Frobenius norm of $T$, while $\Gamma=[d]$ corresponds to the maximum entry of the $T$. 
Our concentration bound will depend on the $\norm{\cdot}_{2, \infty}$ matrix operator norms of different ``flattenings'' of the tensor $T$ into matrices \footnote{ This is reminiscent of how the bounds of \cite{Latala,AdamczakWolff} depend on spectral norms of different flattenings. However, in our case we get the $\norm{\cdot}_{2,\infty}$ norms of the flattenings because of our focus on {\em rarely occuring random variables} .}

\begin{proposition}\label{prop:concentration-degree-d}
Let random variables $\zeta^{(1)}, \zeta^{(2)}, \dots, \zeta^{(d)}$  be i.i.d. draws from $\Zeta$ (sparsity $k \le p m$), \anote{4/22: made it $k \le pm$ from $k=pm$} and let $f$ be a degree $d$ multilinear polynomial in $\upzeta{1}, \dots, \upzeta{d} \in \R^m$ given by 
$$f(\upzeta{1}, \dots, \upzeta{d}):= \sum_{(j_1,j_2 \dots, j_d) \in [m]^d} T_{j_1, \dots, j_d} \prod_{\ell=1}^d \upzeta{\ell}_{j_\ell}, 
$$
with an upper bound $B>0$ on the frobenius norm $\norm{T}_F \le B$, and let   
\begin{equation} \label{eq:flattening:factor}
\imbal= \sum_{\Gamma \subset [d]}~ \frac{\norm{T}_{\Gamma,\infty}^2}{B^2} \cdot (\tau p)^{-|\Gamma|} = \sum_{\Gamma} \Big(\frac{\norm{T}_{\Gamma,\infty}^2}{m^{d-|\Gamma|}}\Big) \Big( \frac{B^2}{m^{d}}  \Big)^{-1} \cdot \frac{1}{(\tau p m)^{|\Gamma|}} 
\end{equation}
Then, 
for any $\eta>0$
we have
\begin{equation}\label{eq:concentration-degree-d}
\Pr\Big[ \abs{f(\upzeta{1}, \dots, \upzeta{d})} \ge (C^2\log (2/\eta))^{d/2}  \min\set{\sqrt{\imbal} \log(2/\eta)^{d/2}, 1/\sqrt{\eta}} \cdot (\tau p)^{d/2} \norm{T}_F \Big] \le \eta. 
\end{equation}
\end{proposition}
Note that in the above proposition $p=k/m$, $\eta$ will typically be chosen to be $O(1/n)$ and $\tau=O(1)$.
The factor $\imbal$ measures how uniformly the mass is spread across the tensor. In particular, when all the entries of the entries are within a constant factor of each other, then we have $\imbal = \max_{\Gamma} O(1)/(p m)^{|\Gamma|}=O(1)$. In most of our specific applications, we will see that $\imbal=O(1)$ as long as $p m> \sqrt{m}$ (see Lemma~\ref{lem:conc-one-partition}); however when $k=p m$ is small, we can tolerate more slack in the bounds required in sparse coding. Finally, we remark that such bounds for multilinear polynomials can often to be used to prove similar bounds for general polynomials using decoupling inequalities ~\cite{dlP}.

\anote{Should this be phrased more generally? Also support picked independently with probability $p$ or uniformly from $k$ sparse vectors? Also, is it $n$ or $m$?}
Concentration bounds for (multilinear) polynomials of hypercontractive random variables are known giving bounds of the form $\Pr[g(x)>t \norm{g}_2] \le \exp\big(-\Omega(t^{2/d}) \big)$ ~\cite{Ryanbook}. More recently, sharper bounds that do not necessarily incur $d$ factor in the exponent and get bounds of the form $\exp(-\Omega(t^2))$ have also been obtained by Latala and Adamczak-Wolff~\cite{HansonWright, Latala, AdamczakWolff} for sub-gaussian random variables and random variables of bounded Orlicz $\psi_2$ norm. However, our random variables are {\em rarely supported} and are non-zero with tiny probability $p=k/m$. Hence, our random variables are not very hypercontractive (the hypercontractive constant is roughly $\sqrt{1/p}$). Applying these bounds directly seems suboptimal and does not give us the extra $p^{d/2}$ term in \eqref{eq:concentration-degree-d} that seems crucial for us. On the other hand, bounds that apply in the rarely-supported regime \cite{KimVu, SchudyS} typically apply to polynomials with non-negative coefficients and do not seem directly applicable. 

We deal with these two considerations by first reducing to the case of ``fully-supported random variables'' (i.e., random variables that are mostly non-zero), and then applying the existing bounds from hypercontractivity. The following lemma shows how we can reduce to the case of ``fully supported'' random variables. 


\begin{lemma}\label{lem:randomfrobnorm}
In the notation of Proposition~\ref{prop:concentration-degree-d}, let $T'=T_{S_1 \times \dots \times S_d}$ represent the tensor $T$ (with $\norm{T}_F \le B$) restricted to the block given by the random supports of $\upzeta{1}, \dots, \upzeta{d}$. Then 
any $\eta>0$, 
we have with probability at least $1-\eta$
\begin{equation}\label{eq:randomfrobnorm}
\norm{T'}_F^2 \le \min\set{\imbal \log(1/\eta)^d , 1/\eta} \cdot (\tau p)^{d} B^2. 
\end{equation}
\end{lemma}
We note that the concentration bounds due to Kim and Vu~\cite{KimVu} can be used to obtain similar bounds on the Frobenius norm of the random sub-tensor in terms of $\imbal$ (particularly when every r.v. is non-zero independently with probability $p$). These inequalities~\cite{KimVu} for non-negative polynomials of non-negative random variables have a dependence on the ``derivatives'' of the polynomial. However we give a self-contained proof below, to get the bounds of the above form and also generalize to the case when $\Dsr$ is $\tau$-negatively correlated support distributions.  
\begin{proof}
For any $\ell \in [d]$ and $j \in [m]$, let $Z^{(\ell)}_j \in \set{0,1}$ represent the random variable that indicates if $j \in S_\ell$ i.e., if $j$ is in the support of $\upzeta{\ell}_j$. Let 
$$S=\sum_{J=(j_1, \dots, j_d)\in [m]^d} (T_{j_1, \dots, j_d})^2 \upZ{1}_{j_1} \upZ{2}_{j_2}\dots \upZ{d}_{j_d}.$$
We have $E[S] \le B^2 (\tau p)^d$. We will show the following claim on the $t$th moment of $S$. 

\begin{claim}
$\E[S^{t}] \le (t^{d} \cdot \imbal)^{t-1} ((\tau p)^d B^2)^t$. 
\end{claim}

We now prove the claim inductively. The base case is true for $t=1$. Assume the statement is true for $t-1$. 
\begin{align*}
\E[S^t] &= \sum_{J^{(1)} \in [m]^d} \sum_{ J^{(2)}  \in [m]^d }  \dots \sum_{ J^{(t)} \in [m]^d } T^2_{J^{(1)}} \cdot T^2_{J^{(2)}} \dots T^2_{J^{(t)}} \prod_{\ell \in [d]} \E\Big[ \prod_{r \in [t]}\upZ{\ell}_{j^{(r)}_\ell} \Big]\\
&= \sum_{J^{(1)}} \dots \sum_{J^{(t-1)}} T^2_{J^{(1)}} \dots T^2_{J^{(t-1)}} \prod_{\ell \in [d]} \E\Big[\prod_{r \in [t-1]}\upZ{\ell}_{j^{(r)}_\ell} \Big] \sum_{J^{(t)} \in [m]^d} T^2_{J^{(t)}} \prod_{\ell \in [d]} \frac{ \E\Big[ \prod_{r \in [t]}\upZ{\ell}_{j^{(r)}_\ell} \Big] }{\E\Big[ \prod_{r \in [t-1]}\upZ{\ell}_{j^{(r)}_\ell} \Big] }
\end{align*}
Let $\Gamma \subseteq [d]$ denote the set of indices $\ell \in [d]$ that are not already present in $\set{j^{(1)}_\ell, j^{(2)}_\ell, \dots, j^{(t-1)}_{\ell}}$ i.e., $\Gamma=\set{\ell \in [d]: j^{(t)}_\ell \ne j^{(1)}_\ell, ~j^{(t)}_\ell \ne j^{(2)}_\ell,~ \dots,~j^{(t)}_\ell \ne j^{(t-1)}_{\ell} }$. Hence,
$$ \frac{ \E\Big[ \prod_{r \in [t]}\upZ{\ell}_{j^{(r)}_\ell} \Big] }{\E\Big[ \prod_{r \in [t-1]}\upZ{\ell}_{j^{(r)}_\ell} \Big] } = \begin{cases} p \tau & \text{ if } \ell \in \Gamma \\ 1 & \text{otherwise}. \end{cases}$$
Further, for each of the indices $\ell \in [d] \setminus \Gamma$, $j^{(t)}_\ell$ can take one of (at most) $t$ indices $\set{j^{(1)}_\ell,j^{(2)}_\ell, \dots, j^{(t-1)}_\ell }$. Since each of the terms is non-negative, we have by summing over all possible $\Gamma \subset [d]$,
\begin{align*}
\sum_{J^{(t)} \in [m]^d} T^2_{J^{(t)}} \prod_{\ell \in [d]} \frac{ \E\Big[ \prod_{r \in [t]}\upZ{\ell}_{j^{(r)}_\ell} \Big] }{\E\Big[ \prod_{r \in [t-1]}\upZ{\ell}_{j^{(r)}_\ell} \Big] } & \le \sum_{\Gamma \subseteq [d]} t^{(d-|\Gamma|)} \max_{J_{\Gamma^c} \in [m]^{d-|\Gamma|}}\sum_{J_\Gamma \in [m]^\Gamma} T^2_J (\tau p)^{|\Gamma|}\\
&\le (\tau p)^d B^2 \sum_{\Gamma \subseteq [d]} t^{(d-|\Gamma|)}  (\tau p)^{|\Gamma|-d} \max_{J_{\Gamma^c} \in [m]^{d-|\Gamma|}} \frac{\sum_{J_\Gamma \in [m]^\Gamma} T^2_J}{B^2} \\
&\le t^d \cdot \imbal (\tau p)^d B^2.
\end{align*}
\begin{align*}
\text{Hence }\E[S^t] &\le \sum_{J^{(1)}} \dots \sum_{J^{(t-1)}} T^2_{J^{(1)}} \dots T^2_{J^{(t-1)}} \prod_{\ell \in [d]} \E\Big[\prod_{r \in [t-1]}\upZ{\ell}_{j^{(r)}_\ell} \Big] \times  t^d \imbal (\tau p)^d B^2 \\
&\le (t^d \imbal)^{t-2} \E[S]^{t-1} \cdot \imbal t^d (\tau p)^d B^2 \le (t^d \imbal)^{t-1} \bigparen{(\tau p)^d B^2}^t,  
\end{align*}
hence proving the claim. Applying Markov's inequality with $\lambda=t^{d} \imbal^{1-1/t} \eta^{-1/t}$, 
\begin{align*}
\Pr\Big[ S \ge \lambda (\tau p)^d B^2 \Big] &
\le \frac{\E\Big[ S^t \Big]}{ \lambda^t \bigparen{(\tau p)^d B^2}^t} 
\le \frac{t^{td} \imbal^{t-1}}{\lambda^t} \le \eta\\
\text{Hence, } \Pr\Big[ S \ge \frac{t^d \imbal^{1-1/t}}{\eta^{1/t}} \cdot (\tau p)^{d} \norm{T}_F^2 \Big] &\le \eta.
\end{align*}
 Picking the better of $t=\log(1/\eta)$ and $t=1$ gives the two bounds. 
\end{proof}

We now prove Proposition~\ref{prop:concentration-degree-d} using Lemma~\ref{lem:randomfrobnorm} along with concentration bounds from hypercontractivity of polynomials of random variables. 

\begin{proof}[Proof of Proposition~\ref{prop:concentration-degree-d}]
A sample $\upzeta{1}, \dots, \upzeta{d}$ is generated as follows: first, the (sparse) supports $S_1, S_2, \dots, S_d \subseteq [n]$ are picked i.i.d and uniformly at randomly and then the values of $\upxi{1}, \upxi{d} \in [-C,C]^{p n}$ are picked i.i.d. from $\calD$. Suppose $T'=T_{|S_1| \times |\dots| \times |S_d|}$ represents the tensor restricted to the block given by the random supports, from Lemma~\ref{lem:randomfrobnorm} 
\begin{equation}
\Pr\Big[ \norm{T'}_F > \min\set{\sqrt{\imbal} \log(2/\eta)^{d/2}, \eta^{-1/2}} \cdot (\tau p)^{d/2} \norm{T}_F \Big] < \frac{\eta}{2}. \label{eq:degreedconc:1}
\end{equation}
The polynomial $f(\upzeta{1}, \dots, \upzeta{d})$ is given by
$$f(\upzeta{1}, \dots, \upzeta{d})=g(\upxi{1}, \dots, \upxi{d}):=\sum_{(i_1,\dots,i_d) \in [p n]^d} T'_{i_1, \dots, i_d} \prod_{\ell=1}^d \upxi{\ell}_{j_\ell}.$$
Further, the above polynomial $g$ is multi-linear (and already decoupled) with $\E[g]=0$, and 
\begin{align*}
\norm{g}_2^2=\E_{\xi}\Big[ 
\Big(\sum_{(j_1,\dots,j_d) \in [p n]^d} T'_{j_1, \dots, j_d} \prod_{\ell=1}^d \upxi{\ell}_{j_\ell} \Big)^2 \Big]
&= 
\sum_{(j_1,\dots,j_d) \in [p n]^d} (T'_{j_1, \dots, j_d})^2 \prod_{\ell=1}^d \E\big[ (\upxi{\ell}_{j_\ell})^2 \big]  \\
\text{Hence }\norm{g}_2 &\le C^d \norm{T'}_F.
\end{align*}
Further, the univariate random variables $\xi_{j_\ell}$ are hypercontractive with $\norm{\xi_{j_\ell}}_q \le C \norm{\xi_{j_\ell}}_2$. Using hypercontractive bounds for low-degree polynomials of hypercontractive variables (see Theorem 10.13 in \cite{Ryanbook}), 
$$\norm{g}_q \le \big(C^2 (q-1) \big)^d \norm{g}_2.$$
We now get the required concentration by consider a large enough $q$, by setting $q=t^{2/k}/(eC^2)$ and $t=\log(2/\eta)^{d/2}>(eC)^{d}$, 
\begin{align}
\Pr\Big[|g(\upxi{1},\dots,\upxi{d})| \ge t \norm{g}_2 \Big] &\le \frac{\norm{g}_q^q}{t^q \norm{g}_2^q} \le \left(\frac{(C^2 q)^{k/2}}{t} \right)^q \nonumber\\
\Pr\Big[ g(\upxi{1}, \dots, \upxi{d}) \ge t C^d \norm{T'}_F \Big] &\le \exp\Big(- \frac{d t^{2/d}}{2e C^2}\Big) \le \frac{\eta}{2}. \label{eq:degreedconc:2} 
\end{align}
By a union bound over the two events in \eqref{eq:degreedconc:1}, \eqref{eq:degreedconc:2}, we get \eqref{eq:concentration-degree-d}.

\end{proof}


We just state a simple application of the above proposition for $d=1$. This is analogous to an application of Bernstein bounds, except that the random variables are not independent (only the values conditioned on the non-zeros are independent) because the support distribution can be mildly dependent.
\begin{lemma}\label{lem:bernstein:correct}
There is a constant $c>0$ such that given $\alpha \in \R^m$ and $\zeta \sim \Zeta$, we have that 
$$\Pr\Big[ \bigabs{\sum_{j \in [m]} \alpha_j \zeta_j} \ge c \log(1/\eta) C \sqrt{\tau} \max\set{ \norm{\alpha}_2 p^{1/2}, \norm{\alpha}_\infty} \Big] \le \eta.$$
\end{lemma}
\begin{proof}
When $T= \alpha \in \R^m$, the two flattened norms correspond to the $\norm{\alpha}_2$ (when $\Gamma=\emptyset$) and $\norm{\alpha}_\infty$ (when $\Gamma=\set{1}$). Applying the bounds with $\rho=\max\set{\norm{\alpha}_2, \norm{\alpha}_\infty p^{-1/2}}$ gives the required bounds.   
\end{proof}

\subsection{Implications for Sparse Coding}
\label{subsec:instan-conc}
We now focus on applying the concentration bounds in Proposition~\ref{prop:concentration-degree-d} for the specific settings that arise in our context. 
Note that in this specific instance, the corresponding tensor has a specific form given by a sum of $m$ rank-1 components.  
\begin{lemma} \label{lem:conc-one-partition}
Let random variables $\zeta^{(1)}, \zeta^{(2)}, \dots, \zeta^{(d)}$  be i.i.d. draws from $\Zeta$ (sparsity $k=p m$). Consider a degree $d$ polynomial $f$ in $\upzeta{1}, \dots, \upzeta{d} \in \R^m$ given by a tensor $T=\sum_{i \in [m]} w_i M_i^{\otimes d}$ (with $w \in \R^m$) as follows 
$$f(\upzeta{1}, \dots, \upzeta{d}):= \sum_{(j_1,j_2 \dots, j_d) \in [m]^d} \sum_{i \in [m]} w_i \prod_{\ell=1}^d M_{i j_{\ell}} \upzeta{\ell}_{j_\ell}, 
$$
where $M_i$ is the $i$th column of $A^T A$ where $A$ is a matrix with spectral norm at most $\sigma$ and incoherence $\mu/\sqrt{n}$.   
Then, any $\eta \in (0,1)$ 
there exists constants $c_1=c_1(d) \ge 1$, 
we have with probability at least $1-\eta$ that
\begin{align}
\Abs{f(\upzeta{1}, \dots, \upzeta{d})} &< \nu(\eta,d) \left(\Big(\min\Bigset{1+\frac{\norm{w}_1^2}{\tau^2 k^2 \norm{w}_2^2}, \frac{1}{\eta}}\Big)^{1/2} \cdot  \norm{w}_2 \Big(\frac{\tau k}{m}\Big)^{d/2}+ \norm{w}_\infty \right) \nonumber \\
&\le \nu(\eta,d) \left(\sqrt{\min\set{1+\tfrac{m}{k^2}, 1/\eta}} \cdot  \norm{w}_2 \Big(\frac{\tau k}{m}\Big)^{d/2}+ \norm{w}_\infty  \right) \label{eq:conc-one-partition}
\end{align}
where $\nu(\eta,d)=c_1\log(1/\eta) \big(C (\sigma^2+\mu\sqrt{\tfrac{m}{n}}) \log(n/\eta)\big)^{d}$
captures polylogarithmic factors in $1/\eta$ and polynomial factors in the constants $C, \sigma, \mu, \beta=m/n$, and other constants $d$. 	
\end{lemma}
We remark that the multiplicative factor of $( \tau k/m)^{d/2}$ corresponds to the improvement (even in the specific case of $d=2$) over the bounds one would obtain using an application of Hanson-Wright and related inequalities. This is crucial in handling a sparsity of $k=\widetilde{\Omega}(m^{1/2})$ in the semirandom case, and $k=\widetilde{\Omega}(m^{2/3})$ in the random case. 
\begin{proof}
In the proof that follows we will focus on the case when $\tau=O(1)$, for sake of exposition (the corresponding bounds with $\tau$ are straightforward using the same approach). We are analyzing the sum
$$f(\upzeta{1}, \dots, \upzeta{d})= \sum_{i \in [m]} w_i \sum_{J=(j_1,j_2 \dots, j_d) \in [m]^d}  \prod_{\ell=1}^d M_{i j_{\ell}} \upzeta{\ell}_{j_\ell}$$

We will split this sum into many parts depending on which of the indices are fixed to be equal to $i$. For $\Gamma \subset [d]$, let 
\begin{align*}
f_\Gamma&= \sum_{i \in [m]} w_i \sum_{J_{\Gamma^c} \in ([m]\setminus \set{i})^{\Gamma^c}} \prod_{\ell' \in \Gamma} \upzeta{\ell'}_{j_{\ell'}} \cdot \prod_{\ell \in \Gamma^c} M_{i j_{\ell}} \upzeta{\ell}_{j_\ell} \\
&= \sum_{i \in [m]} w_i \prod_{\ell' \in \Gamma} \upzeta{\ell'}_{j_{\ell'}} \cdot \prod_{\ell \in \Gamma^c} \sum_{j_\ell \in [m] \setminus \set{i}}M_{i j_{\ell}} \upzeta{\ell}_{j_\ell}
\end{align*}

We have two main cases depending on whether $|\Gamma|>0$ or $\Gamma= \emptyset$. In the former case, we will apply bounds for degree $1$ (from Lemma~\ref{lem:bernstein:correct}) recursively to get the required bound. The latter case is more challenging and we will appeal to the concentration bounds we have derived in Proposition~\ref{prop:concentration-degree-d}. 

\paragraph{Case $|\Gamma|>0$.}
For each $\ell \in \Gamma^c$, consider the sum $H_\ell = \sum_{j_\ell \in [m]\setminus \set{i}} M_{ij_\ell} \upzeta{\ell}_{j_\ell}$. 
We have that $\E[Z_\ell]=0$, 
and $\norm{M_i}_2 \le \norm{M} \le \sigma^2$ (the spectral norm of $M$ upper bounds the length of any row or column), and the entries $\abs{M_{i j_\ell}} \le \mu/\sqrt{n} \le \norm{M_i}_2 \sqrt{k/m}$ since $k> \mu m/(n\sigma^2)$. 
Applying Lemma~\ref{lem:bernstein:correct}, we have
for an appropriate constant $c_d \ge 1$,
\begin{align}
\Pr\Big[ |H_\ell| > c_d \log(n/\eta) \sigma \cdot \sqrt{\tfrac{k}{m}} \Big] &\le \frac{\eta}{d 2^d n}. \nonumber\\
\text{Hence, } \forall i \in [m],~ \Abs{\prod_{\ell \in \Gamma^c} \sum_{j_\ell \in [m]\setminus \set{i}} M_{i j_\ell} \upzeta{\ell}_{j_\ell}} & \le c'_d \Big(\frac{k}{m} \Big)^{(d-|\Gamma|)/2}\big(C \sigma \log(n/\eta)\big)^{(d-|\Gamma|)}, \label{eq:recursivebernstein} 
\end{align}
with probability at least $(1-\eta 2^{-d})$. Let $Z_i = \prod_{\ell' \in \Gamma} \upzeta{\ell'}_{i}$ and $w'_i = w_i \prod_{\ell \in \Gamma} H_\ell$ and $Z_{tot}=\sum_i w'_i Z_i$. We know that $\E[Z_i]=0$ and $\Pr[Z_i \ne 0] \le (k/m)^{|\Gamma|}$. We also have for some constants $c_2,c'_d>0$
\begin{align*}
\norm{w'}_2^2 & \le c_2^2 \big(C \sigma \log(n/\eta)\big)^{2(d-|\Gamma|)} \norm{w}_2^2 \Big(\frac{k}{m}\Big)^{d-|\Gamma|}\\
\forall i \in [m], ~|w'_i Z_i| \le |C w'_i| &\le  c'_d \big(C \sigma \log(n/\eta)\big)^{d-|\Gamma|} \norm{w}_\infty \Big(\frac{k}{m} \Big)^{(d-|\Gamma|)/2} 
\end{align*}
due to \eqref{eq:recursivebernstein}. By Lemma~\ref{lem:bernstein:correct}, we have 
with probability at least $1-\eta 2^{-d}$, we have  
\begin{align}
\abs{f_\Gamma}=\abs{Z_{tot}} &\le c' \log(1/\eta) \big(C \sigma \log(n/\eta)\big)^{d} \Big( \norm{w}_2 \Big(\frac{k}{m}\Big)^{d/2} + \norm{w}_{\infty} \Big(\frac{k}{m} \Big)^{(d-|\Gamma|)/2}  \Big)
\nonumber\\
&\le c'\log(1/\eta) \big(C \sigma \log(n/\eta)\big)^{d} \Big(\norm{w}_2 \Big(\frac{k}{m}\Big)^{d/2}+ \norm{w}_\infty   \Big). \label{eq:one-term-bound1}
\end{align}

\paragraph{Case $\Gamma= \emptyset$.}
 
In this case, we collect together all the terms where $i \notin \set{j_1, j_2, \dots, j_d}$. Here, we will use the incoherence of $A$ and the spectral norm of $A$ to argue that all the flattenings have small norm. Each entry of the tensor 
$$\forall j_1, \dots, j_d \in [m],~ |T_{j_1, \dots, j_d}| = \sum_{i \in [m] \setminus \set{j_1, \dots, j_d}} w_i M_{i j_1} M_{ij_2} \dots M_{i j_d} \le \norm{w}_1 \big(\tfrac{\mu}{\sqrt{n}} \big)^{d} = \mu^{d} \cdot \norm{w}_1 n^{-d/2}.$$ 
In fact this also gives a bound of $\norm{T}_{[d],\infty}^2 \le \mu^{2d} \norm{w}_1^2 n^{-d}$. Further, since the frobenius norm remains the same under all flattenings into matrices, we have 
$$\norm{T}_F=\norm{ M^{\odot d-1} \diag(w) M^T}_F \le \norm{ M^{\odot d-1}}_{op} \norm{\diag(w)}_F \norm{M}_{op} \le \sigma^{2d} \norm{w}_2,$$
where the bounds on the operator norms follow from Lemma~\ref{lem:frob:fact2}. We will use the bound $B=\sigma^{2d} \norm{w}_2$ in Proposition~\ref{prop:concentration-degree-d}.

Consider any flattening $\Gamma_1 \subseteq [d]$ of the indices, and let $r_1=|\Gamma_1|$. For the matrix $\Tgam{\Gamma_1}$, we can give a simple bound based on the maximum entry of $T$ i.e., $\norm{T}_{\Gamma_1,\infty}^2 \le m^{d-|\Gamma_1|} \norm{T}_{[d],\infty}^2$. 
Hence, we have
\begin{align*}
\Big(\frac{\norm{T}_{\Gamma_1,\infty}^2}{m^{d-|\Gamma_1|}}\Big) \Big( \frac{B^2}{m^{d}}  \Big)^{-1} \cdot \frac{1}{(p m)^{|\Gamma_1|}} & \le 
\big( \mu^{2d} \norm{w}_1^2 n^{-d} \big) \big( \norm{w}_2^2 \sigma^{4d} m^{-d}  \big)^{-1} k^{-r_1} \\
&\le \Big(\frac{\mu^2 m}{\sigma^4 n}\Big)^{d} \frac{\norm{w}_1^2}{\norm{w}_2^2 k^{r_1}} \le \Big(\frac{\mu^2 m}{\sigma^4 n} \Big)^{d} \cdot  \frac{m}{k^{r_1}}.
\end{align*}
When $r_1 \ge 2$ this already gives a good bound of $\tilde{O}(m/k^2)$ which is sufficient for our purposes. However, this bound does not suffice when $|\Gamma_1|=r_1=1$; here we use a better bound by using the fact that the spectral norm of $M$ is bounded (in fact, this is where we get an advantage by considering flattenings). Since the length of any row is at most the spectral norm of the matrix we have
$$\norm{T}_{\Gamma_1, \infty}^2 \le \norm{\Tgam{\Gamma_1}}^2_{op} = \norm{M^{\odot \Gamma_1} \diag(w) (M^{\odot \Gamma_1^c})^T}^2_{op} \le \sigma^{4d} \norm{w}^2_\infty,$$
using the spectral norm bounds for Khatri-Rao products from Lemma~\ref{lem:frob:fact2}. 
Combining these, the factor due to flattening is
\begin{align*} 
\forall \Gamma_1 \text{ s.t. } |\Gamma_1|=1,~& ~ \frac{\norm{T}_{\Gamma',\infty}^2 p^{-|\Gamma'|}}{B^2} \le \frac{\norm{w}_\infty^2 m}{\norm{w}_2^2 k} \le 1/k. \\
\text{Hence, }& ~ \imbal= \sum_{\Gamma_1 \subset [d]}~ \frac{\norm{T}_{\Gamma_1,\infty}^2}{ B^2} \cdot p^{-|\Gamma_1|} = 1+ \sum_{\substack{\Gamma_1\\ |\Gamma_1|=1}} \frac{1}{k}+ \Big(\frac{\mu^2 m}{\sigma^4 n} \Big)^{d} \sum_{r_1=2}^d \sum_{\substack{\Gamma_1: \\ |\Gamma'|=r_1}}\frac{\norm{w}_1^2}{k^{r_1} \norm{w}_2^2}\\
&\le d \Big(\frac{\mu^2 m}{\sigma^4 n} \Big)^{d} \Big(1+\frac{\norm{w}_1^2}{k^2 \norm{w}_2^2} \Big) \le d \Big(\frac{\mu^2 m}{\sigma^4 n} \Big)^{d} \Big(1+\frac{m}{ k^2} \Big),
\end{align*}
since the number of subsets $\Gamma_1$ with $|\Gamma_1| \le d^{r_1}$ and the summation is dominated by $r_1=2, r_1=0$. Hence, using Proposition~\ref{prop:concentration-degree-d}, we get that with probability at least $1-\eta/2$, we have for some constant $c_3=c_3(d,\eps')$
\begin{equation}\label{eq:one-term-bound2}
|f_{[d]}| \le c_3 \min\Set{\sqrt{1+\frac{\norm{w}_1^2}{k^2 \norm{w}_2^2}}, \frac{1}{\sqrt{\eta}}} \Big(C \mu  \log(2/\eta) \sqrt{m/n}\Big)^{d} \norm{w}_2 \Big(\frac{k}{m} \Big)^{d/2} .
\end{equation}

Combining the bounds \eqref{eq:one-term-bound1} and \eqref{eq:one-term-bound2} we have with probability at least $(1-\eta)$,
$$\Abs{f(\upzeta{1}, \dots, \upzeta{d})} < \nu(\eta,d) \left(\Big(\min\Bigset{1+\frac{\norm{w}_1^2}{k^2 \norm{w}_2^2}, \frac{1}{\eta}}\Big)^{1/2} \cdot  \norm{w}_2 \Big(\frac{k}{m}\Big)^{d/2}+ \norm{w}_\infty  \right),$$
where $\nu(\eta,d)=c'\log(1/\eta) \big(C (\sigma^2+\mu\sqrt{\tfrac{m}{n}}) \log(n/\eta)\big)^{d}$.


\end{proof}

The following lemma corresponds to a term where the tensor is of rank $1$ of the form $w_i M_i^{\otimes d}$. 
\begin{lemma} \label{lem:conc-special-partition}
Let random variables $\zeta^{(1)}, \zeta^{(2)}, \dots, \zeta^{(d)}$  be i.i.d. draws from $\Zeta$ (sparsity $k=p m$). Consider a degree $d$ polynomial $f$ in $\upzeta{1}, \dots, \upzeta{d} \in \R^m$ given by a tensor $T=w_i M_i^{\otimes d}$ as follows 
$$f(\upzeta{1}, \dots, \upzeta{d}):= w_i \sum_{(j_1,j_2 \dots, j_d) \in [m]^d} \prod_{\ell=1}^d M_{i j_{\ell}} \upzeta{\ell}_{j_\ell}, 
$$
where $M_i$ is the $i$th column of $A^T A$ where $A$ is a matrix with spectral norm at most $\sigma$ and incoherence $\mu/\sqrt{n}$, and $w_i \in \R$.   
There exists constant $c_1=c_1(d) \ge 1$ such that for any $\eta>0$, 
we have with probability at least $1-\eta$ that
\begin{align}
\Abs{f(\upzeta{1}, \dots, \upzeta{d})} &< \abs{w_i} \prod_{\ell \in [d]} \abs{\upzeta{\ell}_i} + c_1 \abs{w_i} \nu(\eta,d) \sqrt{\frac{\tau k}{m}}  \label{eq:concentration-special-block-gen},
\end{align}
where $\nu(\eta,d):=\log(1/\eta) \big(C (\sigma^2+\mu\sqrt{\tfrac{m}{n}}) \log(n/\eta)\big)^{d}$ capture the polylogarithmic factors in $\log(1/\eta)$ and polynomial factors in $C,\mu$ and $m/n$.
Furthermore, for $d \ge 3$ and for any $\eps>0$ if $k \le m^{2/3}/(\tau \log m)$, then we have that with probability at least $1-1/(m\log m)$,  
\begin{equation}
\Abs{f(\upzeta{1}, \dots, \upzeta{d})} < c_1\log(m) \big(C (\sigma^2+\mu\sqrt{\tfrac{m}{n}}) \log(nm)\big)^{d} \cdot \abs{w_i} \Big(\frac{\tau k}{m}\Big)^{(d-2)/2}.  \label{eq:concentration-special-block}
\end{equation}
\end{lemma}

\begin{proof}
We will follow the same proof strategy as in Lemma~\ref{lem:conc-one-partition} by splitting the sum into many parts depending on which of the indices are fixed to be equal to $i$. As in Lemma~\ref{lem:conc-one-partition}, we focus on the case when $\tau=1$, for sake of exposition. For $\Gamma \subset [d]$, let 
\begin{align*}
f_\Gamma&= w_i \sum_{J_{\Gamma^c} \in ([m]\setminus \set{i})^{\Gamma^c}} \prod_{\ell' \in \Gamma} \upzeta{\ell'}_{j_{\ell'}} \cdot \prod_{\ell \in \Gamma^c} M_{i j_{\ell}} \upzeta{\ell}_{j_\ell} \\
&= w_i \prod_{\ell' \in \Gamma} \upzeta{\ell'}_{j_{\ell'}} \cdot \prod_{\ell \in \Gamma^c} \sum_{j_\ell \in [m] \setminus \set{i}}M_{i j_{\ell}} \upzeta{\ell}_{j_\ell}
\end{align*}

We have three cases depending on whether $|\Gamma|=[d]$ or $\Gamma= \emptyset$ or otherwise. 

\paragraph{Case $0<|\Gamma|<d$.}
For each $\ell \in \Gamma^c$, consider the sum $H_\ell = \sum_{j_\ell \in [m]\setminus \set{i}} M_{ij_\ell} \upzeta{\ell}_{j_\ell}$. Recall that $\norm{M_i}_2 \le \norm{M} \le \sigma^2$. We have that $\E[Z_\ell]=0$,
and the entries $|M_{i j_\ell}| \le \mu/\sqrt{n}$. 
As we had in the case $|\Gamma|>0$ in Lemma~\ref{lem:conc-one-partition}, we have
for an appropriate constant $c_d \ge 1$,
\begin{align}
\Pr\Big[ |H_\ell| > c_d \log(1/\eta) \sigma \cdot \sqrt{\tfrac{k}{m}} \Big] &
\le \frac{\eta}{d 2^d }. \nonumber\\
\text{Hence, } \Abs{\prod_{\ell \in \Gamma^c} \sum_{j_\ell \in [m]\setminus \set{i}} M_{i j_\ell} \upzeta{\ell}_{j_\ell}} & \le c'_d \Big(\frac{k}{m} \Big)^{(d-|\Gamma|)/2}\big(C \sigma \log(n/\eta)\big)^{(d-|\Gamma|)}, \label{eq:special-block-recursivebernstein} 
\end{align}
with probability at least $(1-\eta 2^{-(d+1)})$. 
Summing over all $\Gamma \ne [d], \emptyset$, we have with probability at least $(1-\eta/2)$
\begin{equation}
\sum_{\Gamma: 0<|\Gamma|<d} f_\Gamma \le c_1 w_i\log(1/\eta) \big(C (\sigma^2+\mu\sqrt{\tfrac{m}{n}}) \log(n/\eta)\big)^{d} \sqrt{\frac{k}{m}}  \label{eq:special-block-onebound},
\end{equation}

\paragraph{Case $\Gamma= \emptyset$.}
 
In this case, we collect together all the terms where $i \notin \set{j_1, j_2, \dots, j_d}$. Here, we will use the incoherence of $A$ and the spectral norm of $A$ to argue that all the flattenings have small norm. Each entry of the tensor 
$$\forall j_1, \dots, j_d \in [m],~ |T_{j_1, \dots, j_d}| = w_i M_{i j_1} M_{ij_2} \dots M_{i j_d} \le w_i \big(\tfrac{\mu}{\sqrt{n}} \big)^{d} = \mu^{d} \cdot w_i n^{-d/2}.$$ 
Hence, $\norm{T}_{[d],\infty}^2 \le w_i^2 \mu^{2d} n^{-d}$. Further, suppose we denote by $w \in \R^m$, the vector with all entries except the $i$th being $0$ and $i$th co-ordinate being $w_i$,  
$$\norm{T}_F=\norm{ M^{\odot d-1} \diag(w) M^T}_F \le \norm{ M^{\odot d-1}}_{op} \norm{\diag(w)}_F \norm{M}_{op} \le w_i \sigma^{2d},$$
where the bounds on the operator norms follow from Lemma~\ref{lem:frob:fact2}. As before, we will use the bound $B=w_i \sigma^{2d}$ in Proposition~\ref{prop:concentration-degree-d}. 

Consider any flattening $\Gamma_1 \subseteq [d]$ of the indices, and let $r_1=|\Gamma_1|$. For the matrix $\Tgam{\Gamma_1}$, we can give a simple bound based on the maximum entry of $T$ i.e., $\norm{T}_{\Gamma_1,\infty}^2 \le m^{d-|\Gamma_1|} \norm{T}_{[d],\infty}^2$. 
Hence, we have
\begin{align*}
\Big(\frac{\norm{T}_{\Gamma_1,\infty}^2}{m^{d-|\Gamma_1|}}\Big) \Big( \frac{B^2}{m^{d}}  \Big)^{-1} \cdot \frac{1}{(p m)^{|\Gamma_1|}} & \le 
\big( \mu^{2d} w_i^2 n^{-d} \big) \big( w_i^2 \sigma^{4d} m^{-d}  \big)^{-1} k^{-r_1} \\
&\le \Big(\frac{\mu^2 m}{\sigma^4 n}\Big)^{d}  \cdot  \frac{1}{k^{r_1}}.
\end{align*}
Since the number of subsets $\Gamma_1$ with $|\Gamma_1| \le d^{r_1}$ and the summation is dominated by $r_1=0$. Hence, using Proposition~\ref{prop:concentration-degree-d}, we get that with probability at least $1-\eta/2$, we have for some constant $c_3=c_3(d,\eps')$
\begin{equation}\label{eq:special-block-twobound}
|f_{[d]}| \le c_3 \Big(C \mu  \log(2/\eta) \sqrt{m/n}\Big)^{d} \cdot w_i \Big(\frac{k}{m} \Big)^{d/2} .
\end{equation}

Finally, where $\Gamma=[d]$, we have that $f_\Gamma= w_i (M_{ii})^d \prod_{\ell \in \Gamma} \upzeta{\ell}_i$. Hence, combining the bounds \eqref{eq:special-block-onebound} and \eqref{eq:special-block-twobound} we get with probability at least $1-\eta$ that \eqref{eq:concentration-special-block-gen} holds.

For $d=1$, we have as in \eqref{eq:special-block-recursivebernstein} that 
\begin{align*}
f(\upzeta{1})& = w_i \sum_{j \in [m]} M_{ij} \upzeta{1}_{j} = w_i M_{ii} \upzeta{1}_i+ w_i \sum_{j \ne i} M_{ij} \upzeta{1}_{j} \\
\Abs{f(\upzeta{1})- w_i \upzeta{1}_i} &\le c' w_i \log(1/\eta) C \sigma \sqrt{k/m}  
\end{align*}
for some constant $c'$ from Lemma~\ref{lem:bernstein:correct}.

For the furthermore part, we observe that for $d\ge 3$ and $k=m^{2/3}/\log^2 m$, $(k/m)^3\le 1/(m \log^3 m) \le (2^d m \log m )^{-1}$. Hence, all the terms with $|\Gamma| \ge 3$ term are $0$ with probability at least $1-1/(m\log m)$ (i.e., $\eta=1/(m\log m)$). Hence, with probability at least $1-1/(m\log m)$, we have from \eqref{eq:special-block-recursivebernstein} and \eqref{eq:special-block-twobound}
\begin{align*}
\sum_{\Gamma \subseteq [d]} f_\Gamma &= \sum_{\Gamma: |\Gamma|\le 2} f_\Gamma \le c_1 w_i\log(m) \big(C (\sigma^2+\mu\sqrt{\tfrac{m}{n}}) \log(nm)\big)^{d} \Big(\frac{k}{m}\Big)^{(d-2)/2}.
\end{align*}

\end{proof}


\begin{lemma}\label{lem:frobnorm:bound}
Consider the multivariate polynomial in random variables $\upzeta{1}, \upzeta{2}, \dots, \upzeta{2L}$
\begin{align*}
f_i\bigparen{\upzeta{1}, \dots, \upzeta{2L-1}}&= \sum_{J \in [m]^{2L-1}} T^{(i)}_{j_1, \dots, j_{2L}} \prod_{t \in [2L-1]} \upzeta{t}_{j_t}, \nonumber\\
\text{ where } T^{(i)}&=\sum_{i_1,\dots, i_{2L-1} \in [m]}\E_{x} \big[x_{i_1} x_{i_2} \dots x_{i_{2L-1}} x_i\big] M_{i_1} \otimes M_{i_2} \otimes \dots \otimes M_{i}. 
\end{align*}
Then we have that $\forall i \in [m], ~\norm{T^{(i)}}_F \le (4CL)^{2L}\sigma^{4L} \cdot k^{(L-1)/2} $. Further, for any $\eta>0$
$$\Pr_{\zeta} \Big[ \bigabs{f_i\bigparen{\upzeta{1}, \dots, \upzeta{2L-1}}} \ge q_i \cdot \nu(\eta,2L) \cdot \frac{1}{k^{1/4}\sqrt{\eta}} \Paren{\frac{(\tau k)^{3/2}}{m}}^{(2L-1)/2}  \Big] \le \eta,$$
where $\nu(\eta,d):=(2C)^{d} (C^2\log (2/\eta))^{d/2} \sigma^{2d}$.

Moreover, in the ``random'' case when the non-zero $x_i$s satisfy \eqref{eq:support-assumption-1} we get that $\forall i \in [m], ~\norm{T^{(i)}}_F \le (4CL)^{2L}\sigma^{4L} \cdot (\frac{k \tau}{\sqrt{m}})^{L-1}$. In this case we have that for any $\eta>0$
\begin{equation}\label{eq:frobbound:random}
\Pr_{\zeta} \Big[ \bigabs{f_i\bigparen{\upzeta{1}, \dots, \upzeta{2L-1}}} \ge q_i \cdot \nu(\eta,2L) \cdot \frac{1}{(k\tau)^{1/6}\sqrt{\eta}} \Paren{\frac{(\tau k)^{4/3}}{m}}^{(3L-2)/2}  \Big] \le \eta.
\end{equation}
\end{lemma}

Since $\E[x_{i'}]=0$ for each $i' \in [m]$ we only get a contribution from terms such that the indices $i_1, i_2, \dots, i_{2L-1}, i_{2L}=i$ are paired up an even number of times. Let $\calS=(S_0,S_1, \dots, S_R)$ be a partition of indices $\set{1, 2, \dots, 2L-1 } \cup \set{2L}$ such that each of the sets $|S_0|,|S_1|, \dots, |S_R|$ are even and the index $i$ i.e., $2L \in S_0$. Hence, $R \le L-1$. All the indices in a set $S_r$ will take the same value $i^*_r$ i.e., for each $r \in [R]$, there exists $i^*_r \in [m]$ such that all indices $\ell \in S_r$ satisfy $i_{\ell}=i^*_r$ (for $r=0$, this will also be equal to $i$). Finally, for a fixed partition $\calS=(S_0, S_1, \dots, S_R)$, let $s_r=|S_r|$ for each $r \in [R]$. We now upper bound the Frobenius norm given by each partition $\calS$. 

\begin{lemma}\label{lem:semirandom:frobbound}
In the above notation, for any partition $\calS=(S_0,S_1, S_2, \dots, S_R)$ such that $|S_1|, |S_2|, \dots, |S_R|$ are even, we have
\begin{equation}\label{eq:semirandom:frobbound}
\norm{T_{\calS}}_F \le q_i \cdot C^{2L}  \sigma^{4L} k^{R/2} \le q_i C^{2L} \sigma^{4L}  k^{(L-1)/2}, 
\end{equation}
where $\sigma$ is the maximum singular value of $A$.
\end{lemma}
\begin{proof}
Let $i^*_r \in [m]$ denote the common index for all the the indices in $S_r$ i.e.,  $\forall \ell \in S_r$ satisfy $i_{\ell}=i^*_r$.
Without loss of generality we can assume that $\calS=(S_0, S_1, \dots,S_R)$ where $S_0=\set{2L, 1,2,\dots, s_0-1}, S_1=\set{s_0, \dots, s_0+s_1-1}, \dots, S_R=\set{2L-s_R, \dots, 2L-1}$. Then 
\begin{align}
T_\calS&= M_i^{\otimes s_0} \sum_{i^*_1=1}^m \sum_{i^*_2=1}^m \dots \sum_{i^*_R=1}^m q_{i,i^*_1,\dots, i^*_R}(s_0, \dots,s_R) M_{i^*_1}^{\otimes s_1} \otimes M_{i^*_2}^{\otimes s_2} \otimes \dots \otimes M_{i^*_R}^{\otimes s_R} \nonumber\\
\norm{T_\calS}_F &\le C^{2L} \Bignorm{\sum_{i^*_1=1}^m \sum_{i^*_2=1}^m \dots \sum_{i^*_R=1}^m q_{i,i^*_1, \dots, i^*_R} M_{i^*_1}^{\otimes s_1} \otimes M_{i^*_2}^{\otimes s_2} \otimes \dots \otimes M_{i^*_R}^{\otimes s_R}}_F,  \label{eq:frobbound:inter}
\end{align}
since each group $S_r$ in the partition is of even size, and from Lemma~\ref{lem:frob:fact1}.  

We will now prove the following statement inductively on $r \in [R]$ (or rather $R-r$). 
\anote{What does it mean for a tensor to be PSD?? Clarify...} 

\begin{claim}\label{claim:frob:inductive}
For any fixed prefix of indices $i,i^*_1,\dots, i^*_r \in [m]$, suppose we denote by 
$$ B_{i,i^*_1,\dots, i^*_r} = \sum_{i^*_{r+1}, \dots, i^*_{R} \in [m]} \Big(\frac{q_{i,i^*_1, \dots,i^*_r,\dots, i^*_R}}{q_{i,i^*_1, \dots, i^*_r}}\Big) M_{i^*_{r+1}}^{\otimes s_{r+1}} \otimes M_{i^*_{r+2}}^{\otimes s_{r+2}} \otimes \dots \otimes M_{i^*_{R}}^{\otimes s_{R}}, $$
then $B_{i,i^*_1,\dots, i^*_r}$ is PSD and $\norm{B_{i,i^*_1,\dots, i^*_r}}_F \le (\sqrt{k})^{R-r} \sigma^{2(s_{r+1}+\dots+s_R)}$. 
\end{claim}
We now prove this claim by induction on $(R-r)$. Assume it is true for $B_{i,i^*_1, \dots, i^*_r}$, we will now prove it for $B_{i,i^*_1, \dots, i^*_{r-1}}$.
For convenience let $w \in \R^m$ with $w_{i^*_r}=q_{i,i^*_1, \dots,i^*_{r-1}, i^*_r}/ q_{i,i^*_1, \dots,i^*_{r-1}}$ for each $i^*_r \in [m]$. Then
\begin{align*}
B_{i,i^*_1,i^*_2, \dots, i^*_{r-1}}&= \sum_{i^*_r=1}^m w_{i^*_r} M_{i^*_r}^{\otimes s_r} \otimes B_{i,i^*_1,\dots,i^*_{r-1},i^*_r}\\
\Bignorm{B_{i,i^*_1,i^*_2, \dots, i^*_{r-1}}}_F & \le \Bignorm{\sum_{i^*_r=1}^m w_{i^*_r} \norm{B_{i,i^*_1,\dots,i^*_{r-1},i^*_r}}_F M_{i^*_r}^{\otimes s_r} }_F \le (\sqrt{k})^{R-r} \Bignorm{\sum_{i^*_r=1}^m w_{i^*_r} M_{i^*_r}^{\otimes s_r} }_F,
\end{align*}
where the second line follows from Lemma~\ref{lem:frob:fact1} and the last line uses the induction hypothesis. Now note that $s_r$ is even -- so $B_{i^*_1,\dots, i^*_{r-1}}$ is PSD, for an appropriate flattening into a matrix of dimension $n^{(s_{r-1}+\dots+s_R)/2 \times (s_{r-1}+\dots+s_R)/2}$. Further, by flattening into the corresponding symmetric matrix of dimension $n^{(s_{r-1}+\dots+s_R)/2 \times (s_{r-1}+\dots+s_R)/2)}$, we see that  
\begin{align*}
B_{i,i^*_1,i^*_2, \dots, i^*_{r-1}}&\le  (\sqrt{k})^{R-r} \sigma^{s_{r+1}+\dots+s_R}\Bignorm{\sum_{i^*_r=1}^m  w_{i^*_r} M_{i^*_r}^{\otimes s_r/2} (M_{i^*_r}^{\otimes s_r/2})^T}_F\\
&= (\sqrt{k})^{R-r} \sigma^{2(s_{r+1}+\dots+s_R)} \Bignorm{M_{i^*_r}^{\otimes s_r/2} \diag(w) (M_{i^*_r}^{\otimes s_r/2})^T}_F \\
&\le (\sqrt{k})^{R-r} \sigma^{2(s_{r+1}+\dots+s_R)} \norm{w}_2 \sigma^{2s_r} \\
&\le (\sqrt{k})^{R-r+1} \sigma^{2(s_r+s_{r+1}+\dots+s_R)},
\end{align*}
where the second inequality followed from Lemma~\ref{lem:frob:fact2} and last inequality used the fact that $\sum_i w_i \le k$. Further, the base case ($r=R-1$) also follows from Lemma~\ref{lem:frob:fact2} in an identical manner. This establishes the claim.

To conclude the lemma, we use the claim with $r=0$ and observe that from \eqref{eq:frobbound:inter} that 
\begin{align*}
\norm{T_\calS}_F &\le C^{2L} \Bignorm{ q_i M_i^{\otimes s_0} B_{i}     }_F \le q_i \norm{M_i}^{s_0} \sqrt{k}^{R} \sigma^{s_1+ \dots+ s_R} \le q_i k^{R/2} \sigma^{2L}. 
\end{align*}

\end{proof}

\begin{proof}[Proof of Lemma~\ref{lem:frobnorm:bound}]
We first upper bound $\norm{T}_F$. As described earlier, $T$ can be written (after reordering the modes of the tensor) as a sum of corresponding tensors $T_\calS$ over all valid partitions $\calS=(S_0, S_1, \dots, S_R)$ as 
\begin{align*}
T&=\sum_{\calS} T_{\calS}=\sum_{\calS} M_i^{\otimes s_0}\sum_{i^*_1=1}^m \sum_{i^*_2=1}^m \dots \sum_{i^*_R=1}^m \E_x\big[ x_i^{s_0} x_{i^*_1}^{s_1} \dots x_{i^*_{R}}^{s_R}\big] \bigotimes_{\ell=1}^{2L} M_{i^*_\ell}\\
&=\sum_{\calS} T_{\calS}=\sum_{\calS}  M_i^{\otimes s_0}\sum_{i^*_1=1}^m \sum_{i^*_2=1}^m \dots \sum_{i^*_R=1}^m q_{i,i^*_1, \dots, i^*_R}(s_0,s_1,\dots,s_R) \bigotimes_{\ell=1}^{2L} M_{i^*_\ell}
\end{align*}
There are at most $(4L)^{2L}$ such partitions, and $\norm{T_{\calS}}_F$ is upper bounded by Lemma~\ref{lem:semirandom:frobbound}. Hence, by triangle inequality, $\norm{T}_F \le (4LC)^{2L} \sigma^{4L} k^{(L-1)/2}$. Finally, using Proposition~\ref{prop:concentration-degree-d} (choosing the $1/\sqrt{\eta}$ option of the two bounds), and reorganized the terms, the concentration bound follows.  

Finally, the bound for the random case in \eqref{eq:frobbound:random} is obtained by the same argument in Lemma~\ref{lem:semirandom:frobbound} and using the fact that $q_{S,i} \le q_S \cdot \tau k/m$ for all $i,S$ s.t. $i \notin S$. 
\end{proof}



\anote{4/22: Removed a few paragraphs about partitions etc. that seemed redundant.}

In what follows for $J=\paren{j_1, \dots, j_d}$ and $H \subset [d]$, we will denote by $J_H=\paren{j_\ell: \ell \in H}$ to the subset of indices restricted to $H$.

\begin{lemma}\label{lem:semi-random-inductive}
Let $L$ be a constant and let $S = (S_{1}, S_{2}, \dots, S_R)$ be a fixed partition of $[2L-1]$ with $|S_r|\ge 2$ for all $r \in \set{2,3,\dots, R}$. Furthermore, let $H_1, H_2, \dots H_{R}$ be such that $H_r \subseteq S_r$ for each $r \in [R]$. For any fixed prefix of indices $i^*_1,\dots, i^*_r \in [m]$, consider the random sum 
$$ F_{i^*_1,i^*_2, \dots, i^*_r} = \sum_{i^*_{r+1}, \dots, i^*_{R} \in [m]} \Big(\frac{q_{i^*_1, \dots, i^*_R}(d_1,\dots, d_R)}{q_{i^*_1, \dots, i^*_r}(d_1,\dots,d_r)}\Big) \prod_{p=r+1}^R \sum_{\substack{J_{S_p \setminus H_p} \in \\ [m]^{S_p \setminus H_p}}} M^{|H_p|}_{i^*_p,1} \prod_{t \in H_p} \upzeta{t}_1 \prod_{t \in S_p \setminus H_p}  M_{i^*_{p},j_t} \upzeta{t}_{j_t}
$$
then with probability at least $1-\eta$ (over the randomness in $\zeta$s), for every $r \ge 1$ we have that 
 \begin{equation}\label{eq:semi-random-inductive}
 \Abs{F_{i^*_1,i^*_2, \dots, i^*_r}} \leq \nu(\eta,d_r):= c_1\log(1/\eta) \big(2 C^2 (\sigma^2+\mu\sqrt{\tfrac{m}{n}}) \log(n/\eta)\big)^{d_r}
 \end{equation}
 where $d_r = \sum_{i=r+1}^R |S_i|$, and 
$\nu(\eta,d)$
captures poly-logarithmic factors in $1/\eta$ and polynomial factors in the constants $C, \sigma, \mu, \beta=m/n$, and other constants $d$. 	
\end{lemma}
\begin{proof}
We will prove this through induction on $(R-r)$. Assume it is true for $F_{i^*_1, \dots, i^*_r}$, we will now prove it for $F_{i^*_1, \dots, i^*_{r-1}}$. Let $q'_{i^*_1,\dots, i^*_r}=q_{i^*_1,\dots, i^*_r}(d_1, \dots, d_r)$. 
For convenience let $w \in \R^m$ with $w_{i^*_r}=q'_{i^*_1, \dots,i^*_{r-1}, i^*_r}/ q'_{i^*_1, \dots,i^*_{r-1}}$ for each $i^*_r \in [m]$. Then
\begin{align}
F_{i^*_1,i^*_2, \dots, i^*_{r-1}}&= \sum_{i^*_r, i^*_{r+1}, \dots, i^*_{R} \in [m]} 
\Big(\frac{q'_{i^*_1, \dots,i^*_{r-1},i^*_r,\dots, i^*_R}}{q'_{i^*_1, \dots, i^*_{r-1}}} \Big) 
\prod_{p=r}^R \Paren{\sum_{\substack{J_{S_p \setminus H_p} \\
\in [m]^{S_p \setminus H_p}}} \prod_{t \in H_p} \upzeta{t}_1 \prod_{t \in S_p \setminus H_p} \zeta^{(t)}_{j_t} M_{i^*_{p},j_t} M^{|H_p|}_{i^*_p,1} } \nonumber
\\
&= \sum_{i^*_r \in [m]} w_{i^*_r} \Big( \prod_{t \in S_r \setminus H_r} \sum_{j_t \in [m]} \zeta^{(t)}_{j_t} M_{i^*_{r},j_t} M^{|H_r|}_{i^*_r,1} \Big) \cdot \prod_{t \in H_r} \upzeta{t}_1 \cdot F_{i^*_1, \dots, i^*_r} \label{eq:semirandom-inductive:1}.
\end{align}
To bound the sum over $i^*_r$ i.e., the contribution from block $S_r$, we will use Lemma~\ref{lem:conc-one-partition}. Let $w' \in \R^m$ be defined by $w'_{i^*_r} = w_{i^*_r} F_{i^*_i,\dots,i^*_r} \prod_{t \in H_r} \upzeta{t}_1$.  
Note $\norm{w'}_\infty \le C^{|H_r|+|S_r|} \abs{F_{i^*_i,\dots,i^*_r}}$ and 
\begin{align*}
\norm{w'}_2^2 &= \prod_{t \in H_r}\abs{\upzeta{t}_1}^2 \cdot \abs{F_{i^*_i,\dots,i^*_r}}^2 \sum_{i^*_r \in [m]} w_{i^*_r}^2 \le C^{2|H_r|+|S_r|} \abs{F_{i^*_i,\dots,i^*_r}}^2 \sum_{i^*_r \in [m]} w_{i^*_r} \\
&\le \abs{F_{i^*_i,\dots,i^*_r}}^2\sum_{i^*_r\in [m]} \frac{q'_{i^*_1, \dots, i^*_{r-1},i^*_r}}{q'_{i^*_1, \dots, i^*_{r-1}}} \le k C^{2|H_r|+2|S_r|} \abs{F_{i^*_i,\dots,i^*_r}}^2,
\end{align*}
since each sample has at most $k$ non-zero entries. 
$$
F_{i^*_1,i^*_2, \dots, i^*_{r-1}} 
=\sum_{i^*_r \in [m]} w'_{i^*_r} M^{|H_r|}_{i^*_r,1} \prod_{t \in H_p} \upzeta{t}_1 \prod_{t \in S_r \setminus H_r} \sum_{j_t \in [m]} \upzeta{t}_{j_t} M_{i^*_{r},j_t}  .$$
We have two cases depending on whether $H_r=S_r$ or not. If $H_r=S_r$, then 
\begin{align*}
F_{i^*_1,i^*_2, \dots, i^*_{r-1}} &= \sum_{i^*_r \in [m]} w'_{i^*_r} M_{i^*_r,1}^{|S_r|} \le \sum_{i^*_r \in [m]} \abs{w'_{i^*_r}} \iprod{A_{i^*_r},A_1}^{2}\\
&\le \norm{A \diag(w'') A^T}_{op} \le \sigma^2 \abs{F_{i^*_i,\dots,i^*_r}}\cdot  \norm{w'}_\infty, 
\end{align*}
where $w''$ in the intermediate step is the vector with $w''_i=\abs{w'_i}$.
Otherwise, $H_r \ne S_r$. Applying Lemma~\ref{lem:conc-one-partition} (here $d \ge 1$), we have
\begin{align*}
\Abs{F_{i^*_1,i^*_2, \dots, i^*_{r-1}}} &\le \nu(\eta, |S_r|) \Big(\frac{\sqrt{m}}{k} \cdot \norm{w'}_2 \cdot \frac{\tau k}{m} + \norm{w'}_\infty\Big) \\
&\le \nu(\eta, |S_r|)  \Big( \sqrt{\frac{\tau^2 k}{m}} +1 \Big) \abs{F_{i^*_i,\dots,i^*_r}}\cdot C^{|S_r|}\le 2 \nu_1(\eta, |S_r|) \abs{F_{i^*_i,\dots,i^*_r}}\cdot C^{|S_r|},
\end{align*}
where $\nu_1(\eta,|S_r|)$ captures the $\widetilde{O}(1)$ terms in \eqref{eq:conc-one-partition}. By using induction hypothesis and \eqref{eq:semirandom-inductive:1},
\begin{align*}
\Abs{F_{i^*_1,i^*_2, \dots, i^*_{r-1}}} 
&\le \nu(\eta,d_r) \cdot 2 \nu_1(\eta,|S_r|) \cdot C^{|S_r|} \le \nu(\eta, d_{r-1}), 
\end{align*}
since $d_{r-1}=d_r+|S_r|$ and from our choice of $\nu(\eta,d)$. 
Further, the base case ($r=R-1$) also follows from Lemma~\ref{lem:conc-one-partition} in an identical manner. This establishes the claim and hence the lemma.
\end{proof}

The following simple lemma follows from the bound on the spectral norm, and is useful in the analysis.
\begin{lemma}
\label{lem:semi-random-F-full}
Consider fixed indices $i,j \in [m]$ and let 
$$
T_{i} = \sum_{i^*_2, \dots i^*_r \in [m]^r} q_{i,i^*_2, \dots , i^*_r}(d_1+1,\dots,d_r) (M_{i,j})^{d_1} \cdot (M_{i^*_1,j})^{d_2} \dots (M_{i^*_r,j})^{d_r},
$$
where $d_2, \dots d_r \geq 2$. Then for some constant $c'>0$ that $|T_i| \leq c' q_i \cdot \abs{M_{i,j}}^{d_1} \sigma^{2(r-1)} \cdot C^{d_1+\dots+d_r+1}$.
\end{lemma}
\begin{proof}
For convenience, let $q'_{i^*_1, \dots, i^{*}_\ell}=q_{i^*_1, \dots, i^{*}_\ell}(d_1+1, \dots, d_\ell)$ for each $\ell \in [r]$. 
We will prove this by induction on $r$, by establishing the following claim for every $\ell \in \set{2,3,\dots,r}$:
\begin{equation}\label{eq:semirandom:inductive:simple}
\Abs{\sum_{i^*_\ell, \dots, i^*_r \in [m]} \Paren{\frac{q'_{i, i^*_2, \dots, i^*_r}}{q_{i',i^*_2, \dots, i^*_{\ell-1}}}} M_{i,j}^{d_1} \cdot \prod_{t=\ell}^{r} M_{i^*_t,j}^{d_t}} \le  \sigma^{2(r-\ell+1)} \cdot C^{d_\ell+\dots+d_r}.
\end{equation}
To see this, set $w_{i^*_\ell}=q'_{i,i^*_2, \dots, i^*_\ell}/q'_{i,i^*_2, \dots, i^*_{\ell-1}}$ and observe that
\begin{align*}
&\Abs{\sum_{i^*_\ell, \dots, i^*_r \in [m]} \Paren{\frac{q'_{i, i^*_2, \dots, i^*_r}}{q'_{i,i^*_2, \dots, i^*_{\ell-1}}}}  M_{i,j}^{d_1} \cdot \prod_{t=\ell}^{r} M_{i^*_t,j}^{d_t}}=\Abs{\sum_{i^*_\ell \in [m]} w_{i^*_\ell} M_{i^*_\ell,j}^{d_\ell} \sum_{i^*_{\ell+1}, \dots, i^*_r \in [m]} \Paren{\frac{q_{i, i^*_2, \dots, i^*_r}}{q_{i,i^*_2, \dots, i^*_{\ell}}}}  \prod_{t=\ell}^{r} M_{i^*_t,j}^{d_t}}\\
&~\qquad~\le \Abs{\sum_{i^*_\ell \in [m]} w_{i^*_\ell} M_{i^*_\ell,j}^{d_\ell}} \cdot \sigma^{2(r-\ell)} \cdot C^{d_{\ell+1}+\dots+d_r}\le  \sigma^{2(r-\ell)} \cdot C^{d_{\ell+1}+\dots+d_r} \sum_{i^*_\ell \in [m]} w_{i^*_\ell} M_{i^*_\ell,j}^{2} \\
&~\qquad~\le \sigma^{2(r-\ell)} C^{d_{\ell+1}+\dots+d_r} \cdot \norm{A \diag(w) A^T}_{op} \le \sigma^{2(r-\ell+1)} \cdot C^{d_\ell+\dots+d_r},
\end{align*}
where the second line follows from the inductive hypothesis. The base case is when $\ell=r$ and follows an identical argument involving the spectral norm. Hence, the lemma follows by applying the claim to $\ell=2$. 
\end{proof}

\subsection{Proof of Theorem~\ref{thm:single-recovery-main}}
\label{sec:semi-random-recovery}
\anote{4/22: Still confused with this old comment: There's a subtlety about whether the adversarial samples (and hence tensor $T$ via the $q$ values) can depend on the random samples or not. I think one can show similar bounds in the harder model too by trying many random samples for $(\upzeta{1}, \dots, \upzeta{2L})$...}

In this section we show how to use data generated from a semi-random model $\sModel$ and recover columns of $A$ that appear most frequently. The recovery algorithm is sketched in Figure~\ref{ALG:Single_Recovery}.
 \begin{figure}[htb]
 \begin{center}
 \fbox{\parbox{1\textwidth}{
{\bf Algorithm \textsc{RecoverColumns}}$(\sModel, T_1, L, \epsilon)$
\begin{enumerate}
\item Initialize $W = \emptyset$. Set $\eta'_0 = \exp(-m^{O(L)} \log(1/\eps))$.
\item Draw  set $T_0$ of samples from $\sModel$ where $|T_0| \geq 4(2L-1)m \log (m/\eta'_0)/(\beta k)$.
\item For each $2L-1$ tuple $(u^{(1)}, u^{(2)}, \dots, u^{(2L-1)})$ in $T_0$, let
\begin{align}
v = \frac{1}{|T_1|} \sum_{y \in T_1} \iprod{u^{(1)},y}\iprod{u^{(2)},y}\iprod{u^{(3)},y}\dots \iprod{u^{(2L-1)},y} y
\label{eq:statistic-main}
\end{align}
\item Let $\hat{v} = \frac{v}{\|v\|}$. Draw a set $T_v$ of samples from $\sModel$ where $|T_v| \geq c_2 \epsilon^{-3} k n^{c_2} m \log(1/\eta'_0)$.
\item If \textsc{TestColumn}($\hat{v},T_v, \frac{\eta'_0}{\beta n^{c_2}}, \frac{\beta k \eta'_0}{m n^{c_2} \log^{2c} m}, \frac{1}{\log^{2c} m}$) return a vector $\hat{z}$ then $W \leftarrow W \cup \{\hat{z}\}$.
\item Return $W$.
\end{enumerate}
 }}
 \end{center}
 \caption{\label{ALG:Single_Recovery}}
 \end{figure}
 
\newcommand{\hDs}{\widehat{\mathcal{D}}^{(s)}}
In the algorithm the set $T_1$ will be drawn from a semi-random model $\sModel$ that is appropriately re-weighted. See Section~\ref{sec:semi-random-full-alg} for how the above procedure is used in the final algorithm. For the rest of the section we will assume that the set $T_1$ is generated from $\hDs \odot \Dv$, where ${\mathcal{\hat{D}}}^{(s)}$ is an arbitrary distribution over $k$-sparse $\{0,1\}^n$ vectors. 
Next we re-state Theorem~\ref{thm:single-recovery-main} in terms of \text{RecoverColumns} to remind the reader of our goal for the section.

\begin{theorem}[Restatement of Theorem~\ref{thm:single-recovery-main}]
\label{thm:single-recovery-main-restated}
There exist constant $c_1>0$ (potentially depending on $C$) such that the following holds for any $\eps>0$ and constants $c>0$, $L \ge 8$.
Suppose the procedure \textsc{RecoverColumns} is given as input $\poly(k,m,n,1/\eps,1/\beta)$ samples from the semi-random model $\sModel$, and a set $T_1$ of $\poly(k,m,n,1/\eps,1/\beta)$ samples from $\hDs \odot \Dv$, where $\hDs$ is any arbitrary distribution over $\{0,1\}^m$ vectors with at most $k$ non-zeros and having marginals $(q_i: i \in [m])$. Then Algorithm \textsc{RecoverColumns}, with probability at least $1-\frac{1}{m^c}$, outputs a set $W$ such that
\begin{itemize}
\item For each $i$ such that $q_i \geq q_1/\log m$, $W$ contains a vector $\hat{A}_i$, and there exists $b \in \set{\pm 1}$ such that $\|A_i - b\hat{A}_i\| \leq \epsilon$. 
\item For each vector $\hat{z} \in W$, there exists $A_i$ and $b \in \set{\pm 1}$ such that $\|\hat{z} - bA_i\| \leq \epsilon$,
\end{itemize}
provided 
$k \leq \sqrt{n}/(\nu(\frac 1 m, 2L) \tau \mu^2).
$ Here 
$\nu(\eta,d):=c_1 \bigparen{C(\sigma^2 + \mu \sqrt{\frac m n}) \log^2(n/\eta)}^{d}$, and the polynomial bound also hides a dependence on $C$ and $L$.
\end{theorem}


Before we prove the theorem we need two useful lemmas stated below that follow from standard concentration bounds. The first lemma states that given samples from a semi-random model and a column $A_i$, there exist many disjoint $2L-1$ tuples $(u^{(1)}, u^{(2)}, \dots, u^{(2L-1)})$ with supports that intersect in $A_1$ and with the same sign.
\begin{lemma}
\label{lem:many-zeta-conc}
For a fixed $i \in [m]$, a fixed constant $L \ge 8$ and any $\eta_0>0$,  let $T_0$ be samples drawn from $\sModel$ where $|T_0| \geq 4(2L-1)m \log(m/\eta_0)/(\beta k)$. Then with probability at least $1 - (\eta_0/m)^{(2L-1)/4}$, 
there exist at least $\log (m/\eta_0)$ disjoint tuples $(u^{(1)}, u^{(2)}, \dots, u^{(2L-1)})$ in $T_1$ such that for each $j \in [2L-1]$ support of $u^{(j)}$ contains $A_i$ with a positive sign. 
\end{lemma}
\begin{proof}
From the definition of the the semi-random model we have that at least $\beta |T_0|$ samples will be drawn from the standard random model $\Dsr \odot \Dv$. Hence in expectation, $A_i$ will appear in at least $\frac{\beta k}{m} |T_0|$ samples. Furthermore, since $\Dv$ is a symmetric mean zero distribution, we have that in expectation $A_i$ will appear with positive sign in at least $\frac{\beta k}{2 m} |T_0| \geq 2(2L-1) \log (m/\eta_0) $ samples. Hence by Chernoff bound, the probability that in $T_0$, $A_i$ appears in less than $(2L-1)\log (m/\eta_0)$ samples with a positive sign is at most $\exp\bigparen{-\tfrac{1}{4}(2L-1) \log (m /\eta_0)}$. Hence, we get at least $\log (m/ \eta_0)$ disjoint tuples with the required failure probability.
\end{proof}
The next lemma states that with high probability the statistic of interest used in the algorithm~\eqref{eq:statistic-main} will be close to its expected value.
\begin{lemma}
\label{lem:statistic-conc}
Let $L > 0$ be a constant and fix vectors $u^{(1)}, u^{(1)}, \dots, u^{(2L-1)} \in \R^n$ of length at most $C \sigma \sqrt{k}$. Let $T_1$ be a set of samples drawn from $\tDs \odot \Dv$ where $\tDs$ is an arbitrary distribution over at most $k$-sparse vectors in $\{0,1\}^m$. For any $\epsilon_0 > 0$, and $\eta_0 > 0$, if $|T_1| \geq 2 (C^2 \sigma \sqrt{k})^{4L} \epsilon_0^{-2} m \log(1/\eta_0)$, then with probability at least $1-2m\eta_0$, we have that
$$
\Big\|\frac{1}{|T_1|} \sum_{y \in T_1} \iprod{u^{(1)},y}\iprod{u^{(2)},y}\dots \iprod{u^{(2L-1)},y} y - \E[\iprod{u^{(1)},y}\iprod{u^{(2)},y}\dots \iprod{u^{(2L-1)},y}y] \Big\| \leq \epsilon_0
$$
\end{lemma}
\begin{proof}
Using the fact that $u^{(j)}$s and samples $y \in T_1$ are weighted sum of at most $k$ columns of $A$ and the fact that $\Dv$ is in $[-C,-1] \cup [1,C]$ we have that $\|u^{(j)}\|, \|y\| \leq C \sigma \sqrt{k}$, where $\sigma = \|A\|$. Fix a coordinate $i \in [m]$. Then $\frac{1}{|T_1|} \sum_{y \in T_1} \iprod{u^{(1)},y}\iprod{u^{(2)},y}\dots \iprod{u^{(2L-1)},y} y_i$ is a sum of independent random variables bounded in magnitude by $(C^2 \sigma \sqrt{k})^{2L}$. By Hoeffding's inequality we have that the probability that the sum deviates from its expectation by more than $\frac{\epsilon_0}{\sqrt{m}}$ is at most $2e^{\frac{- |T_1| \epsilon^2_0}{2m(C^2 \sigma \sqrt{k})^{4L}}}$. By union bound we get that the probability that any coordinate deviates from its expectation by more than $\frac{\epsilon_0}{\sqrt{m}}$ is at most $2m\exp\bigparen{-\frac{|T_1| \epsilon^2_0}{2m(C^2 \sigma \sqrt{k})^{4L}}} \leq 2m\eta_0$.
\end{proof}
We are now ready to prove the main theorem of this section. The key technical ingredient that we will use is the fact that for the right choice of $u^{(1)}, u^{(2)}, \dots, u^{(2L-1)}$, the expected value of the statistic in~\eqref{eq:statistic-main} will indeed be close to one of the columns of $A$. This is formalized in Theorem~\ref{thm:main-semi-random-high-order}. 

\begin{theorem}
\label{thm:main-semi-random-high-order}
The following holds for any constant $c \geq 2$ and $L \geq 8$. 
Let $A_{n \times m}$ be a $\mu$-incoherent matrix with spectral norm at most $\sigma$, and let $\tDs$ be an arbitrary fixed support distribution over $k$ sparse vectors in $\set{0,1}^m$,\footnote{This fixes the $q$ values which specifies the moments up to order $2L$ of the support distribution $\tDs$.} and further assume that $q_1 \geq q_{\max}/\log m$. 
Let $u^{(1)}=A\upzeta{1},u^{(2)}=A\upzeta{2}, \dots, u^{(2L-1)}=A\upzeta{2L-1}$ be samples drawn from $\Dsr \odot \Dv$ conditioned on $\upzeta{t}(1) > 0$ for $t \in [2L-1]$. 
With probability at least $1-\frac{1}{\log^2 m}$ over the random choices of $\upzeta{1}, \ldots, \upzeta{2L-1}$, 
$$
\E_{\substack{y=Ax\\ x \sim \tDs \odot \Dv}}\Big[\iprod{u^{(1)},y}\iprod{u^{(2)},y}\ldots \iprod{u^{(2L-1)},y}y \Big] = q_1A_{1} + e_{1}
$$
where $\|e_{1}\|=O \Paren{\frac {q_1}{\log^c m}}$, provided 
$
k \leq \sqrt{n}/(\nu(\frac 1 m, 2L) \tau \mu^2 ).
$
Here we have that $\nu(\eta, d) := c_1 \left(C(\sigma^2 + \mu \sqrt{\frac m n}) \log^2(n/\eta) \right)^{d}$, and $c_1$ is an absolute constant.
\end{theorem}
\anote{4/22: I modified $\nu$ definition and the $k$ bound to match Theorem 5.2-- it also puts the extra $log m$ factor inside. }
Setting $L=8$, if $m = \widetilde{O}(n)$ and $\sigma = \widetilde{O}(1)$, the above bound on $k$ will be satisfied when $k = \widetilde{O}(\sqrt{n})$.
\anote{The dependence on $\tau$ needs to be checked. I think it's only $1/\tau$ and not $1/\tau^{2L}$. Also, why do we get $\min\set{\sqrt{n}, \sqrt{m}}$? I only see $\sqrt{m}$ coming in. Of course there are other terms involving $m/n$ and $\sigma$ in the $\nu$ term...}
Before we proceed, we will first prove Theorem~\ref{thm:single-recovery-main} assuming the proof of the above theorem. 
\begin{proof}[Proof of Theorem~\ref{thm:single-recovery-main}]

Consider a fixed support distribution $\hDs$; hence this specifies its marginals $(q_i: i \in [m])$ and its first $2L$ moments specified by the values for $q_{i_1, \dots, i_{t}}(d_1, \dots, d_{t})$, where $t \le 2L, i_1, \dots, i_{t} \in [m], d_1, \dots, d_{t} \in [2L]$. We will first prove that for a fixed $\hDs$, the procedure \textsc{RecoverColumns} will succeed with probability at least $1-\eta_0$ where $\eta_0\le \exp\bigparen{-m^{O(L)}}$, and perform a union bound over an appropriate net of $\hDs$ i.e., a net of values for $q_{i_1, \dots, i_{t}}(d_1, \dots, d_{t})$ (where $t \le 2L, i_1, \dots, i_{t} \in [m], d_1, \dots, d_{t} \in [2L]$).  

\anote{5/5: Tried to get rid of confusing reuse of constant $c$.} Let $c_*>0$ be an absolute constant (that will chosen later appropriately according to Theorem~\ref{thm:sr:testing}).  
Let $T_0$ be the set of samples drawn in step $2$ of the procedure where $|T_0| \geq 4(2L-1)m \log (m/\eta'_0) / (\beta k) $, where $\eta'_0 = \eta_0/\exp(m^{O(L)} \log(1/\eps))$. From Lemma~\ref{lem:statistic-conc} we can assume that for each $2L-1$ tuple in $T_0$, the statistic in \eqref{eq:statistic-main} is $\frac{q_{\max}}{\log^{4c_*} m}$-close to its expectation except with probability at most $2m{|T_0|}^{2L-1}\eta'_0$, provided that $|T_1| \geq 2m(C^2 \sigma \sqrt{k})^{4L} \log^{c_3} m \log(1/\eta'_0)$. Let $A_i$ be a column such that in $\hDs \odot \Dv$ we have that $q_i \geq q_1/\log m$. From Lemma~\ref{lem:many-zeta-conc} we have that, except with probability at most $m\exp\bigparen{-\tfrac{1}{4}(2L-1)\log (m / \eta'_0)}$, there exist at least $\log (m / \eta'_0)$ disjoint $(2L-1)$-tuples in $T_0$ that intersect in $A_i$ with a positive sign. Hence we get from Theorem~\ref{thm:main-semi-random-high-order} that there is at least $1 - ({\frac{1}{\log^2 m}})^{\log (m / \eta'_0)}$ probability that the vector $\hat{v}$ computed in step 4 of the algorithm will be $\frac{1}{\log^{2c_*} m}$-close to $A_i$. Further, from Lemma~\ref{lem:incoherence-implies-rip}, $A$ is $(k,O(1/\log^{4c_*+1}m))$-RIP, and $c_*>0$ is an appropriate absolute constant.\anote{5/5:added.}
Then we get from Theorem~\ref{thm:sr:testing} (with $\gamma = \eta'_0/|T_0|^{4L}$, and $\eta = 1/\log^{2c_*} m$) that a vector that is $\epsilon$-close to $A_i$ will be added to $W$ except with probability at most $\frac{\eta'_0}{|T_0|^{4L}}$. Furthermore no spurious vector that is $\frac{1}{\log^{c_*} m}$ far from all $A_i$ will be added, except with probability at most $\frac{\eta'_0}{|T_0|^{2L}}$. Hence the total probability of failure of the algorithm is at most $m(2{|T_0|}^{(2L-1)}e^{-\log^2 (m / \eta'_0)} + \exp\paren{-\frac{(2L-1)\log (m/\eta'_0)}{4}} + \frac{1}{(\log m)^{2 \log (m / \eta'_0)}} + \frac{\eta'_0 }{|T_0|^{2L}}) \leq \eta_0$.

Finally, it is easy to see that it suffices to consider a $\eps_1$-net over the values of $q_{i_1, \dots, i_{t}}(d_1, \dots, d_{t})$ for each $t \le 2L, i_1, \dots, i_{t} \in [m], d_1, \dots, d_{t} \in [2L]$, where $\eps_1=\eps m^{-O(L)}$. Hence, it suffices to consider a net of size $N'=\exp\bigparen{(2Lm)^{2L} \log(1/\eps_1)}= \exp\bigparen{m^{O(L)} \log(1/\eps)}$. Since our failure probability $\eta_0<1/(N' m^2)$, we can perform an union bound over the net of support distributions $\hDs$ and conclude the statement of the theorem. 
\end{proof}

\subsection{Proof of Theorem~\ref{thm:main-semi-random-high-order}: Recovering Frequently Occurring Columns}

Let $y = \sum_{i \in [m]} x_i A_i$. Then we have that
\begin{align}
&\E_{x,v}[\iprod{u^{(1)},y}\iprod{u^{(2)},y}\ldots \iprod{u^{(2L-1)},y}y] = \sum_{i \in [m]} \gamma_i A_i,~~~ \text{where} \nonumber\\
&\gamma_i = \sum_{j_1,\dots,j_{2L-1} \in [m]} \upzeta{1}_{j_1} \ldots \upzeta{2L-1}_{j_{2L-1}} \sum_{i_1,\ldots i_{2L-1} \in [m]} \E[x_{i_1} \ldots x_{i_{2L-1}} x_i] M_{i_1,j_1} \dots M_{i_{2L-1},j_{2L-1}}
\label{eq:gamma-i-semi-random}
\end{align}
We will show that with high probability~(over the $\zeta$s), $\gamma_1 = q_1(1 \pm \frac{1}{\log^c m})$ and that $\|\sum_{i \neq 1} \gamma_i A_i\| = o \Paren{\frac{q_1}{\log^c m}}$ for our choice of $k$.
Notice that for a given $i$, any term in the expression for $\gamma_i$ as in \eqref{eq:gamma-i-semi-random} will survive only if the indices $i_1, i_2, \dots, i_{2L-1}$ form a partition $S = (S_1, S_2, \ldots S_R)$ such that $|S_1|$ is odd and $|S_p|$ is even for $p \geq 2$. $S_1$ is the special partition that must correspond to indices that equal $i$. Hence, $(S_1, S_2, \ldots S_R)$ must satisfy $i_t=i$ for $t \in S_1$ and $i_t = i^*_{r}$ for $t \in S_r$ for $r \geq 2$, for indices $i^*_2, \dots i^*_{R} \in [m]$. We call such a partition a valid partition and denote $|S|=R$ as the size of the partition. Let $d_1, d_2, \dots d_R$ denote the sizes of the corresponding sets in the partition, i.e., $d_j = |S_j|$. Notice that $d_1 \geq 1$ must be odd and any other $d_j$ must be an even integer. Using the notation from Section~\ref{subsec:model} we have that 
$$
\E\big[x_{i_1} \dots x_{i_{2L-1}}x_i\big] = \begin{cases} q_{i,i^*_2,i^*_3,\dots,i^*_{R}}(d_1+1, d_2, \dots, d_R), & S \text{ is valid}\\
0, & \textrm{otherwise}
\end{cases}
$$
Recall that by choice, $\upzeta{\ell}_1 \ge 1$ for each $\ell \in [2L-1]$. Hence, the 
value of the inner summation in \eqref{eq:gamma-i-semi-random} will depend on how many of the indices $j_1, j_2, \dots, j_{2L-1}$ are equal to $1$. This is because we have that $\zeta^{\ell}_1$ is a constant in $[1,C]$ for all $\ell \in [2L-1]$. Hence, let $H = (H_1, H_2, \dots H_R)$ be such that $H_r \subseteq S_r$ and for each $r$ we have that $j_t=1$ for $t \in H_r$. Let $h$ denote the total number of fixed variables, i.e. $h = \sum_{r \in [R]} |H_r|$. Notice that $h$ ranges from $0$ to $2L-1$ and there are $2^{2L-1}$ possible partitions $H$. The total number of valid $(S,H)$ partitionings is at most $(4L)^{2L}$. Hence 
\begin{align}
&\gamma_{i} = \sum_{(S,H)} \gamma_i(S,H), ~~\text{where}~~ \gamma_i(S,H):= \nonumber\\
&= \sum_{\substack{(i^*_2, \dots, i^*_R) \\\in [m]^{R-1}}} q_{i,i^*_2,\dots,i^*_{R}}(d_1+1, d_2, \dots, d_R) \prod_{r \in [R]} \sum_{\substack{J_{S_r\setminus H_r} \in \\ ([m]\setminus \set{1})^{S_r\setminus H_r}}} (M_{i^*_r,1})^{|H_r|} \prod_{t \in H_r} \upzeta{t}_{1} \cdot \prod_{t \in S_r \setminus H_r}M_{i^*_r,j_t} \upzeta{t}_{j_t} \nonumber
\end{align}
Note that by triangle inequality, $\norm{\sum_{i \ne 1} \gamma_i A_i}_2 \le \sum_{(S,H)} \norm{ \sum_{i \ne 1} \gamma_i(S,H) A_i}_2$. Hence, we will obtain bounds on $\gamma_i(S,H)$ depending on the type of partition $(S,H)$, and upper bound $\norm{ \sum_{i \ne 1} \gamma_i(S,H) A_i}_2$ for each $(S,H)$. We have three cases depending on the number of fixed indices $h$. 
\anote{Explain what $J_S$ means.}

\noindent \textbf{Case 1: $h=0$.} In this case none of the random variables are fixed to $1$. Here we use Lemmas~\ref{lem:frobnorm:bound} to claim that with probability at least $1-\eta$ over the randomness in $\zeta_j$s,
\begin{align}
|\gamma_i(S,H=\emptyset)| 
&\leq {q_i \eta^{-1/2} \cdot \nu(\eta,2L) \Bigparen{\frac{\tau^{3/2} k^{3/2}}{m}}^{L-\tfrac{1}{2}}}
\label{eq:gammai-F-empty}
\end{align}
In the final analysis we will set $\eta = \left(m \log^2 m (4L)^{2L} \right)^{-1}$.  In this case we get that
\begin{align*}
|\gamma_i(S,H=\emptyset)| 
&\leq {q_i \sqrt{m} \log m (4L)^L \cdot \nu(\eta,2L) \Bigparen{\frac{\tau^{3/2} k^{3/2}}{m}}^{L-\tfrac{1}{2}}}.
\end{align*}
We will set $L$ large enough in the above equation such that we get 
\begin{align*}
|\gamma_i(S,H=\emptyset)| 
&\leq q_i C^{2L} \cdot \nu^2(\eta,2L) \frac{\mu^2 k}{n}.
\end{align*}
$L \geq 8$ suffice for this purpose for our choice of $k$.

\noindent \textbf{Case 2: $h=2L-1$.} In this case all the random variables are fixed and $\gamma_i$ deterministically equals
\begin{align}
\gamma_i(S,H) &= \zeta^{(1)}_1 \zeta^{(2)}_1 \dots \zeta^{(2L-1)}_1 \sum_{i^*_2, i^*_3, \dots, i^*_R \in [m]} q_{i,i^*_2,i^*_3,\dots,i^*_{R}}(d_1+1, d_2, \dots, d_R) M_{i,1}^{d_1} M_{i^*_2,1}^{d_2} \dots M_{i^*_R,1}^{d_R} \nonumber \\
&= \begin{cases}
\zeta^{(1)}_1 \zeta^{(2)}_1 \dots \zeta^{(2L-1)}_1 \cdot q_i M_{i,1}^{d_1}, & R=1\\
\pm O \left(C^{4L-1}\sigma^{2(R-1)} \cdot q_i  M_{i,1}^{d_1}  \right), & \text{ otherwise },
\end{cases}
\label{eq:semi-random-F-full}
\end{align}
where the case when $R \ne 1$ follows from Lemma~\ref{lem:semi-random-F-full}. 

\noindent \textbf{Case 3: $1 \leq h< 2L-1$.} 
For improved readability, for the rest of the analysis we will use the following notation. For a set of indices $J=(j_1, \dots, j_d)$, for $S \subset [d]$, we will use $J_S=(j_t: t \in S)$, and $\sum_{\substack{J_S}} $ to denote the sum over indices ${J_{S} \in ([m]\setminus \set{1})^{|S|}}$.
In this case we can write 
\begin{align*}
\gamma_i(S,H) &= q_i(d_1+1) \prod_{t' \in H_1} \zeta^{(t')}_1 \sum_{\substack{J_{S_1 \setminus H_1}}} M^{|H_1|}_{i,1} \prod_{t \in S_1 \setminus H_1} \upzeta{t}_{j_t} M_{i,j_t}  F_{i},~~~\text{where} \nonumber\\ 
 F_{i} &= \sum_{i^*_{2}, \dots, i^*_{R} \in [m]} \Big(\frac{q_{i, \dots,i^*_r,\dots, i^*_R}(d_1+1, d_2, \dots, d_R)}{q_{i}(d_1+1)}\Big) \prod_{r=2}^R \Bigparen{\sum_{\substack{J_{S_r \setminus H_r} }} M^{|H_p|}_{i^*_p,1} \prod_{t \in S_p \setminus J_p}  M_{i^*_{p},j_t} \upzeta{t}_{j_t} }.
\end{align*}
When $|H_1| \geq 1$, we can apply Lemma~\ref{lem:semi-random-inductive} with $r=1$ we get with probability at least $1-\eta$ over the randomness in $\zeta_j$s that
$|F_i| \leq \nu(\eta,2L)$.
Hence, we get that with probability at least $1-\eta$ over the randomness in $\zeta_j$s,
\begin{align}
\gamma_i(S,H) = w_i q_i(d_1+1) \prod_{t' \in H_1} \zeta^{(t')}_1  \sum_{\substack{J_{S_1 \setminus H_1} }} \prod_{t \in S_1 \setminus H_1} \upzeta{t} M_{i,j_t} M^{|H_1|}_{i,1}
\label{eq:semi-random-gammai-simplified}
\end{align}
where $|w_i| \leq \nu(\eta,2L)$. Next we use Lemma~\ref{lem:conc-special-partition} to get that with probability at least $1-2\eta$, the above sum is bounded as
\begin{align}
\abs{\gamma_i(S,H)} \leq \begin{cases}
\frac{q_i w_i \mu}{\sqrt{n}} \bigparen{ Z_i + \nu(\eta,2L) \sqrt{\frac{k}{m}}}, & i\neq 1\\
\frac{q_i w_i \mu}{\sqrt{n}} \cdot \nu(\eta,2L) \sqrt{\frac{k}{m}}, & i = 1
\end{cases}
\label{eq:semi-random-gammai-final-bound}
\end{align}
Here $Z_i=\prod_{t \in S_1 \setminus H_1} \abs{\upzeta{t}_{i}}$ are non-negative random variables bounded by $C^{|S_1 \setminus H_1|}$. Further $Z_i$ are each non-zero with probability at most $p\cdot (\tau p)^{|S_1 \setminus H_1|-1}$ and they are $\tau$-negatively correlated
 (with the values conditioned on non-zeros being drawn independently).
\anote{What is $\tau$-negatively correlated? Mention in prelims?}

If $|H_1|=0$ then there must exist $r \geq 2$ such that $|H_r| \geq 1$. Without loss of generality assume that $|H_2| \geq 1$. Then we can write $\gamma_i(S,H)$ as
\begin{align*}
\gamma_i(S,H)&=\sum_{i^*_2} q_{i,i^*_2}(d_1+1,d_2)  \sum_{\substack{J_{S_1}}} \prod_{t \in S_1} M_{i,j_t} \upzeta{t}_{j_t}  \cdot \prod_{t' \in H_2} \upzeta{t'}_1 \sum_{\substack{J_{S_2 \setminus H_2} }}M^{|H_2|}_{i^*_2,1} \prod_{t \in S_2 \setminus H_2} \upzeta{t}_{j_t} M_{i^*_2,j_t}  F'_{i,i^*_2},\\
\text{where }& F'_{i,i^*_2}=\sum_{i^*_{3}, \dots, i^*_{R} \in [m]} \Big(\frac{q_{i, \dots,i^*_r,\dots, i^*_R}(d_1+1, d_2, \dots, d_R)}{q_{i,i^*_2}(d_1+1, d_2)}\Big) \prod_{r=3}^R \left(\sum_{J_{S_r \setminus H_r} } M^{|H_r|}_{i^*_r,1} \prod_{t \in S_r \setminus J_r}  M_{i^*_{r},j_t} \upzeta{t}_{j_t} \right).
\end{align*}
We can again apply Lemma~\ref{lem:semi-random-inductive} with $r=2$ to get that with probability at least $1-\eta$ over the randomness in $\zeta_j$s,
$$
\abs{F'_{i,i^*_2}} \leq \nu(\eta,2L-d_1-d_2-1).
$$
Hence, we can rearrange and write $\gamma_i(S,H)$ as
\begin{align}
\gamma_i(S,H) &= q_{i}(d_1+1) \prod_{t' \in H_2} \upzeta{t'}_1  \sum_{J_{S_1}} \prod_{t \in S_1} \upzeta{t}_{j_t} M_{i,j_t} F''_{i}, ~~\text{ where}\nonumber
\\
F''_{i} &= \sum_{i^*_2 \in [m]} \frac{q_{i,i^*_2}(d_1+1,d_2)}{q_i(d_1+1)} \cdot F'_{i,i^*_2} M^{|H_2|}_{i^*_2,1}\sum_{J_{S_2 \setminus H_2}} \prod_{t \in S_2 \setminus H_2} \upzeta{t}_{j_t} M_{i^*_2,j_t}  \label{eq:gammai-simplified-h1-zero} 
\end{align}
We split this sum into two, depending on whether $i^*_2=1$ or not. 
Here we have that
\begin{align}
F''_{i,a} := \frac{q_{i,1}(d_1+1,d_2)}{q_i(d_1+1)}\cdot F'_{i,i^*_2=1}\sum_{J_{S_2 \setminus H_2}} \prod_{t \in S_2 \setminus H_2} \upzeta{t}_{j_t} M_{1,j_t} 
\label{eq:F1}
\end{align}
and
\begin{align}
F''_{i,b} :=  \sum_{i^*_2 \in [m]\setminus \{1\}} \frac{q_{i,i^*_2}(d_1+1,d_2)}{q_i(d_1+1)} \cdot F'_{i,i^*_2} M^{|H_2|}_{i^*_2,1}  \sum_{J_{S_2 \setminus H_2} } \prod_{t \in S_2 \setminus H_2} \upzeta{t}_{j_t} M_{i^*_2,j_t} 
\label{eq:F2}
\end{align}
Here when $|H_2| < |S_2|$ we will use Lemma~\ref{lem:conc-special-partition} and the fact that $j_t \neq 1$ for $t \in |S_2 \setminus H_2|$ to get that with probability at least $1-\eta$ over the randomness in $\zeta_j$s, we can bound $F''_{i,a}$ as
$$
|F''_{i,a}| \leq \nu(\eta,2L-d_1-1-|H_2|)\cdot \frac{q_{i,1}(d_1+1,d_2)}{q_i(d_1+1)}  \sqrt{\frac{k \tau}{m}}
$$
Combining this with the simple bound when $S_2=H_2$ we get
\begin{align}
|F''_{i,a}| &= \begin{cases}
\nu(\eta,2L-d_1-d_2-1) \cdot \frac{q_{i,1}(d_1+1,d_2)}{q_i(d_1+1)}, & |H_2| = |S_2| \\
\nu(\eta,2L-d_1-1-|H_2|)\cdot \frac{q_{i,1}(d_1+1,d_2)}{q_i(d_1+1)} \sqrt{\frac{k \tau}{m}}, & \text{ otherwise }
\end{cases} 
\label{eq:bound-F1}
\end{align}
Next we bound $F''_{i,b}$. When $|H_2| < |S_2|$ we will use the concentration bound from Lemma~\ref{lem:conc-one-partition}. However, when applying Lemma~\ref{lem:conc-one-partition} we will use the fact that $|w_i| = |F'_{i,i^*_2}M^{|H_2|}_{i^*_2,1}| \leq \nu(\eta,2L-d_1-d_2-1) \mu/\sqrt{n}$. This is because we are summing over $i^*_2 \neq 1$ and we have $|H_2| \geq 1$. Hence, by incoherence we have that $|M^{|H_2|}_{i^*_2,1}| \leq \mu/\sqrt{n}$. Hence we get that with probability at least $1-\eta$ over the randomness in $\set{\zeta_j}$
to get that
$$
|F''_{i,b}| \leq \nu(\eta,2L-d_1-1-|H_2|)\frac{\mu}{\sqrt{n}}
$$
When $|H_2|=|S_2|$ we get 
$$
|F''_{i,b}| \leq \sum_{i^*_2 \in [m]\setminus \{1\}} \frac{q_{i,i^*_2}(d_1+1,d_2)}{q_i(d_1+1)} \cdot \bigabs{F'_{i,i^*_2} M^{|H_2|}_{i^*_2,1}}.
$$ 
Using the fact that $|H_2| \geq 2$, $i^*_2 \neq 1$ and that the columns are incoherent we get,
\begin{align*}
|F''_{i,b}| & \leq \nu(\eta,2L-d_1-d_2-1) \cdot \frac{\mu^2}{n} \cdot \sum_{i^*_2 \in [m]\setminus \{1\}} \frac{q_{i,i^*_2}(d_1+1,d_2)}{q_i(d_1+1)} \\
&\leq \nu(\eta,2L-d_1-d_2-1)\cdot \frac{C^{d_2}\mu^2 k}{n}
\end{align*} 
where in the last inequality we use Lemma~\ref{eq:qd-bound}.
Combining the above bounds we get that with probability least $1-\eta$ over the randomness in $\zeta_j$s, 
\begin{align}
F''_{i,b} &= \begin{cases}
\nu(\eta,2L-d_1-d_2-1)\frac{C^{d_2}\mu^2 k}{n}, & |H_2| = |S_2| \\
\nu(\eta,2L-d_1-1-|H_2|)\frac{\mu}{\sqrt{n}}, & \text{ otherwise }
\end{cases} 
\label{eq:bound-F2}
\end{align}
We Combine the above bounds on $F''_{i,a}$ and $F''_{i,b}$ and to get the following bound on $F''_i$ that holds with probability $1-2\eta$ over the randomness in $\zeta$s
\begin{align}
|F''_i| &\leq \begin{cases}
\nu(\eta,2L-d_1-d_2-1) \left( \frac{q_{i,1}(d_1+1,d_2)}{q_i(d_1+1)} + \frac{C^{d_2}\mu^2 k}{n} \right), & |H_2| = |S_2|\\
\nu(\eta,2L-d_1-1-|H_2|) \left( \frac{q_{i,1}(d_1+1,d_2)}{q_i(d_1+1)}\sqrt{\frac{k \tau}{m}} + \frac{\mu}{\sqrt{n}} \right), & \text{ otherwise }
\end{cases}
\end{align}
Finally, we get a bound on $\gamma_i(S,H)$ by using Lemma~\ref{lem:conc-special-partition} with $w_i$ in the Lemma set to $q_i(d_1+1) \prod_{t' \in H_2} \zeta^{(t')}_1 F''_i$. Notice that the absolute value of $w_i$ is bounded by $q_i C^{d_1+1} C^{|H_2|} |F''_i|$. Hence we get that
\begin{align}
|\gamma_i(S,H)|  
& \leq \begin{cases}
q_1 C^{2L} \nu(\eta,2L) \sqrt{\frac{\tau k}{m}}, & i=1\\
q_i C^{2L} \nu(\eta,2L) \Bigparen{Z_i + \sqrt{\frac{k \tau}{m}}} \Bigparen{\frac{q_{i,1}}{q_i} + \frac{\mu^2 k}{n}}, & \text{ otherwise }
\end{cases}
\label{eq:gamma-i-H-zero-final}
\end{align}
Here $Z_i=\prod_{t \in S_1} \abs{\upzeta{t}_{i}}$ are non-negative random variables bounded by $C^{|S_1|}$. Further $Z_i$ are each non-zero with probability at most $p\cdot (\tau p)^{|S_1|-1}$ and they are $\tau$-negatively correlated
 (with the values conditioned on non-zeros being drawn independently).

\paragraph{Putting it Together.} We will set $\eta = \left({m \log^2 m (4L)^{2L}} \right)^{-1}$ so that all the above bounds hold simultaneously for each $i \in [m]$ and each partitioning $S,H$. We first gather the coefficient of $A_1$, i.e., $\gamma_1$. For the case of $h=2L-1$ we get that $\gamma_1(S,H) \geq q_1$ from \eqref{eq:semi-random-F-full}. Here we have used the fact that $\zeta^{(t)}_{1} \geq 1$ for all $t \in [2L-1]$. For any other partition we get from \eqref{eq:semi-random-gammai-final-bound}, \eqref{eq:gamma-i-H-zero-final} and \eqref{eq:gammai-F-empty} that 
\anote{Put in a $\tau$ term.}
$$
\gamma_1 \leq q_1 C^{2L}\nu(\eta,2L)\cdot \sqrt{\frac{\tau k}{m}} = O\Bigparen{\frac{q_1}{(4L)^{2L} \log^{c} m} }
$$
for our choice of $k$. Hence, summing over all partitions we get that term corresponding to $A_1$ in \eqref{eq:gamma-i-semi-random} equals $a_1 A_1 + e_1$ where $a_1 \geq q_1$ and $\|e_1\| = O(\frac{q_1}{\log^c m})$. 

Next we bound $\|\sum_{i \neq 1} \gamma_i A_i\|$. 
In order to show that $\|\sum_{i \neq 1} \gamma_i A_i\| \leq \frac{q_1}{\log^c m}$ it is enough to show that for any $(S,H)$,
$$
\bignorm{\sum_{i \neq 1} \gamma_i(S,H) A_i}_2 \leq \frac{q_1}{(4L)^{2L}\log^c m}
$$
Using the fact that $\|A\|_2 \le \sigma$, we have that 
$
\norm{\sum_{i \neq 1} \gamma_i(S,H) A_i}_2 \leq \sigma \sqrt{\sum_{i \neq 1} \gamma^2_i(S,H)}.
$
Hence, it will suffice to show that for any 
\begin{equation}\label{eq:gammaSH}
\forall (S,H),~~\sum_{i \neq 1} \gamma^2_i(S,H) \leq \frac{q^2_1}{(4L)^{4L}  \sigma^2 \log^{2c}m}
\end{equation}.

\anote{Changing all the $\nu^2(\eta,2L)$ to $\nu(\eta,2L)$.}
From \eqref{eq:semi-random-F-full}, \eqref{eq:semi-random-gammai-final-bound}, \eqref{eq:gamma-i-H-zero-final} and \eqref{eq:gammai-F-empty} we get that 
We notice that across all partitions 
\begin{align*}
|\gamma_i(S,H)| &\leq q_i C^{2L} \nu(\eta,2L) \Bigparen{Z_i + \sqrt{\frac{\tau k}{m} }}\Bigparen{\frac{q_{i,1}}{q_i} + \frac{\mu^2 k}{n}}\\
&= \gamma^{(1)}_i(S,H) + \gamma^{(2)}_i(S,H) + \gamma^{(3)}_i(S,H) + \gamma^{(4)}_i(S,H), \text{ where }\\
\gamma^{(1)}_i(S,H) &= q_{i,1} C^{2L} \nu(\eta,2L)Z_i, \quad \quad\quad 
\gamma^{(2)}_i(S,H) = q_{i} C^{2L} \nu(\eta,2L)\frac{\mu^2 k}{n}Z_i,\\
\gamma^{(3)}_i(S,H) &= q_{i,1} C^{2L} \nu(\eta,2L)\sqrt{\frac{\tau k}{m}},  \quad\quad
\gamma^{(4)}_i(S,H) = q_{i} C^{2L} \nu(\eta,2L)\frac{\mu^2 k \sqrt{\tau k}}{n \sqrt{m}}.
\end{align*}
We will separately show that $\forall j \in \set{1,2,3,4}$
$$
\sum_{i \neq 1} (\gamma^{(j)}_i(S,H))^2 \leq \frac{q^2_1}{4 \sigma^2 (4L)^{4L} \log^{2c} m }.
$$
For $j=4$ we have
\begin{align*}
\sum_{i \neq 1} (\gamma^{(4)}_i(S,H))^2 &= {\sum_{i \neq 1} q^2_i C^{4L} \nu^2(\eta,2L) \frac{\mu^4 k^3}{n^2 m}}\\
&\leq q_1^2 \log^2 m \sum_{i \neq 1} C^{4L} \nu^2(\eta,2L)   \frac{\mu^4 k^3}{n^2 m} \leq \frac{q^2_1}{4 \sigma^2 (4L)^{4L} \log^{2c} m }
\end{align*}
for our choice of $k$ and using the fact that $q_1 \geq q_{\max}/\log m$.
Similarly for $j=3$ we get that
\begin{align*}
\sum_{i \neq 1} (\gamma^{(3)}_i(S,H))^2 &=
{\sum_{i \neq 1} q^2_{i,1} C^{4L} \nu^2(\eta,2L) \frac{\tau k}{m}}
\leq q^2_1 C^{4L} \nu^2(\eta,2L) {\sum_{i \neq 1} \frac{q^2_{i,1}}{q^2_1} \cdot \frac{\tau k}{m}}\\
& \leq q^2_1 C^{4L} \nu^2(\eta,2L) {\sum_{i \neq 1} \frac{q_{i,1}}{q_1}  \frac{\tau k}{m}} \leq q^2_1 C^{4L} \nu^2(\eta,2L) \frac{\tau k^2}{{m}} \leq  \frac{q^2_1}{4 \sigma^2 (4L)^{4L} \log^{2c} m }.
\end{align*}
Here we have used the fact that $\sum_{i \neq 1} \frac{q_{i,1}}{q_1} \leq k$, and $k< \sqrt{m} \tau^{-1}/\nu(\eta,2L)$ (this is one of the terms that requires $k=o(\sqrt{m})$). Next we bound the term corresponding to $j=1$, i.e.,
\begin{align}
\sum_{i \neq 1} (\gamma^{(1)}_i(S,H))^2 &= C^{4L} \nu^2(\eta,2L) {\sum_{i \neq 1} q^2_{i,1}Z^2_i} 
 \leq q^2_1 C^{4L} \nu^2(\eta,2L) {\sum_{i \neq 1} \frac{q^2_{i,1}} {q^2_1} Z^2_i}
\label{eq:random-sum-1}
\end{align}
Notice that $Z_i$s are non-negative random variables with support distribution that is $\tau$-negatively correlated. Hence, using Lemma~\ref{lem:Z-conc} with $p = \frac{k}{m}$ and $\|a\|_1 = \sum_{i  \neq 1} \frac{q_{i,1}}{q_1} \leq k$, 
\begin{align*}
\sum_{i \neq 1} \frac{q^2_{i,1}} {q^2_1} Z^2_i \leq \frac{C^2 \tau}{\log^{(c)} m},
~~\text{ with probability at least } 1-\frac{2}{m^2} - \frac{1}{\log^2 m}.
\end{align*}
 \anote{Why do we have the extra $1/\log^2 m$ in the failure probability?}
Substituting back in \eqref{eq:random-sum-1} we get for our choice of $k=o(\sqrt{m})$, 
$$
\sum_{i \neq 1} (\gamma^{(1)}_i(S,H))^2 \leq \frac{q^2_1}{4 \sigma^2 (4L)^{4L} \log^{2c} m }.$$
Finally for $j=2$ we get  
\begin{align}
\sum_{i \neq 1} (\gamma^{(2)}_i(S,H))^2 &= q^2_{i} C^{4L} {\nu^2(\eta,2L)} {\sum_{i \neq 1} \frac{\mu^4 k^2}{n^2}Z^2_i}  \leq q^2_{1} \log^2 m C^{4L} {\nu^2(\eta,2L)} {\sum_{i \neq 1} \frac{\mu^4 k^2}{n^2}Z^2_i}
\label{eq:random-sum-2}
\end{align}
Again using Lemma~\ref{lem:Z-conc} with $p = \frac{k}{m}$ and $\|a\|_1 = \sum_{i \neq 1} \frac{\mu^4 k^2 }{n^2 } \leq \frac{1}{\log^c m}$ we get that 
$$
\sum_{i \neq 1} \frac{\mu^4 k^2}{n^2}Z^2_i \leq \frac{C^2 \tau}{\log^{c} m}, ~~\text{ with probability at least } 1-\frac{1}{m^2} - \frac{1}{\log^2 m}.
$$
Substituting back in \eqref{eq:random-sum-2} we get for our choice of $k$,
$$
\sum_{i \neq 1} (\gamma^{(2)}_i(S,H))^2 \leq \frac{q^2_1}{4 \sigma^2 (4L)^{4L} \log^{2c} m }
$$
Hence this establishes \eqref{eq:gammaSH} and we get the required bound on $\norm{\sum_{i \ne 1} \gamma_i A_i}_2$.
This concludes the proof of Theorem~\ref{thm:main-semi-random-high-order}.

\subsection{The Semirandom algorithm: Proof of Theorem~\ref{thm:main-full-algorithm}. }
\label{sec:semi-random-full-alg}
In this section we use the subroutine developed in Section~\ref{sec:semi-random-recovery} for recovering large frequency columns to show how to recover all the columns of $A$ with high probability and prove Theorem~\ref{thm:main-full-algorithm}. Recall that the algorithm in Figure~\ref{ALG:Single_Recovery} searches over all $2L-1$ tuples $(u^{(1)}, u^{(2)}, \dots, u^{(2L-1)})$ and computes the statistic in \eqref{eq:statistic-main}. If the data were generated from the standard random model, one would be able to claim that for each column $A_i$, at least one of the candidate tuples will lead to a vector close to $A_i$. However, in the case of semi-random data, one can only hope to recover large frequency columns as the adversary can add additional data in such a manner so as to making a particular column's marginal $q_i$ very small. Hence, we need an iterative approach where we recover large frequency columns and then re-weigh the data in order to uncover more new columns. We will show that such a re-weighting can be done by solving a simple linear program. A key step while doing the re-weighting is to find out if a given sample $y=Ax$ contains columns $A_i$ for which we already have good approximations $\hat{A}_i$. Furthermore, we also need to make sure that this can be done by just looking at the support of $y$ and not the randomness in the non-zero values of $x$. Lemma~\ref{lem:find-support-deterministic} shows that this can indeed be done by simply looking at $|\iprod{y,\hat{A}_i}|$ if $A$ is incoherent and $k$ does not exceed $\sqrt{n}$. We will rely on this lemma in our algorithm described in Figure~\ref{ALG:Full_Recovery}.
 \begin{figure}[htb]
 \begin{center}
 \fbox{\parbox{1\textwidth}{
{\bf Algorithm \textsc{RecoverDict}}$(\sModel, L, \epsilon)$
\begin{enumerate}
\item Initialize $W^* = \emptyset$, $\tol = \frac{1}{m^2}$. Constants $c_1,c_2>0$ are appropriately chosen.
\item Repeat $m$ times
\begin{itemize}
\item Draw  set $T$ of samples from $\sModel$ where $|T| \geq \frac{c_1 k^2 m^4 n^{O(1)} \log^3 m}{\beta^2 \epsilon^6}$.
\item For each $\hat{A}_i \in W^*$, find the set $V(\hat{A}_i) = \{(y=Ax) \in T: i \in supp(x) \}$ as in Lemma~\ref{lem:find-support-deterministic}.
\item Find weights $w_j \in [0,1]$ for $j=1$ to $|T|$ such that 
\begin{align*}
\sum_j w_j &\geq \beta |T| \\
\sum_{j: y_j \in V(\hat{A}_i)} w_j &\leq \frac{k(1+\tol)}{m}(\sum_j w_j), \text{ for all } \hat{A}_i \in W^*
\end{align*}
\item Form $T_1$ by picking each $y_j \in T$ with probability $\frac{w_j}{\sum_j w_j}$ where $|T_1| = c_2 k n^{c_3} m/\epsilon^3$. 
\item $W^* = W^* \cup \textsc{RecoverColumns}(\sModel, T_1, L, \epsilon)$.
\end{itemize}
\item Return $W^*$.
\end{enumerate}
 }}
 \end{center}
 \caption{\label{ALG:Full_Recovery}}
 \end{figure}
We now provide the proof of Theorem~\ref{thm:main-full-algorithm} restated below to remind the reader.
\anote{4/22: Added any $\epsilon$, made number of samples $T$ to $N$.}
\begin{theorem}[Restatement of Theorem~\ref{thm:main-full-algorithm}]
\label{thm:main-full-algorithm-restate}
Let $A$ be a $\mu$-incoherent $n \times m$ dictionary with spectral norm $\sigma$. For any $\epsilon>0$, any constant $L \geq 8$, given $N=\poly(k,m,n,1/\eps,1/\beta)$ samples from the semi-random model $\sModel$, Algorithm \textsc{RecoverDict} with probability at least $1 - \frac{1}{m}$, outputs a set $W^*$ such that
\begin{itemize}
\item For each column $A_i$ of $A$, there exists $\hat{A}_i \in W^*$ such that $\|A_i - b\hat{A}_i\| \leq \epsilon$.
\item For each $\hat{A}_i \in W^*$, there exists a column $A_i$ of $A$ such that $\|\hat{A}_i - b{A}_i\| \leq \epsilon$,
\end{itemize}
provided $k \leq \sqrt{n}/\nu_1(\tfrac{1}{m}, 16)$.
Here $\nu_1(\eta, d) := c_1 \tau \mu^2\bigparen{C(\sigma^2 + \mu \sqrt{\frac m n}) \log^2(n/\eta)}^{d}$, 
$c_1>0$ is a constant (potentially depending on $C$), and the polynomial bound for $N$ also hides a dependence on $C,L$.
\end{theorem}

\begin{proof}[Proof of Theorem~\ref{thm:main-full-algorithm}]
Since $W^*$ is empty initially, from the guarantee of Theorem~\ref{thm:single-recovery-main} we have that in the first step of the \textsc{RecoverDict} an $\epsilon$-close vector to at least one column of $A$ will be added to $W^*$ except with probability at most $\frac{1}{m^2}$. Next assume that we have recovered $m' < m$ columns of $A$ to good accuracy. If we are given $|T|$ samples from $\sModel$ we know that at least $\beta |T|$ belong to the random portion. In this portion the expected marginal $q_i$ of each column $A_i$ is $\frac{k}{m}$. Hence, by Chernoff bound the marginal of each column in the $\beta |T|$ samples will be at most $\frac{k}{m}(1+\tol)$ except with probability $m e^{-\log^2 m}$. Hence the linear program involving $w_j$s has a feasible solution that puts a weight $1$ on all the random samples and weight $0$ on all the additional semi-random samples. Let $w_1, w_2, \dots, w_{|T|}$ be the solution output by the linear program. Define the corresponding support distribution induced as $\hat{q}$, i.e., for any $I \subseteq [m]$, $\hat{q}_I = \frac{\sum_{j \in V(I)} w_j}{\sum_j w_j}$, where $V(I) = \cap_{r \in I} V(A_r)$. Denote by $\hat{q}_j$ the induced marginal on column $A_j$. Then we have that $\sum_{j \in [m]} \hat{q}_j = k$. Furthermore, we also have that for the $m'$ columns in $W^*$ the sum of the corresponding $q_j$ is at most $\frac{m' k}{m}(1+\tol)$. Hence we get that there must be an uncovered column $j^*$ such that 
\begin{align*}
\hat{q}_{j^*} &\geq \frac{k - \frac{m' k}{m}(1+\tol)}{m-m'}
\geq \frac{k}{m}(1+\tol) - k \tol, \quad \quad (\text{ since } m' \le m-1)\\
&\geq \frac{k}{m}(1+\tol)\Paren{1-\frac{1}{\log^2 m}}
\end{align*}
\end{proof}
Hence when we feed the set $T_1$ into the \textsc{RecoverColumns} procedure, by Theorem~\ref{thm:single-recovery-main} an $\epsilon$-close approximation to a new column will be added to $W^*$ except with probability at most $\frac{1}{m^2}$. Notice that for the guarantee of Theorem~\ref{thm:single-recovery-main} to hold it is crucial that the non-zero values of $x$ in the samples $y=Ax$ present in set $T_1$ are drawn independently from $\Dsr$. However, this is true since from Lemma~\ref{lem:find-support-deterministic} we only use the support of the samples to do the re-weighting\footnote{When $k$ exceeds $\sqrt{n}$ this will not be true and we cannot deterministically determine the correct supports for each sample.}. Additionally since $|T| \gg |T_1|$ no sample in $T$ will be repeated more than once in $T_1$, and hence we can assume that when used in procedure \textsc{RecoverColumns}, the values are picked independently from $\Dsr$ for each sample in $T_1$ conditioned on the support. 

To see why no sample will be repeated more than once with high probability, let $p_j = \Pr[\text{ sample } \ysamp{j} \text{ is } chosen]$. Then we have since $\sum_j w_j \geq \beta |T|$ that $p_j \le 1/(\beta |T|)$. Hence, we get that the probability that sample $\ysamp{j}$ is repeated more than once in $T_1$ is at most
\anote{Modified this}
\begin{align*}
\Pr\big[ \ysamp{j} \text{ repeated more than once } \big] &\le {|T_1| \choose 2} p_j^2 \le \frac{|T_1|^2}{\beta^2 |T|^2}, \text{and}\\
\Pr\big[\text{no sample is repeated} \big] & \le {|T_1| \choose 2} \sum_{j \in [T]} p_j^2 \le \frac{|T_1|^2}{\beta^2 |T|} \le \frac{1}{m^2}
\end{align*}
as required, since the number of samples $|T|$ is chosen to be sufficiently large. 
\anote{May need to update the number of samples..}
\begin{lemma}
\label{lem:find-support-deterministic}
Let A be a $\mu$-incoherent matrix and let the set $W^*$ contain unit length vectors that are $\epsilon$-close approximations to a subset of the columns of $A$. Given a support set $I \subseteq [m]$ such that $|I| \le  k$ and $y = \sum_{i \in I} \alpha_i {A}_i$ where $\abs{\alpha_i} \in [1,C]$, we have that
\begin{itemize}
\item For each $\hat{A}_i \in W^*$ such that $i \in I$, $|\iprod{y,\hat{A}_i}| \geq \frac 1 2$.
\item For each $\hat{A}_i \in W^*$ such that $i \notin I$, $|\iprod{y,\hat{A}_i}| < \frac 1 2$.
\end{itemize}
provided $k \leq \frac{\sqrt{n}}{8C\mu}$ and $\epsilon \leq \frac{1}{8Ck}$. 
\end{lemma}
We will use this lemma with samples $y=Ax$. Observe that this is a deterministic statement that does not depend on the values of the non-zeros in the sample $y=Ax$, and only depends on the support of $x$.
\begin{proof}
Notice that it is enough to show that $|\iprod{y,\hat{A}_i}-\alpha_i| \leq \frac 1 4$ for each $i \in [m]$ since $|\alpha_i| \geq 1$ if $i \in I$ and $0$ otherwise. Given $i \in [m]$ we have
\begin{align*}
\iprod{y,\hat{A}_i} &= \alpha_i \iprod{A_i, \hat{A}_i} +\sum_{j \in I\setminus \{i\}} \alpha_j \iprod{A_j, \hat{A}_i}\\
&= \alpha_i \left(\iprod{A_i, A_i} + \iprod{A_i, \hat{A}_i-A_i} \right) +  \sum_{j \in I\setminus \{i\}}  \alpha_j \left(\iprod{A_j, A_i} + \iprod{A_j,\hat{A}_i-A_i} \right)\\
&= \alpha_i + \alpha_i \iprod{A_i, \hat{A}_i-A_i} + \sum_{j \in I\setminus \{i\}}  \alpha_j \left(\iprod{A_j, A_i} + \iprod{A_j,\hat{A}_i-A_i} \right)
\end{align*}
Using the fact that $\|\hat{A}_i - A_i\| \leq \epsilon$ we get that 
$
\abs{\alpha_i \iprod{A_i, \hat{A}_i-A_i}} \leq C \epsilon,
$
and similarly $\abs{\alpha_j \iprod{A_j, \hat{A}_i-A_i}| \leq C \epsilon}$ for each $j \in I\setminus \set{i}$.
Finally, using the fact that $A$ is $\mu$-incoherent we have
$$
|\alpha_j \iprod{A_i,A_j}| \leq \frac{C \mu}{\sqrt{n}}.
$$
Hence for our choice of $k$ and $\epsilon$,
$$
|\iprod{y,\hat{A}_i} - \alpha_i| \leq kC\epsilon + \frac{C k \mu}{\sqrt{n}} \leq \frac 1 4.
$$

\end{proof}

\section{Efficient algorithms for the random model: Beyond $\sqrt{n}$ sparsity}
\label{sec:random}
In this section we show that when the data is generated from the standard random model $\Dsr \odot \Dv$ our approach from Section~\ref{sec:semi-random-recovery} leads to an algorithm that can handle sparsity up to $\widetilde{O}(n^{2/3})$  which improves upon the state-of-art results in certain regimes, as described in Section~\ref{sec:related}. 
\anote{4/22:Removed comparisons to old work, since it is somewhat complicated.}
As in the semi-random case, we will look at the statistic $\E[\iprod{u^{(1)},y}\iprod{u^{(2)},y}\iprod{u^{(3)},y}\dots \iprod{u^{(2L-1)},y}~ y]$ for a constant $L \geq 8$. Here $u^{(1)}, u^{(2)}, \dots, u^{(2L-1)}$ are samples that all have a particular column, say $A_i$, in their support such that $A_i$ appears with the same sign in each sample. Unlike in the semi-random case where one was only able to recover high frequency columns, here we will show that then one can good approximation to any columns $A_i$ via this approach. Hence, in this case we do not need to iteratively re-weigh the data to recover more columns. This is due to the fact that in the random case, given a sample $y=Ax$, we have that $P(x_i \neq 0) = \frac{k}{m}$. Hence, all columns are large frequency columns. Furthermore, when analyzing various sums of polynomials over the $\zeta$ random variables as in Section~\ref{sec:semi-random-recovery} we will be able to use better concentration bounds using the fact that the support distribution $\Dsr$ satisfies \eqref{eq:support-assumption-1} and using the corresponding consequences from Lemma~\ref{lem:q-sum} and Lemma~\ref{lem:q-d-sum}. The main theorem of this section stated below claims that the \textsc{RecoverColumns} procedure in Figure~\ref{ALG:Single_Recovery} will output good approximations to all columns of $A$ when fed with data from the random model $\Dsr \odot \Dv$.
\anote{4/22: Added every quantifier for $c, \epsilon$.}
\begin{theorem}
\label{thm:random-recovery-main}
There exists constants $c_1>0$ (potentially depending on $C$) and $c_2>0$ such that the following holds for any $\epsilon>0$, any constants $c>0$, $L \ge 8$. Let $A_{n \times m}$ be a $\mu$-incoherent matrix with spectral norm at most $\sigma$ that satisfies $(k,\delta)$-RIP for $\delta < 1/(C^2 \log^{c_2} n)$.
Given 
$\poly(k,m,n,1/\eps)$ samples from the random model $\Dsr \odot \Dv$, Algorithm \textsc{RecoverColumns}, with probability at least $1-\frac{1}{m^c}$, outputs a set $W$ such that
\begin{itemize}
\item For each $i \in [m]$, $W$ contains a vector $\hat{A}_i$, and there exists $b \in \set{\pm 1}$ such that $\|A_i - b\hat{A}_i\| \leq \epsilon$. 
\item For each vector $\hat{z} \in W$, there exists $A_i$ and $b \in \set{\pm 1}$ such that $\|\hat{z} - bA_i\| \leq \epsilon$,
\end{itemize}
provided 
$k \leq {n^{2/3}}/(\nu(\frac 1 m, 2L) \tau \mu^2).
$ Here 
$\nu(\eta,d):=c_1 \bigparen{C(\sigma^2 + \mu \sqrt{\frac m n}) \log^2(n/\eta)}^{d}$, and the polynomial bound also hides a dependence on $C$ and $L$.
\end{theorem}

Here, we use $\Dsr \odot \Dv$ as the first argument to the \textsc{RecoverColumns} procedure and it should be viewed as a model $\mathcal{M}_{\beta}(\Dsr, \Dsr, \Dv)$ with $\beta=1$. Again the bound above is strongest when $m = O(n), \sigma=O(1)$ in which case we get
$
k \leq \widetilde{O}(n^{2/3}),
$
However, as in the semirandom case, we can handle $m = n^{1+\epsilon_0}$ for a sufficiently small constant $\epsilon_0 > 0$ with a weaker dependence on the sparsity.
The main technical result of this section is the following analogue of Theorem~\ref{thm:single-recovery-main} from Section~\ref{sec:semi-random-recovery}.
\anote{5/5:We don't need RIP for the following theorem.}
\begin{theorem}
\label{thm:random-high-order}
The following holds for any constants $c \geq 2$ and $L \geq 8$. Let $A_{n \times m}$ be a $\mu$-incoherent matrix with spectral norm at most $\sigma$. 
Let $u^{(1)}=A\upzeta{1},u^{(2)}=A\upzeta{2}, \dots, u^{(2L-1)}=A\upzeta{2L-1}$ be samples drawn from $\Dsr \odot \Dv$ conditioned on $\upzeta{t}(1) > 0$ for $t \in [2L-1]$. Let $y = \sum_{i \in [m]} x_i A_i$ be a random vector drawn from $\Dsr \odot \Dv$. With probability at least $1-\frac{1}{\log^2 m}$ over the choice of $\upzeta{1}, \ldots, \upzeta{2L-1}$ we have that
$$
\E_{x,v}[\iprod{u^{(1)},y}\iprod{u^{(2)},y}\ldots \iprod{u^{(2L-1)},y}y] = q_1A_{1} + e_{1}
$$
where $\|e_{1}\|=O \Paren{\frac {q_1}{\log^c m}}$, and
$
k \leq {{n^{2/3}}}/{\nu(\frac 1 m, 2L) \tau \mu^2}.
$
\anote{4/22: Changed $nu$ definition to be consistent with earlier.}
Here $\nu(\eta, d) = c_1 \left(C(\sigma^2 + \mu \sqrt{\frac m n}) \log^2(n/\eta) \right)^{d}$ for a constant $c_1>0$,
and the expectation is over the value distribution (non-zero values) of the samples $x$. 
\end{theorem}
Before we proceed, we will first prove Theorem~\ref{thm:random-recovery-main} assuming the proof of the above theorem. Unlike the semirandom model, in the random model $\Ds=\Ds_R$ is fixed here; so we do not need to perform a union bound over possible semirandom support distributions as in Section~\ref{sec:semi-random-recovery}. 
\anote{4/22: modified the proof mildly to match the proof in previous section (particularly the $c$ dependence etc.) }
\begin{proof}[Proof of Theorem~\ref{thm:random-recovery-main}]
The proof is similar to the proof of Theorem~\ref{thm:single-recovery-main}. Let $T_0$ be the set of samples drawn in step $2$ of the \textsc{RecoverColumns} procedure. Let $c_*>0$ be an absolute constant (it will be chosen later based on Theorem~\ref{thm:sr:testing}).
From Lemma~\ref{lem:statistic-conc} we can assume that for each $2L-1$ tuple in $T_0$, the statistic in \eqref{eq:statistic-main} is $\frac{k}{m\log^{4c_*} m}$-close to its expectation, except with probability at most $2m{|T_0|}^{2L-1}\exp\paren{- \log^2 m}$. \anote{4/22:modified this to include the $k/m$ factor in closeness. } Let $A_i$ be a column of $A$. From Lemma~\ref{lem:many-zeta-conc} we have that, except with probability at most $m\exp\bigparen{-(2L-1)\log m/4}$, there exist at least $\log m$ disjoint $(2L-1)$-tuples in $T_0$ that intersect in $A_i$ with a positive sign. Hence we get from Theorem~\ref{thm:random-high-order} that there is at least $1 - (\log m)^{-2\log m}$ probability that the vector $\hat{v}$ computed in step 4 of the algorithm will be $\frac{1}{\log^{2c_*} m}$-close to $A_i$. Then we get from Theorem~\ref{thm:sr:testing} (for an appropriate $c_*>0$) that a vector that is $\epsilon$-close to $A_i$ will be added to $W$ except with probability at most $m^{-4L}$. Furthermore no spurious vector that is $\frac{1}{\log^{c_*} m}$ far from all $A_i$ will be added, except with probability at most $(|T_0|/m^2)^{2L}$. Hence the total probability of failure of the algorithm is at most $m(2{|T_0|}^{(2L-1)}e^{\log^2 m} + \exp\paren{-\frac{(2L-1)\log m}{4}} + \frac{1}{(\log m)^{2 \log m}} + \frac{|T_0|^{2L}}{m^{4L}}) \leq \frac{1}{m^c}$.
\end{proof}

\subsection{Proof of Theorem~\ref{thm:random-high-order}}
The proof will be identical to that of Theorem~\ref{thm:main-semi-random-high-order}. However, since the support distribution $\Dsr$ satisfies \eqref{eq:support-assumption-1}, we will be able to use the additional consequences of Lemma~\ref{lem:q-sum} and Lemma~\ref{lem:q-d-sum} to get much better concentration bounds for various random sums involved. This will lead to an improved sparsity tolerance of $\approx n^{2/3}$.

Let $y = \sum_{i \in [m]} x_i A_i$. Then we have that
\begin{align}
&\E_{x,v}[\iprod{u^{(1)},y}\iprod{u^{(2)},y}\ldots \iprod{u^{(2L-1)},y}y] = \sum_{i \in [m]} \gamma_i A_i,~~~ \text{where} \nonumber\\
&\gamma_i = \sum_{j_1,\dots,j_{2L-1} \in [m]} \upzeta{1}_{j_1} \ldots \upzeta{2L-1}_{j_{2L-1}} \sum_{i_1,\ldots i_{2L-1} \in [m]} \E[x_{i_1} \ldots x_{i_{2L-1}} x_i] M_{i_1,j_1} \dots M_{i_{2L-1},j_{2L-1}}
\label{eq:gamma-i-random}
\end{align}
We will show that with high probability~(over the $\zeta$s), $\gamma_1 = q_1(1 \pm \frac{1}{\log^c m})$ and that $\|\sum_{i \neq 1} \gamma_i A_i\| = O \Paren{\frac{q_1}{\log^c m}}$ for our choice of $k$.
Notice that for a given $i$, any term in the expression for $\gamma_i$ as in \eqref{eq:gamma-i-semi-random} will survive only if the indices $i_1, i_2, \dots, i_{2L-1}$ form a partition $S = (S_1, S_2, \ldots S_R)$ such that $|S_1|$ is odd and $|S_p|$ is even for $p \geq 2$. $S_1$ is the special partition that must correspond to indices that equal $i$. Hence, $(S_1, S_2, \ldots S_R)$ must satisfy $i_t=i$ for $t \in S_1$ and $i_t = i^*_{r}$ for $t \in S_r$ for $r \geq 2$, for indices $i^*_2, \dots i^*_{R} \in [m]$. We call such a partition a valid partition and denote $|S|=R$ as the size of the partition. Let $d_1, d_2, \dots d_R$ denote the sizes of the corresponding sets in the partition, i.e., $d_j = |S_j|$. Notice that $d_1 \geq 1$ must be odd and any other $d_j$ must be an even integer. Using the notation from Section~\ref{subsec:model} we have that 
$$
\E\big[x_{i_1} \dots x_{i_{2L-1}}x_i\big] = \begin{cases} q_{i,i^*_2,i^*_3,\dots,i^*_{R}}(d_1+1, d_2, \dots, d_R), & S \text{ is valid}\\
0, & \textrm{otherwise}
\end{cases}
$$
Recall that by choice, $\upzeta{\ell}_1 \ge 1$ for each $\ell \in [2L-1]$. Hence, the 
value of the inner summation in \eqref{eq:gamma-i-random} will depend on how many of the indices $j_1, j_2, \dots, j_{2L-1}$ are equal to $1$. This is because we have that $\zeta^{\ell}_1$ is a constant in $[1,C]$ for all $\ell \in [2L-1]$. Hence, let $H = (H_1, H_2, \dots H_R)$ be such that $H_r \subseteq S_r$ and for each $r$ we have that $j_t=1$ for $t \in H_r$. Let $h$ denote the total number of fixed variables, i.e. $h = \sum_{r \in [R]} |H_r|$. Notice that $h$ ranges from $0$ to $2L-1$ and there are $2^{2L-1}$ possible partitions $H$. The total number of valid $(S,H)$ partitionings is at most $(4L)^{2L}$. Hence 
\begin{align}
&\gamma_{i} = \sum_{(S,H)} \gamma_i(S,H), ~~\text{where}~~ \gamma_i(S,H):= \nonumber\\
&= \sum_{\substack{(i^*_2, \dots, i^*_R) \\\in [m]^{R-1}}} q_{i,i^*_2,\dots,i^*_{R}}(d_1+1, d_2, \dots, d_R) \prod_{r \in [R]} \sum_{\substack{J_{S_r\setminus H_r}}} (M_{i^*_r,1})^{|H_r|} \prod_{t \in H_r} \upzeta{t}_{1} \cdot \prod_{t \in S_r \setminus H_r}M_{i^*_r,j_t} \upzeta{t}_{j_t} \nonumber
\end{align}
Note that by triangle inequality, $\norm{\sum_{i \ne 1} \gamma_i A_i}_2 \le \sum_{(S,H)} \norm{ \sum_{i \ne 1} \gamma_i(S,H) A_i}_2$. Hence, we will obtain bounds on $\gamma_i(S,H)$ depending on the type of partition $(S,H)$, and upper bound $\norm{ \sum_{i \ne 1} \gamma_i(S,H) A_i}_2$ for each $(S,H)$. We have three cases depending on the number of fixed indices $h$. 
\anote{Explain what $J_S$ means.}

\noindent \textbf{Case 1: $h=0$.} In this case none of the random variables are fixed to $1$. In this case we use the second consequence of Lemma~\ref{lem:frobnorm:bound} to get that with probability at least $1-\eta$, over the randomness in $\zeta_j$s,
\begin{align}
|\gamma_i(S,H=\emptyset)| 
&\leq {q_i \eta^{-1/2} \cdot \nu(\eta,2L) \sigma^{4L} \cdot \Bigparen{\frac{\tau^{4/3} k^{4/3}}{m}}^{\tfrac{3L}{2}-1}}
\label{eq:gammai-random-F-empty}
\end{align}
In the final analysis we will set $\eta = \left(m \log^2 m (4L)^{2L} \right)^{-1}$.  In this case we get that
\begin{align*}
|\gamma_i(S,H=\emptyset)| 
&\leq {q_i \sqrt{m} \log m \cdot (4L)^{2L}  \nu(\eta,2L) \sigma^{4L} \cdot\Bigparen{\frac{\tau^{4/3} k^{4/3}}{m}}^{\tfrac{3L}{2}-1}}.
\end{align*}
We will set $L$ large enough in the above equation such that we get 
\begin{align}
|\gamma_i(S,H=\emptyset)| 
&\leq q_i \nu(\eta,2L)  \frac{\mu^2 k}{m \sqrt{m}}.
\label{eq:gammai-F-empty-random}
\end{align}
$L \geq 8$ suffice for this purpose for our choice of $k$.

\noindent \textbf{Case 2: $h=2L-1$.} In this case all the random variables are fixed and $\gamma_i$ deterministically equals
\begin{align}
|\gamma_i(S,H)| &= \zeta^{(1)}_1 \zeta^{(2)}_1 \dots \zeta^{(2L-1)}_1 \sum_{i^*_2, i^*_3, \dots, i^*_R \in [m]} q_{i,i^*_2,i^*_3,\dots,i^*_{R}}(d_1+1, d_2, \dots, d_R) M_{i,1}^{d_1} M_{i^*_2,1}^{d_2} \dots M_{i^*_R,1}^{d_R} \nonumber \\
&= \begin{cases}
\zeta^{(1)}_1 \zeta^{(2)}_1 \dots \zeta^{(2L-1)}_1 \cdot q_i M_{i,1}^{d_1}, & R=1\\
 O \left(C^{4L-1}\sigma^{2(R-1)} \cdot q_i  M_{i,1}^{d_1} \frac{k\tau}{m}  \right), & \text{ otherwise },
\end{cases}
\label{eq:random-F-full}
\end{align}
where the case when $R \ne 1$ follows from Lemma~\ref{lem:semi-random-F-full}. As opposed to the similar case in the semi-random scenario~\eqref{eq:semi-random-F-full}, here we get an additional factor of $\frac{k \tau}{m}$ for $R \neq 1$, since we use the stronger fact that $q_{i,i^*_2,i^*_3,\dots,i^*_{R}}(d_1+1, d_2, \dots, d_R)$ satisfies the stronger conditions of Lemma~\ref{lem:q-sum} and Lemma~\ref{lem:q-d-sum}.

\noindent \textbf{Case 3: $1 \leq h< 2L-1$.} 
Similar to the semi-random case, we can write 
\begin{align*}
\gamma_i(S,H) &= q_i(d_1+1) \prod_{t' \in H_1} \zeta^{(t')}_1 \sum_{\substack{J_{S_1 \setminus H_1}}} M^{|H_1|}_{i,1} \prod_{t \in S_1 \setminus H_1} \upzeta{t}_{j_t} M_{i,j_t}  F_{i},~~~\text{where} \nonumber\\ 
 F_{i} &= \sum_{i^*_{2}, \dots, i^*_{R} \in [m]} \Big(\frac{q_{i, \dots,i^*_r,\dots, i^*_R}(d_1+1, d_2, \dots, d_R)}{q_{i}(d_1+1)}\Big) \prod_{r=2}^R \Bigparen{\sum_{\substack{J_{S_r \setminus H_r} }} M^{|H_p|}_{i^*_p,1} \prod_{t \in S_p \setminus J_p}  M_{i^*_{p},j_t} \upzeta{t}_{j_t} }.
\end{align*}
When $|H_1| \geq 1$, we can apply Lemma~\ref{lem:semi-random-inductive} with $r=1$ we get with probability at least $1-\eta$ over the randomness in $\zeta_j$s that
$|F_i| \leq (\frac{k \tau}{m})^{R-1}\nu(\eta,2L)$. Here again the extra $(\frac{k \tau}{m})^{R-1}$ is due to the fact that we have a better bound of $\frac{k \tau}{m}$ on $\|w'\|_{\infty}$ in the application of the lemma. 
Hence, we get that with probability at least $1-\eta$ over the randomness in $\zeta_j$s,
\begin{align}
\gamma_i(S,H) = w_i q_i(d_1+1) \prod_{t' \in H_1} \zeta^{(t')}_1  \sum_{\substack{J_{S_1 \setminus H_1} }} \prod_{t \in S_1 \setminus H_1} \upzeta{t} M_{i,j_t} M^{|H_1|}_{i,1}
\label{eq:random-gammai-simplified}
\end{align}
where $|w_i| \leq (\frac{k \tau}{m})^{R-1} \nu(\eta,2L)$. 

Unlike the semi-random case we will bound the expression above in two different ways depending on the value of $R$. This careful analysis of the expression above will help us go beyond the $\sqrt{n}$ bound. When $R \geq 2$, we have $|w_i| \leq \frac{k \tau}{m}\nu(\eta,2L)$ and we use Lemma~\ref{lem:conc-special-partition} to get that with probability at least $1-2\eta$, the above sum is bounded as
\begin{align}
\abs{\gamma_i(S,H)} \leq \begin{cases}
\nu(\eta,2L)\frac{q_i k \tau \mu}{m\sqrt{n}} \bigparen{ Z_i + \nu(\eta,2L) \sqrt{\frac{k\tau}{m}}}, & i\neq 1, R \geq 2\\
\nu(\eta,2L)\frac{q_i k \tau}{m} \cdot \nu(\eta,2L) \sqrt{\frac{k\tau}{m}}, & i = 1, R \geq 2
\end{cases}
\label{eq:random-gammai-final-bound}
\end{align}
Here $Z_i=\prod_{t \in S_1 \setminus H_1} \abs{\upzeta{t}_{i}}$ are non-negative random variables bounded by $C^{|S_1 \setminus H_1|}$. Further $Z_i$ are each non-zero with probability at most $p\cdot (\tau p)^{|S_1 \setminus H_1|-1}$ and they are $\tau$-negatively correlated
 (with the values conditioned on non-zeros being drawn independently). The additional $\frac{\mu}{\sqrt{n}}$ factor in the case of $i \neq 1$ is due to the fact that $|H_1| \geq 1$ and we have $M^{|H_1|}_{i,1}$ in the expansion.
\anote{What is $\tau$-negatively correlated? Mention in prelims?}

When $R=1$ we have $F_i=1$ and hence $w_i=1$. Here we will directly use Lemma~\ref{lem:conc-special-partition}. However the application of the Lemma will depend on whether $|H_1| \geq 2L-4$ or not. If $|H_1| \geq 2L-4$ then we use the bound that holds for degree $d \leq 3$ and otherwise we use the better bound. Hence, we have
\begin{align}
\abs{\gamma_i(S,H)} \leq \begin{cases}
q_i (\frac{\mu}{\sqrt{n}})^{2L-4} \bigparen{ Z_i + \nu(\eta,2L) \sqrt{\frac{k \tau}{m}}}, & i\neq 1, |H_1| \geq 2L-4\\
\frac{q_i \mu}{\sqrt{n}} \nu(\eta,2L) \sqrt{\frac{k \tau}{m}}, & i \neq 1, |H_1| \leq 2L-4\\
{q_i} \cdot \nu(\eta,2L) \sqrt{\frac{k \tau}{m}}, & i = 1
\end{cases}
\label{eq:random-gammai-final-bound-R1}
\end{align}

Next we look at the case when $|H_1|=0$. Hence, there must exist $r \geq 2$ such that $|H_r| \geq 1$. Without loss of generality assume that $|H_2| \geq 1$. Then we can write $\gamma_i(S,H)$ as
\begin{align*}
\gamma_i(S,H)&=\sum_{i^*_2} q_{i,i^*_2}(d_1+1,d_2)  \sum_{\substack{J_{S_1}}} \prod_{t \in S_1} M_{i,j_t} \upzeta{t}_{j_t}  \cdot \prod_{t' \in H_2} \upzeta{t'}_1 \sum_{\substack{J_{S_2 \setminus H_2} }}M^{|H_2|}_{i^*_2,1} \prod_{t \in S_2 \setminus H_2} \upzeta{t}_{j_t} M_{i^*_2,j_t}  F'_{i,i^*_2},\\
\text{where }& F'_{i,i^*_2}=\sum_{i^*_{3}, \dots, i^*_{R} \in [m]} \Big(\frac{q_{i, \dots,i^*_r,\dots, i^*_R}(d_1+1, d_2, \dots, d_R)}{q_{i,i^*_2}(d_1+1, d_2)}\Big) \prod_{r=3}^R \left(\sum_{J_{S_r \setminus H_r} } M^{|H_r|}_{i^*_r,1} \prod_{t \in S_r \setminus J_r}  M_{i^*_{r},j_t} \upzeta{t}_{j_t} \right).
\end{align*}
We can again apply Lemma~\ref{lem:semi-random-inductive} with $r=2$ to get that with probability at least $1-\eta$ over the randomness in $\zeta_j$s,
$$
\abs{F'_{i,i^*_2}} \leq \nu(\eta,2L-d_1-d_2-1) (\frac{k \tau}{m})^{R-2}.
$$
Hence, we can rearrange and write $\gamma_i(S,H)$ as
\begin{align}
\gamma_i(S,H) &= q_{i}(d_1+1) \prod_{t' \in H_2} \upzeta{t'}_1  \sum_{J_{S_1}} \prod_{t \in S_1} \upzeta{t}_{j_t} M_{i,j_t} F''_{i}, ~~\text{ where}\nonumber
\\
F''_{i} &= \sum_{i^*_2 \in [m]} \frac{q_{i,i^*_2}(d_1+1,d_2)}{q_i(d_1+1)} \cdot F'_{i,i^*_2} M^{|H_2|}_{i^*_2,1}\sum_{J_{S_2 \setminus H_2}} \prod_{t \in S_2 \setminus H_2} \upzeta{t}_{j_t} M_{i^*_2,j_t}  \label{eq:gammai-random-simplified-h1-zero} 
\end{align}
We split this sum into two, depending on whether $i^*_2=1$ or not. 
Here we have that
\begin{align}
F''_{i,a} := \frac{q_{i,1}(d_1+1,d_2)}{q_i(d_1+1)}\cdot F'_{i,i^*_2=1}\sum_{J_{S_2 \setminus H_2}} \prod_{t \in S_2 \setminus H_2} \upzeta{t}_{j_t} M_{1,j_t} 
\label{eq:F1-random}
\end{align}
and
\begin{align}
F''_{i,b} :=  \sum_{i^*_2 \in [m]\setminus \{1\}} \frac{q_{i,i^*_2}(d_1+1,d_2)}{q_i(d_1+1)} \cdot F'_{i,i^*_2} M^{|H_2|}_{i^*_2,1}  \sum_{J_{S_2 \setminus H_2} } \prod_{t \in S_2 \setminus H_2} \upzeta{t}_{j_t} M_{i^*_2,j_t} 
\label{eq:F2-random}
\end{align}
Here when $|H_2| < |S_2|$ we will use Lemma~\ref{lem:conc-special-partition} and the fact that $j_t \neq 1$ for $t \in |S_2 \setminus H_2|$ to get that with probability at least $1-\eta$ over the randomness in $\zeta_j$s, we can bound $F''_{i,a}$ as
\begin{align*}
|F''_{i,a}| &\leq \nu(\eta,2L-d_1-1-|H_2|)\cdot \frac{q_{i,1}(d_1+1,d_2)}{q_i(d_1+1)}  ({\frac{k \tau}{m}})^{R-3/2}\\
&\leq \nu(\eta,2L-d_1-1-|H_2|)\cdot C^{d_2} (\frac{k \tau}{m})^{R-1/2} 
\end{align*}
where in the last inequality we have used the stronger consequence of Lemma~\ref{lem:q-d-sum}. 
Combining this with the simple bound when $S_2=H_2$ we get 
\begin{align}
|F''_{i,a}| &= \begin{cases}
\nu(\eta,2L-d_1-d_2-1) \cdot C^{d_2} (\frac{k \tau}{m})^{R-1}, & |H_2| = |S_2| \\
\nu(\eta,2L-d_1-1-|H_2|)\cdot C^{d_2} (\frac{k \tau}{m})^{R-1/2}, & \text{ otherwise }
\end{cases} 
\label{eq:bound-F1-random}
\end{align}
Next we bound $F''_{i,b}$. When $|H_2| < |S_2|$ we will use the concentration bound from Lemma~\ref{lem:conc-one-partition}. However, when applying Lemma~\ref{lem:conc-one-partition} we will use the fact that $|w_{i^*_2}| = |\frac{q_{i,i^*_2}(d_1+1,d2)}{q_i(d_1+1)}F_{i,i^*_2}M^{|H_2|}_{i^*_2,1}| \leq C^{d_2} \cdot k\tau/m \cdot \nu(\eta,2L-d_1-d_2-1) \mu/\sqrt{n}$. This is because of the stronger consequence of Lemma~\ref{lem:q-d-sum} and the fact we are summing over $i^*_2 \neq 1$. Furthermore we are in the case when $|H_2| \geq 1$. Hence, by incoherence we have that $|M^{|H_2|}_{i^*_2,1}| \leq \mu/\sqrt{n}$. Hence we get that with probability at least $1-\eta$ over the randomness in $\set{\zeta_j}$
to get that
$$
|F''_{i,b}| \leq C^{d_2} \nu(\eta,2L-d_1-1-|H_2|)\frac{k\mu \tau}{m\sqrt{n}}
$$
When $|H_2|=|S_2|$ we get 
$$
|F''_{i,b}| \leq \sum_{i^*_2 \in [m]\setminus \{1\}} \frac{q_{i,i^*_2}(d_1+1,d_2)}{q_i(d_1+1)} \cdot \bigabs{F'_{i,i^*_2} M^{|S_2|}_{i^*_2,1}}.
$$ 
Using the fact that $|S_2| \geq 2$, $i^*_2 \neq 1$ and that the columns are incoherent we get,
\begin{align*}
|F''_{i,b}| & \leq \nu(\eta,2L-d_1-d_2-1) \cdot \frac{\mu^2}{n} \cdot \sum_{i^*_2 \in [m]\setminus \{1\}} \frac{q_{i,i^*_2}(d_1+1,d_2)}{q_i(d_1+1)} \\
&\leq \nu(\eta,2L-d_1-d_2-1)\cdot \frac{C^{d_2}\mu^2 k}{n}
\end{align*} 
where in the last inequality we use Lemma~\ref{eq:qd-bound}.
Combining the above bounds we get that with probability least $1-\eta$ over the randomness in $\zeta_j$s, 
\begin{align}
F''_{i,b} &= \begin{cases}
\nu(\eta,2L-d_1-d_2-1)\frac{C^{d_2}\mu^2 k}{n}, & |H_2| = |S_2| \\
C^{d_2} \nu(\eta,2L-d_1-1-|H_2|)\frac{k\mu \tau}{m\sqrt{n}}, & \text{ otherwise }
\end{cases} 
\label{eq:bound-F2-random}
\end{align}
We combine the above bounds on $F''_{i,a}$ and $F''_{i,b}$ and to get the following bound on $F''_i$ that holds with probability $1-2\eta$ over the randomness in $\zeta$s
\begin{align}
|F''_i| &\leq \begin{cases}
\nu(\eta,2L-d_1-d_2-1) \left( C^{d_2}(\frac{k\tau}{m})^{R-1} + \frac{C^{d_2}\mu^2 k}{n} \right), & |H_2| = |S_2|\\
C^{d_2} \nu(\eta,2L-d_1-1-|H_2|) \left( (\frac{k\tau}{m})^{R-1/2} + \frac{k\mu \tau}{m\sqrt{n}} \right), & \text{ otherwise }
\end{cases}
\end{align}
Finally, we get a bound on $\gamma_i(S,H)$. Here unlike the semi-random case we use Lemma~\ref{lem:conc-one-partition} with $w_i$ in the Lemma set to $q_i(d_1+1) \prod_{t' \in H_2} \zeta^{(t')}_1 F''_i$. This is because we have a good upper bound on $|w_i|$ of $q_i C^{d_1+1} C^{|H_2|} |F''_i|$. Hence we get that
\begin{align}
|\gamma_i(S,H)|  
& \leq \begin{cases}
q_1 C^{2L} \nu(\eta,2L) ({\frac{\tau k}{m}})^{3/2}, & i=1\\
q_i C^{2L} \nu(\eta,2L) (\sqrt{\frac{k \tau}{m}})(\frac{k\tau}{m})^{3/2}, & \text{ otherwise }
\end{cases}
\label{eq:gamma-i-H-zero-final-random}
\end{align}

\paragraph{Putting it Together.} We will set $\eta = \left({m \log^2 m (4L)^{2L}} \right)^{-1}$ so that all the above bounds hold simultaneously for each $i \in [m]$ and each partitioning $S,H$. We first gather the coefficient of $A_1$, i.e., $\gamma_1$. For the case of $h=2L-1$ we get that $\gamma_1(S,H) \geq q_1$ from \eqref{eq:random-F-full}. Here we have used the fact that $\zeta^{(t)}_{1} \geq 1$ for all $t \in [2L-1]$. For any other partition we get from \eqref{eq:random-gammai-final-bound}, \eqref{eq:gamma-i-H-zero-final-random} and \eqref{eq:gammai-F-empty-random} that 
\anote{Put in a $\tau$ term.}
$$
\gamma_1 \leq q_1 C^{2L}\nu(\eta,2L)\cdot \sqrt{\frac{\tau k}{m}} = O\Bigparen{\frac{q_1}{(4L)^{2L} \log^{c} m} }
$$
for our choice of $k$. Hence, summing over all partitions we get that term corresponding to $A_1$ in \eqref{eq:gamma-i-semi-random} equals $a_1 A_1 + e_1$ where $a_1 \geq q_1$ and $\|e_1\| = O(\frac{q_1}{\log^c m})$. 

Next we bound $\|\sum_{i \neq 1} \gamma_i A_i\|$. 
In order to show that $\|\sum_{i \neq 1} \gamma_i A_i\| \leq \frac{q_1}{\log^c m}$ it is enough to show that for any $(S,H)$,
$$
\bignorm{\sum_{i \neq 1} \gamma_i(S,H) A_i}_2 \leq \frac{q_1}{(4L)^{2L}\log^c m}
$$
Using the fact that $\|A\|_2 \le \sigma$, we have that 
$
\norm{\sum_{i \neq 1} \gamma_i(S,H) A_i}_2 \leq \sigma \sqrt{\sum_{i \neq 1} \gamma^2_i(S,H)}.
$
Hence, it will suffice to show that for any 
\begin{equation}\label{eq:gammaSH}
\forall (S,H),~~\sum_{i \neq 1} \gamma^2_i(S,H) \leq \frac{q^2_1}{(4L)^{4L}  \sigma^2 \log^{2c}m}
\end{equation}.

\anote{Changing all the $\nu^2(\eta,2L)$ to $\nu(\eta,2L)$.}
From \eqref{eq:semi-random-F-full}, \eqref{eq:semi-random-gammai-final-bound}, \eqref{eq:gamma-i-H-zero-final} and \eqref{eq:gammai-F-empty} we get that 
We notice that across all partitions 
\begin{align*}
|\gamma_i(S,H)| &\leq q_i C^{4L} \nu(\eta,2L) \Bigparen{\frac{k Z_i}{m} + \sqrt{\frac{\tau k}{m} }}\max\left(\frac{\mu}{\sqrt{n}}, (\frac{k \tau}{m})^{3/2} \right)\\
&= \gamma^{(1)}_i(S,H) + \gamma^{(2)}_i(S,H), \text{ where for our choice of  } k = o(m^{2/3}) \text{ we have },\\
\gamma^{(1)}_i(S,H) &= q_{i} C^{4L} \nu(\eta,2L) \frac{k\mu}{m\sqrt{n}}Z_i, \quad \quad\quad 
\gamma^{(2)}_i(S,H) = q_{i} C^{4L} \nu(\eta,2L)\sqrt{\frac{k\tau}{m}}\frac{\mu}{\sqrt{n}}
\end{align*}
We will separately show that $\forall j \in \set{1,2}$
$$
\sum_{i \neq 1} (\gamma^{(j)}_i(S,H))^2 \leq \frac{q^2_1}{4 \sigma^2 (4L)^{4L} \log^{2c} m }.
$$
For $j=1$ we have
\begin{align*}
\sum_{i \neq 1} (\gamma^{(1)}_i(S,H))^2 &= \sum_{i \neq 1} q^2_i C^{8L} \nu^2(\eta,2L) (\frac{k\tau}{m\sqrt{n}})^2 Z^2_i\\
& \leq q^2_1 C^{8L} \nu^2(\eta,2L) \sum_{i \neq 1} (\frac{k\tau}{m\sqrt{n}})^2 Z^2_i
\end{align*}

Notice that $Z_i$s are non-negative random variables with support distribution that is $\tau$-negatively correlated. Hence, using Lemma~\ref{lem:Z-conc} with $p = \frac{k}{m}$ and $\|a\|_1 = \frac{k \tau}{\sqrt{n}}$ we get,
\begin{align*}
\sum_{i \neq 1} (\gamma^{(1)}_i(S,H))^2 &\leq q^2_i C^{8L} \nu^2(\eta,2L) \frac{C^2\sqrt{k^2 \tau}}{\sqrt{m\sqrt{n}}\log^{(c-1)/2} m}\\
& \leq \frac{q^2_1}{4 \sigma^2 (4L)^{4L} \log^{2c} m }
\end{align*}
for our choice of $k$.
For $j=2$ we have
\begin{align*}
\sum_{i \neq 1} (\gamma^{(2)}_i(S,H))^2 &= \sum_{i \neq 1} q^2_i C^{8L} \nu^2(\eta,2L) (\frac{k\tau \mu^2}{m{n}})\\
& \leq q^2_1 C^{8L} \nu^2(\eta,2L)(\frac{k\tau \mu^2}{{n}})\\
&\leq \frac{q^2_1}{4 \sigma^2 (4L)^{4L} \log^{2c} m }
\end{align*}
for our choice of $k$. Combining all partitions we get that $\|\sum_{i  \neq 1} \gamma_i A_i\| = O(\frac{q_1}{\log^c m})$. This establishes the proof of Theorem~\ref{thm:random-high-order}.

\section{Acknowledgements}
The authors thank Sivaraman Balakrishnan, Aditya Bhaskara, Anindya De, Konstantin Makarychev and David Steurer for several helpful discussions. 

\bibliographystyle{alpha}
\bibliography{aravind}

\newcommand{\etalchar}[1]{$^{#1}$}
\begin{thebibliography}{AGMM15}

\bibitem[AAJ{\etalchar{+}}13]{Alekhsparse}
Alekh Agarwal, Animashree Anandkumar, Prateek Jain, Praneeth Netrapalli, and
  Rashish Tandon.
\newblock Learning sparsely used overcomplete dictionaries via alternating
  minimization.
\newblock {\em CoRR}, abs/1310.7991, 2013.

\bibitem[AAN13]{agarwal2013exact}
Alekh Agarwal, Animashree Anandkumar, and Praneeth Netrapalli.
\newblock Exact recovery of sparsely used overcomplete dictionaries.
\newblock {\em stat}, 1050:8--39, 2013.

\bibitem[ABGM14]{arora2014more}
Sanjeev Arora, Aditya Bhaskara, Rong Ge, and Tengyu Ma.
\newblock More algorithms for provable dictionary learning.
\newblock {\em arXiv preprint arXiv:1401.0579}, 2014.

\bibitem[AEB06]{aharon2006uniqueness}
Michal Aharon, Michael Elad, and Alfred~M Bruckstein.
\newblock On the uniqueness of overcomplete dictionaries, and a practical way
  to retrieve them.
\newblock {\em Linear algebra and its applications}, 416(1):48--67, 2006.

\bibitem[AGM14]{AGMM14}
Sanjeev Arora, Rong Ge, and Ankur Moitra.
\newblock New algorithms for learning incoherent and overcomplete dictionaries.
\newblock In {\em Proceedings of The 27th Conference on Learning Theory, {COLT}
  2014, Barcelona, Spain, June 13-15, 2014}, pages 779--806, 2014.

\bibitem[AGMM15]{AGMM15}
Sanjeev Arora, Rong Ge, Tengyu Ma, and Ankur Moitra.
\newblock Simple, efficient, and neural algorithms for sparse coding.
\newblock In {\em Proceedings of The 28th Conference on Learning Theory, {COLT}
  2015, Paris, France, July 3-6, 2015}, pages 113--149, 2015.

\bibitem[AV17]{AV18}
Pranjal Awasthi and Aravindan Vijayaraghavan.
\newblock Clustering semi-random mixtures of gaussians.
\newblock {\em CoRR}, abs/1711.08841, 2017.

\bibitem[AW15]{AdamczakWolff}
Rados{\l}aw Adamczak and Pawe{\l} Wolff.
\newblock Concentration inequalities for non-lipschitz functions with bounded
  derivatives of higher order.
\newblock {\em Probability Theory and Related Fields}, 162(3):531--586, Aug
  2015.

\bibitem[BCV14]{BCV}
Aditya Bhaskara, Moses Charikar, and Aravindan Vijayaraghavan.
\newblock Uniqueness of tensor decompositions with applications to polynomial
  identifiability.
\newblock {\em Proceedings of the Conference on Learning Theory (COLT).}, 2014.

\bibitem[BDDW08]{Baraniuk}
Richard Baraniuk, Mark Davenport, Ronald DeVore, and Michael Wakin.
\newblock A simple proof of the restricted isometry property for random
  matrices.
\newblock {\em Constructive Approximation}, 28(3):253--263, Dec 2008.

\bibitem[BDF{\etalchar{+}}11]{bourgain2011explicit}
Jean Bourgain, Stephen Dilworth, Kevin Ford, Sergei Konyagin, Denka Kutzarova,
  et~al.
\newblock Explicit constructions of rip matrices and related problems.
\newblock {\em Duke Mathematical Journal}, 159(1):145--185, 2011.

\bibitem[BKS15]{BKS15}
Boaz Barak, Jonathan~A. Kelner, and David Steurer.
\newblock Dictionary learning and tensor decomposition via the sum-of-squares
  method.
\newblock In {\em Proceedings of the Forty-Seventh Annual ACM on Symposium on
  Theory of Computing}, STOC '15, pages 143--151, New York, NY, USA, 2015. ACM.

\bibitem[BN16]{blasiok2016improved}
Jaros{\l}aw B{\l}asiok and Jelani Nelson.
\newblock An improved analysis of the er-spud dictionary learning algorithm.
\newblock {\em arXiv preprint arXiv:1602.05719}, 2016.

\bibitem[BS95]{BS92}
Avrim Blum and Joel Spencer.
\newblock Coloring random and semi-random k-colorable graphs.
\newblock {\em J. Algorithms}, 19:204--234, September 1995.

\bibitem[Com94]{ICA1}
Pierre Comon.
\newblock Independent component analysis, a new concept?
\newblock {\em Signal Processing}, 36(3):287 -- 314, 1994.
\newblock Higher Order Statistics.

\bibitem[CRT06]{candes2006stable}
Emmanuel~J Candes, Justin~K Romberg, and Terence Tao.
\newblock Stable signal recovery from incomplete and inaccurate measurements.
\newblock {\em Communications on pure and applied mathematics},
  59(8):1207--1223, 2006.

\bibitem[CT05]{CTao05}
E.~J. Candes and T.~Tao.
\newblock Decoding by linear programming.
\newblock {\em IEEE Trans. Inf. Theor.}, 51(12):4203--4215, December 2005.

\bibitem[CT10]{CT10}
Emmanuel~J. Cand\`{e}s and Terence Tao.
\newblock The power of convex relaxation: Near-optimal matrix completion.
\newblock {\em IEEE Trans. Inf. Theor.}, 56(5):2053--2080, May 2010.

\bibitem[DF16]{DF16}
Roee David and Uriel Feige.
\newblock On the effect of randomness on planted 3-coloring models.
\newblock In {\em Proceedings of the 48th Annual ACM SIGACT Symposium on Theory
  of Computing}, STOC 2016, pages 77--90, New York, NY, USA, 2016. ACM.

\bibitem[DH01]{donoho2001uncertainty}
David~L Donoho and Xiaoming Huo.
\newblock Uncertainty principles and ideal atomic decomposition.
\newblock {\em IEEE transactions on information theory}, 47(7):2845--2862,
  2001.

\bibitem[DLCC07]{Cardoso}
L.~De~Lathauwer, J.~Castaing, and J.~Cardoso.
\newblock Fourth-order cumulant-based blind identification of underdetermined
  mixtures.
\newblock {\em IEEE Trans. on Signal Processing}, 55(6):2965--2973, 2007.

\bibitem[dlPMS95]{dlP}
Victor~H. de~la Pena and S.~J. Montgomery-Smith.
\newblock Decoupling inequalities for the tail probabilities of multivariate
  $u$-statistics.
\newblock {\em Ann. Probab.}, 23(2):806--816, 04 1995.

\bibitem[DMA97]{davis1997adaptive}
Geoff Davis, Stephane Mallat, and Marco Avellaneda.
\newblock Adaptive greedy approximations.
\newblock {\em Constructive approximation}, 13(1):57--98, 1997.

\bibitem[DP09]{DubhashiPanconesi}
Devdatt Dubhashi and Alessandro Panconesi.
\newblock {\em Concentration of Measure for the Analysis of Randomized
  Algorithms}.
\newblock Cambridge University Press, New York, NY, USA, 1st edition, 2009.

\bibitem[DS89]{donoho1989uncertainty}
David~L Donoho and Philip~B Stark.
\newblock Uncertainty principles and signal recovery.
\newblock {\em SIAM Journal on Applied Mathematics}, 49(3):906--931, 1989.

\bibitem[Fel68]{BerryEsseen}
William Feller.
\newblock {\em An Introduction to Probability Theory and Its Applications},
  volume~1.
\newblock Wiley, January 1968.

\bibitem[FJK96]{ICA2}
Alan~M. Frieze, Mark Jerrum, and Ravi Kannan.
\newblock Learning linear transformations.
\newblock In {\em FOCS}, 1996.

\bibitem[FK98]{FK99}
U.~Feige and J.~Kilian.
\newblock Heuristics for finding large independent sets, with applications to
  coloring semi-random graphs.
\newblock In {\em Foundations of Computer Science, 1998. Proceedings.39th
  Annual Symposium on}, pages 674 --683, nov 1998.

\bibitem[GCB12]{goodfellow2012large}
Ian Goodfellow, Aaron Courville, and Yoshua Bengio.
\newblock Large-scale feature learning with spike-and-slab sparse coding.
\newblock {\em arXiv preprint arXiv:1206.6407}, 2012.

\bibitem[GTC05]{georgiev2005sparse}
Pando Georgiev, Fabian Theis, and Andrzej Cichocki.
\newblock Sparse component analysis and blind source separation of
  underdetermined mixtures.
\newblock {\em IEEE transactions on neural networks}, 16(4):992--996, 2005.

\bibitem[GVX14]{GVX14}
Navin Goyal, Santosh Vempala, and Ying Xiao.
\newblock Fourier {PCA} and robust tensor decomposition.
\newblock In {\em Symposium on Theory of Computing, {STOC} 2014, New York, NY,
  USA, May 31 - June 03, 2014}, pages 584--593, 2014.

\bibitem[HW71]{HansonWright}
D.~L. Hanson and F.~T. Wright.
\newblock A bound on tail probabilities for quadratic forms in independent
  random variables.
\newblock {\em Ann. Math. Statist.}, 42(3):1079--1083, 06 1971.

\bibitem[KMM11]{KMM}
Alexandra Kolla, Konstantin Makarychev, and Yury Makarychev.
\newblock How to play unique games against a semi-random adversary.
\newblock In {\em Proceedings of 52nd IEEE symposium on Foundations of Computer
  Science}, FOCS '11, 2011.

\bibitem[KS17]{kothari2017outlier}
Pravesh~K Kothari and David Steurer.
\newblock Outlier-robust moment-estimation via sum-of-squares.
\newblock {\em arXiv preprint arXiv:1711.11581}, 2017.

\bibitem[KV00]{KimVu}
Jeong~Han Kim and Van~H. Vu.
\newblock Concentration of multivariate polynomials and its applications.
\newblock {\em Combinatorica}, 20(3):417--434, Mar 2000.

\bibitem[Lat06]{Latala}
Rafał Latała.
\newblock Estimates of moments and tails of gaussian chaoses.
\newblock {\em Ann. Probab.}, 34(6):2315--2331, 11 2006.

\bibitem[Low99]{lowe1999object}
David~G Lowe.
\newblock Object recognition from local scale-invariant features.
\newblock In {\em Computer vision, 1999. The proceedings of the seventh IEEE
  international conference on}, volume~2, pages 1150--1157. Ieee, 1999.

\bibitem[LV15]{luh2015random}
Kyle Luh and Van Vu.
\newblock Random matrices: l1 concentration and dictionary learning with few
  samples.
\newblock In {\em Foundations of Computer Science (FOCS), 2015 IEEE 56th Annual
  Symposium on}, pages 1409--1425. IEEE, 2015.

\bibitem[MMV12]{MMV12}
Konstantin Makarychev, Yury Makarychev, and Aravindan Vijayaraghavan.
\newblock Approximation algorithms for semi-random partitioning problems.
\newblock In {\em Proceedings of the 44th Symposium on Theory of Computing
  (STOC)}, pages 367--384. ACM, 2012.

\bibitem[MMV13]{MMVfas}
Konstantin Makarychev, Yury Makarychev, and Aravindan Vijayaraghavan.
\newblock Sorting noisy data with partial information.
\newblock In {\em Proceedings of the 4th conference on Innovations in
  Theoretical Computer Science}, pages 515--528. ACM, 2013.

\bibitem[MMV14]{MMV14}
Konstantin Makarychev, Yury Makarychev, and Aravindan Vijayaraghavan.
\newblock Constant factor approximations for balanced cut in the random pie
  model.
\newblock In {\em Proceedings of the 46th Symposium on Theory of Computing
  (STOC)}. ACM, 2014.

\bibitem[MMV15]{MMVCC}
Konstantin Makarychev, Yury Makarychev, and Aravindan Vijayaraghavan.
\newblock Correlation clustering with noisy partial information.
\newblock {\em Proceedings of the Conference on Learning Theory (COLT)}, 2015.

\bibitem[MMV16]{MMVSBM}
Konstantin Makarychev, Yury Makarychev, and Aravindan Vijayaraghavan.
\newblock Learning communities in the presence of errors.
\newblock {\em Proceedings of the Conference on Learning Theory (COLT)}, 2016.

\bibitem[MPW15]{MPW15}
Ankur Moitra, William Perry, and Alexander~S. Wein.
\newblock How robust are reconstruction thresholds for community detection.
\newblock {\em CoRR}, abs/1511.01473, 2015.

\bibitem[MS10]{MS10}
Claire Mathieu and Warren Schudy.
\newblock Correlation clustering with noisy input.
\newblock In {\em Proceedings of the Twenty-first Annual ACM-SIAM Symposium on
  Discrete Algorithms}, SODA '10, pages 712--728, Philadelphia, PA, USA, 2010.
  Society for Industrial and Applied Mathematics.

\bibitem[MSS16]{ma2016polynomial}
Tengyu Ma, Jonathan Shi, and David Steurer.
\newblock Polynomial-time tensor decompositions with sum-of-squares.
\newblock In {\em Foundations of Computer Science (FOCS), 2016 IEEE 57th Annual
  Symposium on}, pages 438--446. IEEE, 2016.

\bibitem[O'D14]{Ryanbook}
Ryan O'Donnell.
\newblock {\em Analysis of Boolean Functions}.
\newblock Cambridge University Press, New York, NY, USA, 2014.

\bibitem[OF97]{sparse1}
Bruno~A. Olshausen and David~J. Field.
\newblock Sparse coding with an overcomplete basis set: A strategy employed by
  v1?
\newblock {\em Vision Research}, 37(23):3311 -- 3325, 1997.

\bibitem[PW17]{PW15}
A.~Perry and A.~S. Wein.
\newblock A semidefinite program for unbalanced multisection in the stochastic
  block model.
\newblock In {\em 2017 International Conference on Sampling Theory and
  Applications (SampTA)}, pages 64--67, July 2017.

\bibitem[QSW14]{qu2014finding}
Qing Qu, Ju~Sun, and John Wright.
\newblock Finding a sparse vector in a subspace: Linear sparsity using
  alternating directions.
\newblock In {\em Advances in Neural Information Processing Systems}, pages
  3401--3409, 2014.

\bibitem[SS12]{SchudyS}
Warren Schudy and Maxim Sviridenko.
\newblock Concentration and moment inequalities for polynomials of independent
  random variables.
\newblock In {\em SODA}, 2012.

\bibitem[SWW13]{SWW}
Daniel~A. Spielman, Huan Wang, and John Wright.
\newblock Exact recovery of sparsely-used dictionaries.
\newblock In {\em Proceedings of the Twenty-Third International Joint
  Conference on Artificial Intelligence}, IJCAI '13, pages 3087--3090. AAAI
  Press, 2013.

\end{thebibliography}

\appendix

\section{Proofs from Section~\ref{sec:prelims}}
\begin{proof}[Proof of Lemma~\ref{lem:q-sum}]
The lower bound of $1$ can be easily seen by setting $i_R = i_{R-1}$. For the upper bound, let $\mathcal{S}_k$ be the set of all $k$-sparse vectors in $\{0,1\}^m$. Then we have $\Psymb(\cup_{\zeta \in \mathcal{S}_k} \zeta) = 1$. Let $A$ be the event that $(\zeta_{i_1} \neq 0, \dots, \zeta_{i_{R-1}}\neq 0)$. Since each vector is $k$-sparse we have that 
\begin{align*}
\sum_{i_R \in [m]} q_{i_1, i_2, \dots, i_R} &= \sum_{i_R \in [m]} \Psymb(\zeta_{i_1} \neq 0, \zeta_{i_2} \neq 0, \dots, \zeta_{i_R} \neq 0)\\
&\leq k \Psymb(A)\\
&= k q_{i_1, i_2, \dots, i_{R-1}}
\end{align*}
For the second part, let $S$ be the set of indices $i_1,i_2,\dots,i_{R-1}$. Then we have 
\begin{align*}
\frac{q_{i_1,i_2,\dots,i_R}}{q_{i_1,i_2,\dots,i_{R-1}}} &= \frac{\Psymb(\bigcap_{j \in S} \zeta_{j} \neq 0 \text{ and } \zeta_R \neq 0)}{\Psymb(\bigcap_{j \in S} \zeta_{j} \neq 0)}\\
&= \frac{\Psymb(\bigcap_{j \in S} \zeta_{j} \neq 0) \Psymb(\zeta_R \neq 0 | \bigcap_{j \in S} \zeta_{j} \neq 0 )}{\Psymb(\bigcap_{j \in S} \zeta_{j} \neq 0)}\\
&= \Psymb(\zeta_R \neq 0 | \bigcap_{j \in S} \zeta_{j} \neq 0 ) \leq \frac{k  \tau}{m}
\end{align*}
where the last inequality makes use the fact that the $\zeta$s are $\tau$-negatively correlated.
\end{proof}
\begin{proof}[Proof of Lemma~\ref{lem:q-d-sum}]
Again the lower bound is easy to see by setting $i_R = i_{R-1}$. For the upper bound Let $A$ be the event that $(\zeta_{i_1} \neq 0, \dots, \zeta_{i_{R-1}}\neq 0)$ and $B$ be the event $(\zeta_{i_1} \neq 0, \dots, \zeta_{i_{R}}\neq 0)$. We have
\begin{align*}
\sum_{i_R \in [m]} \frac{q_{i_1, i_2, \dots, i_R}(d_1, d_2, \dots, d_R)}{q_{i_1, i_2, \dots, i_{R-1}}(d_1, d_2, \dots, d_{R-1})} &= \sum_{i_R \in [m]} \frac{\E [ x^{d_1}_{i_1}x^{d_2}_{i_2}\dots x^{d_R}_{i_R} | B] \Psymb(B)}{\E[x^{d_1}_{i_1}x^{d_2}_{i_2}\dots x^{d_{R-1}}_{i_{R-1}} | A] \Psymb(A)}\\
&\leq C^{d_R} \sum_{i_R \in [m]} \frac{\Psymb(B)}{\Psymb(A)} \leq \frac{\tau k C^{d_R}}{m}.
\end{align*}
Here we have used the fact that values are picked independently from $\Dv$ conditioned on support and hence $\E[x^{d_1}_{i_1}x^{d_2}_{i_2}\dots x^{d_R}_{i_R} | B] = \prod_{t} E[x^{d_t}_{i_t} | \zeta_{i_t} \neq 0]$. Furthermore, we have $E[x^{d_R}_{i_R} | \zeta_{i_R} \neq 0] \leq C^{d_R}$ and from Lemma~\ref{lem:q-sum} we have that $\sum_{i_R \in [m]} \frac{\Psymb(B)}{\Psymb(A)} \leq \frac{k \tau}{m}$. The second part follows similarly by noting that
$$
\frac{q_{i_1, i_2, \dots, i_R}(d_1, d_2, \dots, d_R)}{q_{i_1, i_2, \dots, i_R}(d_1, d_2, \dots, d_{R-1})} \leq C^{d_R} \frac{q_{i_1, i_2, \dots, i_R}}{q_{i_1, i_2, \dots, i_{R-1}}}.
$$
and using the second consequence of Lemma~\ref{lem:q-sum}.
\end{proof}
\begin{proof}[Proof of Lemma~\ref{lem:spectral-norm-incoherent}]
For the lower bound notice that since $A$ has unit length columns we have that $\|A\|^2_F=1$. Since the squared Frobenius norm is also the sum of squared singular values and the rank of $A$ is at most $n$, we must have $\|A\|^2 \geq \frac{m}{n}$.

For the upper bound, consider a unit length vector $x \in \R^m$. We have,
\begin{align*}
\|Ax\|^2 &= \sum_{i \in [m]} x^2_i \|A_i\|^2 + \sum_{i \neq j} x_i x_j \iprod{A_i,A_j}.\\
&= \|x\|^2 + \sum_{i \neq j} x_i x_j \iprod{A_i,A_j}\\
&\leq \|x\|^2 + \sqrt{\sum_{i \neq j} x^2_i x^2_j}\sqrt{\sum_{i \neq j} {\iprod{A_i,A_j}}^2}\\
& \leq \|x\|^2 + \|x\|^2 \frac{\mu m}{\sqrt{n}}\\
&\leq (1+\frac{\mu m}{\sqrt{n}})
\end{align*}
\end{proof}
\begin{proof}[Proof of Lemma~\ref{lem:spectral-norm-rip}]
Lower bound follows exactly from the same argument as above. For the upper bound, given a vector $x \in \R^m$ we write it as a sum of $\frac m k$\footnote{For simplicity we assume that $k$ is a multiple of $m$.}, $k$-sparse vectors, i.e., $x = y_1+y_2+\dots+y_{\frac m k}$. Here $y_1$ is a vector that is non-zero on coordinates $1$ to $k$ and takes the same value as $x$ in those coordinates. Similarly, $y_i$ is a vector that is non-zero in coordinates $(i-1)k+1$ to $ik$ and takes the same value as $x$ in those coordinates. Then we have $Ax = \sum_{i=1}^{\frac m k} Ay_i$ and that $\sum_{i=1}^{\frac m k} \|y_i\|^2 = \|x\|^2$. Hence we get by triangle inequality that 
\begin{align*}
\|Ax\| &\leq \sum_{i=1}^{\frac m k} \|Ay_i\|\\
&\leq (1+\delta)\sum_{i=1}^{\frac m k} \|y_i\| \leq (1+\delta)\sqrt{\frac m k}\sqrt{\sum_{i=1}^{\frac m k} ||y_i||^2}\\
&= (1+\delta)\sqrt{\frac m k} \|x\|
\end{align*}
where in the second inequality we have used the fact that $A$ satisfies $(k,\delta)$-RIP and $y_i$s are $k$-sparse vectors.
\end{proof}
\begin{proof}[Proof of Lemma~\ref{lem:incoherence-implies-rip}]
Let $A$ be a $\mu$-incoherent matrix. Given a $k$-sparse vector $x$, we assume w.l.o.g. that the first $k$ coordinates of $x$ are non-zero. Then we have $Ax = \sum_{i=1}^k x_i A_i$. Hence we get
\begin{align*}
\|Ax\|^2 &= \sum_{i=1}^k x^2_i \|A_i\|^2 + \sum_{i \neq j}x_i x_j \iprod{A_i,A_j}\\
&= \|x\|^2 + \sum_{i \neq j}x_i x_j \iprod{A_i,A_j}\\
&\leq \|x\|^2 \pm \sqrt{\sum_{i \neq j} x^2_i x^2_j}\sqrt{\sum_{i \neq j} \iprod{A_i,A_j}^2}\\
&\leq \|x\|^2 \pm \|x\|^2 \frac{\mu k}{\sqrt{n}}
\end{align*}
Hence we get that 
$$
(1-\delta) \leq \frac{\|Ax\|}{\|x\|} \leq (1+\delta)
$$
for $\delta = \frac{2\mu k}{\sqrt{n}}$
\end{proof}
\begin{proof}[Proof of Lemma~\ref{lem:largevalues}]
The first part just follows from the fact that the maximum singular value of $A_T$ is at most $1+\delta$. 

Suppose for contradiction $T_\gamma=|\set{i \in [m]: |\iprod{z,A_i}|> \gamma}| \ge (1+\delta)/\gamma^2+1$. Let $T$ be any subset of $(1+\delta)/\gamma^2+1$ of $T_\gamma$. 

From the RIP property of $A$, we have for any unit vector $z \in \R^n$,  $\norm{z^T A_T}_2 \le \norm{A_T} \le (1+\delta)$. Suppose $|T|= 1+(1+\delta)/\gamma^2$,
\begin{align*}
|T|\gamma^2 \le \sum_{i \in T} \iprod{z, A_i}^2 = \norm{z^T A_T}_2^2 &\le  (1+\delta) \\
\text{Hence } |T| &\le \frac{1+\delta}{\gamma^2},
\end{align*}
which contradicts the assumption that $|T| \ge 1/\gamma^2+1$.

\end{proof}

\begin{proof}[Proof of Lemma~\ref{lem:incoherencefromRIP}]
Let $B$ be the submatrix of $A$ restricted to the columns given by $T \cup {i}$. From the RIP property the max singular value $\norm{B} \le 1+\delta$. Since this is also the maximum left singular value, 
\begin{align*}
(1+\delta)^2 \ge \norm{A_i^T B}_2^2 &=\norm{A_i}_2^2+\sum_{j \in T} \iprod{A_i, A_j}^2 .\\
\text{Hence } \sum_{j \in T} \iprod{A_i,A_j}^2 &\le 2\delta+\delta^2. 
\end{align*}
\end{proof}

\section{Auxiliary Lemmas for Producing Candidates}
We prove two simple linear-algebraic facts, that is useful in the analysis for the initialization procedure. The first is about the operator norms of the Khatri-Rao product (see \cite{BCV} shows an analogous statement for minimum singular value) . 
\begin{lemma}\label{lem:frob:fact2}
Consider any two matrices $A \in \R^{n_1 \times m}, B \in \R^{n_1 \times m}$ and let $A_i$ ($B_i$) be the $i$th column of $A$ ($B$ respectively). If $M=A \odot B$ denotes the $(n_1 n_2) \times m$ matrix with the $i$th column $M_i = A_i \otimes B_i$. 
$\norm{M}_{op} \le \norm{A}_{op} \norm{B}_{op}$.
\end{lemma}
\begin{proof}
Consider the matrix $M$, and let $u \in \R^m$ be any unit vector. The length of the $n_1 n_2$ dimensional vector $Mu$ is
\begin{align*}
\norm{Mu}_2 &= \norm{\sum_{i \in [m]} u_i A_i \otimes B_i }_2 = \norm{\sum_{i \in [m]} u_i A_i B_i^T}_F \\
&= \norm{A \diag(u) B^T}_F \le \norm{A} \cdot \norm{\diag(u)}_F \cdot \norm{B^T} \le \norm{u}_2 \norm{A} \norm{B}\le \norm{A} \norm{B}.  
\end{align*}

\end{proof}

\anote{Mildly changed the lemma to make it more general. }
The second lemma involves Frobenius norms of tensor products of PSD matrices. 
\begin{lemma}\label{lem:frob:fact1}
Given any $n \in \N$ and a set of $n$ PSD matrices $A_1, A_2, \dots, A_n \succeq 0$, and $n$ other matrices $B_1, \dots, B_n$, we have
$$\Bignorm{\sum_{i=1}^n A_i \otimes B_i}_F \le \Bignorm{\sum_{i=1}^n  \norm{B_i}_F A_i}_F .$$
\end{lemma}
\begin{proof}
Let $\iprod{A_{i}, A_{i'}}=\tr(A_i A_{i'})$ be the vector inner product of the flattened matrices $A_i, A_{i'}$ (similarly for $B_i, B_{i'}$).  
\begin{align*}
\Bignorm{\sum_{i=1}^n A_i \otimes B_i}_F^2 &= \Iprod{ \sum_{i=1}^n A_i \otimes B_i, \sum_{i'=1}^n A_{i'} \otimes B_{i'}  }=\sum_{i=1}^n \sum_{i'=1}^n \iprod{A_i, A_{i'}} \iprod{B_i, B_{i'}} \\
&\le \sum_{i=1}^n \sum_{i'=1}^n \iprod{A_i, A_{i'}} \norm{B_i}_F \norm{B_{i'}}_F \le \Bignorm{\sum_{i=1}^n \norm{B_i}_F A_i}_F^2, 
\end{align*}
where the first inequality on the last line also used the fact that $\tr(M_1 M_2) \ge 0$ when $M_1, M_2 \succeq 0$. 
\end{proof}

\anote{4/19: Seems incorrect as stated. There is an extra additive term of $1/\polylog m$ in the loss? See the next lemma for a more correct statement?}
\anote{4/19: Rewrote Lemma with the extra corrective term. }

\begin{lemma}
\label{lem:Z-conc}
Let $Z_1, Z_2, \dots Z_m$ be real valued random variables bounded in magnitude by $C > 0$ such that $[Z_1 \ne 0], [Z_2 \ne 0], \dots, [Z_m \ne 0]$ are $\tau$-negatively correlated (as in Section~\ref{sec:prelims}), $\Pr(Z_i \neq 0) \leq p$, and the values of the non-zeros are independent conditioned on the support. 
Let $Z = \sum_i a^2_i Z^2_i$ for real values $a_1, a_2, \dots a_m$ such that $\|a\|_1 p C^2 \leq 1$. Then for any constant $c \geq 4$, with probability at least $1-\|a\|_1 p \log^c m - \frac{1}{m^2}$, we have that
$$
Z = \sum_i a^2_i Z^2_i 
\le \frac{C^2 \tau}{\log^c m}.
$$
\end{lemma}
\begin{proof}
\anote{Need to correct/ extend this to $\tau$-correlated?}
We first give a simple proof using Chernoff bound for the case when $\tau=1$ (negatively correlated).  
Let $T = \{i \in [m]: |a_i| \geq \frac{1}{\log^c m}\}$. Then we have that $|T| \leq \|a\|_1 \log^c m$. Define an event $E = \cup_{i \in T} (Z_i \neq 0)$. By union bound we have that $P(E) \leq p|T| \leq \|a\|_1 p \log^c m$. Conditioned on $E$ not occurring we have that $Z = \sum_{i \notin T}a^2_i Z^2_i$. Notice that $\max_{i \notin T} |a^2_i Z^2_i| \leq \frac{C^2}{\log^{2c} m}$. Furthermore we have that
\begin{align*}
\E[\sum_{i \notin T}a^2_i Z^2_i] &\leq \sum_{i \notin T} a^2_i C^2 p \leq \frac{\|a\|_1 C^2 p}{\log^c m}.
\end{align*}

Let $\lambda= C^2 \log^{-c} m \ge C^{2} \tau \big( \sqrt{\norm{a}_1 p} \log^{-(3c-1)/2} m + \log^{-(2c-1)} m \big)$. 
Hence by using Chernoff-Hoeffding bounds for negatively correlated random variables~\cite{DubhashiPanconesi} 
we get that 
\begin{align*}
\Pr\Big[Z \ge \lambda ~|~ E^c \Big]&\le \exp\Big( - \frac{\lambda^2}{\max_{i \notin T} |a^2_i Z^2_i| \cdot (2\E[\sum_{i \notin T}a^2_i Z^2_i] + \lambda) } \Big)\\
&\le \exp\Big( - \frac{\lambda^2}{C^2\log^{-2c}m \cdot (2 \|a\|_1 C^2 p\log^{-c} m + \lambda)} \Big) \le \frac{1}{m^2},
\end{align*}
for our choice of $\lambda$. 
Noticing that $P(E^c) \geq 1 - \|a\|_1 p \log^c m$ we get the claim.

An alternate proof that also extends to the more general case of $\tau$-negatively correlated support distribution, can be obtained by using Lemma~\ref{lem:randomfrobnorm} with the vector (first-order tensor) corresponding to terms $i \notin T$ with $d=1, \rho=1+(\max_{i \notin T} a_i^2/\E[Z])$ to obtain the conclusion of the above lemma. 
\end{proof}

\end{document}